\documentclass[3p]{elsarticle}
\usepackage{lineno}
\usepackage{hyperref}
\journal{Journal of \LaTeX\ Templates}
\usepackage{amssymb}
\usepackage{mathtools}
\usepackage{graphicx}
\usepackage{hyperref}
\usepackage{amssymb,amstext,amsmath,mathabx,amsthm}
\usepackage{verbatim,color}
\usepackage{url}
\usepackage{scalerel} 
\usepackage{array} 
\newcolumntype{L}{>{$}l<{$}}

\usepackage{amssymb,amstext,amsmath}
\newcommand{\overbar}[1]{\mkern 1.5mu\overline{\mkern-1.5mu#1\mkern-1.5mu}\mkern 1.5mu}
\newcommand{\no}[1]{\overbar{#1}}
\def\F{\mathcal F}
\def\P{\mathcal P}

\def\pr{\mathbb{P}}
\def\prev{\mathbb{P}}
\def\K{\mathcal{K}}

\def\C{\mathcal{C}}

\newtheorem{example}{Example}
\newtheorem{remark}{Remark}
\newtheorem{definition}{Definition}
\newtheorem{theorem}{Theorem}

\begin{document}

\begin{frontmatter}

\title{On Trivalent Logics,  Compound Conditionals, \\ and  Probabilistic Deduction Theorems\tnoteref{mytitlenote}}


 \author[ag]{Angelo Gilio\fnref{fn2}}
 \address[ag]{Department of Basic and Applied Sciences for Engineering, University of Rome ``La Sapienza'', Italy}
 \ead{angelo.gilio1948@gmail.com}

 \author[do]{David E.\ Over}
 \address[do]{Department of Psychology, Durham University,  United Kingdom}
 \ead{david.over@durham.ac.uk}
 \author[np]{Niki Pfeifer}
 \address[np]{Department of Philosophy, University of Regensburg, Germany}
 \ead{niki.pfeifer@ur.de}
 
 \author[gs]{Giuseppe Sanfilippo\corref{cor1}\fnref{fn1}}
 \address[gs]{Department of Mathematics and Computer Science, University of Palermo, Italy}
 \ead{giuseppe.sanfilippo@unipa.it}

 \fntext[fn1]{Also affiliated with INdAM-GNAMPA, Italy}

 \fntext[fn2]{Retired}
 \cortext[cor1]{Corresponding author}





\begin{abstract}
In this paper we recall  some results for conditional events, compound conditionals, conditional random quantities, p-consistency, and p-entailment.
Then, we show the  equivalence between  bets on conditionals   and  conditional bets, by reviewing   de Finetti's trivalent analysis of conditionals. But our approach  goes beyond de Finetti's early trivalent logical analysis and is  based on  
his later ideas, aiming to take
his proposals to a higher level. We  examine two recent articles that explore trivalent logics for conditionals and their definitions of logical validity and compare them with our approach to compound conditionals. 
We  prove a Probabilistic Deduction Theorem for  conditional events. After that,  we study some probabilistic deduction theorems, by presenting several examples.
 We  focus on iterated conditionals and the invalidity of the Import-Export principle  in the  light of our  Probabilistic Deduction Theorem. We  use the inference from a disjunction, \emph{$A$ or $B$}, to the  conditional, \emph{if not-$A$ then $B$}, as an example  to show  the invalidity of the  Import-Export principle. We also introduce a General Import-Export principle and we illustrate it by examining  some p-valid inference rules of System P. Finally, we briefly discuss some related work relevant to AI. 
\end{abstract}


\end{frontmatter}


\section{Introduction}
Conditionals are important in human reasoning under uncertainty because they allow individuals to make decisions and inferences based on incomplete or uncertain information. 
Thus, the interpretation and evaluation of conditionals is a key challenge for
artificial and human reasoning under uncertainty. 
Conditionals are relevant in AI because they are often used to implement if-then rules and are a crucial tool in a wide range of tasks and applications, from language processing to decision making.

There has recently been increasing interest in de Finetti's analysis of the conditional, \emph{if $A$ then $B$}, as what he called a \emph{conditional event} (\cite{definetti36,definetti37}) and symbolized as $B|A$.
Some logicians and psychologists of reasoning have focused on his \emph{trivalent} classification of conditionals as true, false, or void (\cite{BaOP13,BPOT18,EgRS20,Lass20,LaBa21,PoOB10}).
This classification has been also adopted in the field of artificial intelligence and probabilistic nonmonotonic reasoning  (see, e.g., \cite{biazzo02,biazzo05,CaLS07,coletti02,CoSV15,gilio02,gilio16,PeVa17,PfSa17}).
Research on trivalent logics and compounds of conditionals has been presented in many papers
(see, e.g., \cite{adams75,benferhat97,Cala87,Cala17,ColettiPV19,CiDu12,CiDu13,edgington95,GoNW91,Kauf09,McGe89,Miln97,NgWa94,vanFraassen76}).
A number of  authors argue that the set of values $\{${\em true, false, void}$\}$ of a conditional event $B|A$ should be represented in numerical terms as $\{1,0,P(B|A)\}$, where the (subjective) probability $P(B|A)$ may assume any value in the interval between 0 and 1, $[0,1]$ 
(see, e.g., \cite{gilio90,gilio12,GiSa14,jeffrey91,kleiter18,Lad95,pfeifer09b,SGOP20}). Based on this representation,  the geometrical approach for coherence checking has been extended to conditional probability assessments  (\cite{gilio90}). 
\\
In this paper, we give our reasons for adopting this proposal, and we explain how to use it as the basis of an account of some compound conditionals as conditional random quantities with values in the interval $[0, 1]$.  We use a principled way for
our analysis at the logical, cognitive, and psychological level of analysis and
discuss our work in the light of  selected recent literature. In particular, we
 examine two recent articles that adopt the trivalent view of conditionals and compare trivalent definitions of validity with  the notion of  probabilistic validity, p-validity. 
We apply our
coherence-based probability interpretation of conditionals, and compositions of
conditionals, by studying  their connections to the consequence relation, 
where each conclusion is related to the respective premises.
We give, in our account of conditional reasoning, some results on  several probabilistic versions of the deduction theorem. 
Moreover, we study 
properties of the behaviour of iterated conditionals in the context of a General
Import-Export principle, by examining some well known inference rules of 
nonmonotonic reasoning.
Compared to other approaches in the
literature, our approach, besides preserving basic probabilistic properties, allows to properly extending them to the case  where compound and iterated  conditionals are considered (\cite{GiSa21A}).  In particular, conditionals and iterated conditionals allow to suitably characterize the p-validity of inference rules in nonmonotonic reasoning (see, e.g., \cite{GiPS20,GiSa19, GiSa21E}).

We follow de Finetti in closely comparing an indicative conditional (IC), with a \emph{conditional bet} (CB). When we make this link, we do not look at an IC, \emph{if $A$ then $B$}, and a CB, informally of the form \emph{if $A$ then I bet on $B$}, as a material conditional, which is logically equivalent to $\no{A}\vee B$ (not-$A$ or $B$),
or a Stalnaker/Lewis conditional (see \cite{SPOG18} for the differences in detail). A CB only comes into effect when $A$ holds. We will show formally later (see Section \ref{SUBSEC:CONDBETCOND}) that such a conditional bet can be interpreted as a bet on the conditional event $B|A$. For an informal example, let the IC be \emph{If the coin is spun then it will come up heads}, and the CB be \emph{If the coin is spun, I bet it comes up heads}. There is this parallel relation between these conditionals. The IC is true, and the CB is won, when $AB$ holds, the IC is false, and the CB is lost, when $A \no{B}$ holds. When $\no{A}$ holds, the IC is void, and a counterfactual, \emph{If the coin had been spun, then it would have come up heads}, might be used in its place, and the CB will also be   void in the sense of being neither won nor lost. 

At this stage, it might appear that we have described only a trivalent account of the conditional, yielding a three-valued truth table, with the numerical values 1, 0, and $v$ for void. But de Finetti did not stop at this early trivalent analysis (\cite{definetti36}): he eventually extended his theory to another level. 

In \cite[pp. 1164-1165]{deFi80}, de Finetti  distinguished three levels of analysis: Level 0, Level 1, and Level 2. Level 0 is that of classical logic, where we assume that every statement is true or false, i.e., in numerical terms 1 or 0.
These levels of knowledge are also discussed in \cite{BPOT18,OvBa16}.
We recall below, with more details, the three levels of knowledge of an individual with respect to any event $E$. \\ 
Level 0. From a logical point of view, $E$ is false, ``0'', or true, ``1''. \\
Level 1. From a cognitive point of view,
$E$ can be false, ``0'', or uncertain, ``?'', or true, ``1''. \\
Level 2. From a psychological (subjective) point of view, in case of certainty $E$ is false, ``0'', or true, ``1''. In case of uncertainty, 
$E$ is given  a (subjective) probability ``$P(E)$''. \\ 
In other words, Level 0 specifies that events are two-valued logical entities (independently from the opinion of individuals). \\
Level 1 clarifies that  individuals can be certain (i.e., they know that $E$ is false, or that it is true), or uncertain about $E$ (i.e., they do not know if $E$ is false, or it is true).  
Level 2 specifies that people who are uncertain about $E$ represent their uncertainty by a (subjective) probability, which is a measure of their degree of belief on $E$ being true. \\ 
More in general, given a family of $n$ events $\F=\{E_1,\ldots,E_n\}$, we can say that: \\ \ \\
- we stay at Level 0 on $\F$ if we stay at Level 0 on each $E \in \F$; \\ \ \\
- we stay (partially) at Level 1 on $\F$ if we stay at Level 1 on some $E \in \F$; \\ \ \\
- we stay (partially) at Level 2 on $\F$ if we stay at Level 2 on some $E \in \F$. \\ \ \\
The analysis of de Finetti of the three levels of knowledge could  be extended to a conditional event $A|H$ by considering the partition $\pi=\{AH, \no{A}H, \no{H}\}$, as  illustrated below:\\ 
- we stay at Level 0 on  $A|H$ when we stay at Level 0 on $\pi$, in which case we know that $A|H$ is true, or we know that it is false, or we know that it is void. \\ \ \\
- we stay at Level 1 on $A|H$ when we stay at Level 1 on $\pi$; that is, we are uncertain, among the three possible cases, $AH, \no{A}H, \no{H}$, as to which is the true one. \\ \ \\
- when we represent uncertainty by probabilities, we can move to Level 2 in different ways depending on the uncertainty we would like to measure by degrees of belief; 
in particular,  we  are at Level 2 on $A|H$
if we assign   $P(A|H)$ to $A|H$, which measures the {\em conditional uncertainty} of $AH$ with respect to the two alternatives $AH$, or $\no{A}H$. Indeed, 
\[
P(A|H)=P(AH|H)=P[AH \, | \,( AH \vee \no{A}H)] \,.
\]
Notice that  a similar comment applies to the  unconditional event $A$ because $P(A)=P(A|\Omega)
=P(A|(A\vee \no{A})).$\\
We observe that, if  we represent our uncertainty about $AH$ and $\no{A}H$  by $P(AH)$ and $P(\no{A}H)$, then $P(A|H)$ is uniquely determined, when $P(AH)+P(\no{A}H)>0$, by the formula
\[
P(A|H)\;=\;P[AH \, | (\, AH \vee \no{A}H)]\;=\;\frac{P(AH)}{P(AH)+P(\no{A}H)} \,.
\]
Notice that 
we need a direct evaluation of $P(A|H)$ when $P(AH)=P(\no{A}H)=0$,  because in this case $P(A|H)$ is  not uniquely determined. But in this case, we can use the Ramsey test as extended by Stalnaker
(\cite{ramsey94,stalnaker68}). In this version of the test, we evaluate a conditional, \emph{if $H$ then $A$}, by hypothetically supposing $H$, while making minimal changes to preserve consistency in our beliefs, and then we judge our degree of belief in $A$ under this supposition of $H$. 
Similar analyses apply for  representing other conditional uncertainties. \\ 


By coming back to the indicative conditional IC,  at Level 1 we can express some uncertainty, and this is   what at Level 2 the  numerical value  $v=P(A|H)$ can be taken to represent. When the coin is not spun, we are uncertain whether it would have landed heads. At Level 2, we can refine our uncertainty into different degrees of belief, i.e., of subjective probability, and at this level, we can make precise bets on outcomes (see
\cite{BPOT18,OvBa16}
for more on these three levels). For example, we may know that the coin has a worn edge on one side and a certain tendency to come up heads. At this level of de Finetti's analysis, $v$ is what we are willing to pay, in money or epistemic utility, in a fair bet, which will pay us 1 unit, of money or epistemic utility, when $AB$  holds,  and  0 units, when $A \no{B}$ holds, and which will be returned to us when $\no{A}$ holds. In the simplest case, we may believe that the coin is fair, and then we will pay, 0.5 of a Euro, 50 cents, to win 1 Euro provided that the coin is spun and comes up heads. We will receive 0 Euros when the coin is spun and comes up tails, and we will get our 50 cents back when the coin is not spun. 

We can derive what $v$ is from the expected value, or prevision, of the conditional bet when this bet is fair, i.e., its expected value is 0. We have 
\[
0 = P(A\wedge B)(1 - v) + P(A \wedge\no{B})(0-v) + P(\no{A})(v-v),
\]
or equivalently: $v = P(A\wedge  B)+ P(\no{A})v$, and so  $P(A)v = P(A\wedge  B)$ and 
$v = P(A\wedge B)/P(A)$ when $P(A)>0$. 

We see, then, that $v$ is the conditional probability of $B$ given $A$, $P(B|A)$, and we identify $v$ with $P(B|A)$ for both IC and CB. Further aspects will be examined in Section \ref{SUBSEC:CONDBETCOND}.

Our full de Finetti analysis, at Level 2, is like an ``interval-valued" account, because $P(B|A)$ can have any value from 0 to 1. If we believe the coin is double-headed, we will say $P(B|A) = 1$. If we think that the coin is slightly biased to heads, we may judge $P(B|A) = 0.55$.

There are certainly reasons to study the partial three-valued analysis at Level 1. It is important, for instance, to study uncertainty in the psychology of reasoning, without presupposing that this state of mind can always be made more precise with probability judgments. People could sometimes just be uncertain about $X$ and unable to refine this to a probability judgment, $P(X)$. Nevertheless, there are many advantages in a full Level 2 de Finetti analysis. One is that a logical truth, such as most simply \emph{if   $A$  then $A$}, is never merely labelled ``void'' at Level 2, but always has the value 1, since $P(A|A) = 1$. Another advantage is that the full analysis allows us to define a logically valid inference in a more intuitive and natural way than can be done at Level 1, as we will point out in the formal development below, where we compare different definitions of logical validity (see further \cite{Cruz20}, on logically valid inference). Yet another advantage is that we can fully clarify the status of the Probabilistic Deduction Theorem in a de Finetti analysis, by exploiting compound conditionals, as we will also explain below.  Notice that in our approach compound conditionals, such as conjunctions and disjunctions, are defined (not as three-valued objects, but) as suitable conditional random quantities, where some of their values are (coherent) probability values. 
For instance, given two conditional events $A|H$ and $B|K$, if (for a given individual) $P(A|H)=x$ and $P(B|K)=y$, then the  possible values of the conjunction $(A|H)\wedge(B|K)$ are: $1,0,x,y,z$, where $z$ is (for the same individual) the prevision of $(A|H)\wedge(B|K)$. 
Then, we directly define compound conditionals at Level 2. In what follows this aspect will be  implied, and  it will be clear   from the context at what level we are examining the  objects.

A related point, and yet another advantage of our approach, is that we can clarify the status of the Probabilistic Deduction Theorem at Level 2. The deduction theorem is a fundamental metalogical theorem in classical logic \footnote{The presence of the deduction theorem allows for proofs in logic which are ``much more natural, simple and convenient'' compared to its absence \cite[p. 80]{surma81}.
Historically, the distinction between object- and meta-language was introduced much later, the admissibility of the deduction theorem, however,  was already ``taken for granted by Aristotle and explicitly by the Stoics'' \cite[p.320]{kneale84}. According to Kleene, ``the deduction theorem as an informal theorem proved about particular systems like the propositional calculus and the predicate calculus [\dots] first appears explicitly in Herbrand 1930 [\emph{Recherches sur la th\'eorie de la d\'emonstration}] (and without proof in Herbrand 1928 [\emph{Sur la th\'eorie de la d\'emonstration}]); and as a general methodological principle for axiomatic-deductive systems in Tarski 1930 [\emph{\"Uber einige fundamentale Begriffe der Metamathematik}]. According to Tarski 1956 [\emph{Logic, semantics, metamathematics}] footnote to p. 32, it was known and applied by Tarski since 1921'' \cite[p. 39]{kleene02}. Surma (\cite{surma81}) traces the deduction theorem back to Bolzano's \emph{Wissenschaftslehre} (1837). Bolzano's discovery was  rediscovered by Tarski in the 1920ies. Surma  also mentions that the  deduction theorem's name was coined by David Hilbert \cite[p. 79]{surma81}. 
}. It explains the relation between the material conditional, $\no{A}\vee B$, and logical consequence. 
More formally, in classical propositional logic, the (full) deduction theorem states that 
the premise set 
\textbf{$\Gamma\cup \{A\}$}, where $\Gamma$ is a set of  propositional formulas, logically implies the conclusion
$B$ if and only if 
$\Gamma$ implies  the material conditional 
$\no{A}\vee B$ (see, e.g., \cite{Franks2021}).
 However, the  deduction theorem is not generally valid in our probability logic, where $\Gamma$ is a (p-consistent) set of conditional events, the consequence relation is based on p-entailment ($\Rightarrow_p$) and the material conditional $\no{A}\vee B$ is replaced by the  conditional event  $B|A$. Some closely related metalogical theorems for the material conditional, particularly monotonicity, contraposition, and transitivity, usually do not hold in probability logic for the conditional event. In this paper, we will investigate the conditions under which a (restricted) version of the deduction theorem, the Probabilistic Deduction Theorem, holds. We will also compare our approach with other proposals for a probabilistic logic and what these imply about the deduction theorem.  We remark that the contents of our paper are relevant to AI for several motivations: 
\begin{itemize}
    \item   we investigate the interpretation and evaluation of simple  and compound conditionals, which  is a key challenge in artificial intelligence;
 \item  our probabilistic coherence-based approach is more realistic as it allows for considering  only those conditional events which are of interest (and hence algebraic structures are not required for the families of conditional events);
 \item we can properly manage conditionals with conditioning events evaluated by zero probabilities;
 \item  our approach goes beyond trivalent logics, because our compound and nested conditionals are defined as conditional random quantities in the setting of coherence: as a consequence all the basic logical and probabilistic properties are preserved;
\item  we give probabilistic versions of the deduction theorem and we introduce a General Import-Export
principle.
\end{itemize} 
 
The paper is organized in the following way. In Section \ref{SEC:PRELIM}, we recall some results for conditional events, compound conditionals, conditional random quantities, p-consistency, and p-entailment. In Section 
\ref{SEC:TRIVLOG}, we firstly show the equivalence between a bet on a conditional event and a conditional bet, and secondly review the trivalent de Finetti  analysis of conditionals. We  examine two recent articles that explore trivalent logics for conditionals and their definitions of logical validity and compare them with our approach to compound conditionals. In Section
\ref{SEC:PROBDEDTHM}, we  show that the  deduction theorem does not follow for the conditional events when using p-entailment, and then we obtain some probabilistic versions of the deduction theorem, with further results and examples. 
In Section \ref{SEC:IMPEXP}
we first focus on iterated conditionals and the invalidity of the Import-Export principle for the conditional event. In particular, we consider the 
inference from a disjunction, \emph{$A$ or $B$}, to a conditional, \emph{if not-$A$ then $B$}, as an example  and explain how the invalidity of the classical deduction theorem is related to the invalidity of the Import-Export principle for the conditional event.
Then,  we introduce  a \emph{General Import-Export principle} 
in relation to iterated conditionals
 and we give a result which relates it to p-consistency and p-entailment. We also illustrate the validity of the General Import-Export principle in some inference rules of System P
where the  probabilistic deduction theorem is not appliable.  In Section \ref{EQ:FURTHER} we give some further comments on the p-entailment from a family  of conditional events $\F$ to a conditional event $E|H$ and the p-entailment from $\F\cup\{H\}$ to the event $E$. 
In Section \ref{SEC:RW} we briefly compare our approach to some related work on conditionals.
Finally, in Section \ref{SEC:CONCL}  we give some concluding remarks.

\section{Preliminary notions and results}
\label{SEC:PRELIM}
Uncertainty about real facts will be formalized here by judgments about events. 
In formal terms, an event $A$  is a two-valued logical entity:  \emph{true}, or \emph{false}. 
The \emph{indicator} of $A$, denoted by the same symbol, is  1, or 0, according to  whether $A$ is true, or false, respectively. 
We  denote by
$\Omega$ the sure event and by $\emptyset$ the impossible one.
We denote by $A\land B$ (resp., $A\vee B$), or simply by $AB$, the  conjunction (resp., disjunction) of $A$ and $B$. By $\no{A}$ we denote the negation of $A$.  Given two events $A$ and $B$ we say that $A$ logically implies $B$, denoted by   $A \subseteq B$, when $A\no{B}=\emptyset$, or equivalently  $\no{A}\vee B=\Omega$.  Notice that the symbol $\subseteq$ will be also used to denote the inclusion relation between two sets. Therefore the interpretation of  $\subseteq$ will be context-dependent.
Moreover, we use the symbol $\Longleftrightarrow$ to denote the logical equivalence {\em if and only if}.

Given two events $A$ and $H$, with $H \neq \emptyset$, the conditional event $A|H$  is a three-valued logical entity which is \emph{true}, or
\emph{false}, or \emph{void}, according to whether $AH$ is true, or
$\no{A}H$ is true, or $\no{H}$ is true, respectively. The negation $\no{A|H}$ of $A|H$ is defined as $\no{A}|H$.

In the betting framework, to assess $P(A|H)=x$ amounts to saying that, for every real number $s$,  you are willing to pay 
an amount $s\,x$ and to receive $s$, or 0, or $s\, x$, according
to whether $AH$ is true, or $\no{A}H$ is true, or $\no{H}$
is true (in this case the bet is called off), respectively. Hence, for the random gain $G=sH(A-x)$, the possible values are $s(1-x)$, or $-s\,x$, or $0$, according
to whether $AH$ is true, or $\no{A}H$ is true, or $\no{H}$
is true, respectively. \\
We denote by $X$ a \emph{random quantity}, that is  an 
uncertain real quantity,  which has a well determined but unknown value, and we use the symbol $\prev$ for the prevision. 
We assume that  $X$ has a finite set of possible values. Given any event $H\neq \emptyset$,  by the betting analysis, if you  assess that the prevision of $``X$ {\em conditional on} $H$'' (or short:  $``X$ {\em given} $H$''), $\pr(X|H)$, is equal to $\mu$, this means that for any given  real number $s$ you are willing to pay an amount $s\mu$ and to receive  $sX$, or $s\mu$, according  to whether $H$ is true, or  false (bet  called off), respectively. 
We recall that the \emph{coherence} of a conditional prevision assessment on  a  family of conditional random quantities means  that \emph{in any finite combination of $n$ bets,  after discarding the case where all the bets are called off, it cannot happen that the values  of the random gain  are all positive, or all negative}.  If you are incoherent a Dutch book can be made against you, i.e. there exists a finite combination of $n$ bets, ensuring that you suffer  a loss in all the cases where at least a bet is not called off.

In particular, when $X$ is (the indicator of) an event $A$, then $\prev(X|H)=P(A|H)$.
Given a conditional event $A|H$
with  $P(A|H) = x$,
the indicator of $A|H$, denoted by the same symbol, is
\begin{equation}\label{EQ:AgH}
	A|H=
	AH+x \no{H}=AH+x (1-H)=\left\{\begin{array}{ll}
		1, &\mbox{if $AH$ is true,}\\
		0, &\mbox{if $\no{A}H$ is true,}\\
		x, &\mbox{if $\no{H}$ is true.}\\
	\end{array}
	\right.
\end{equation} 
Notice that  $\prev(AH + x\no{H})=xP(H)+xP(\no{H})=x$. 
The third  value of the random quantity  $A|H$  (subjectively) depends on the assessed probability  $P(A|H)=x$. 
When $H\subseteq A$ (i.e., $AH=H$), it holds that $P(A|H)=1$; then,
for the indicator $A|H$ it holds that 
\begin{equation}\label{EQ:AgH=1}
	A|H=AH+x\no{H}=H+\no{H}=1, \;\; (\mbox{when }
	H\subseteq A). 
\end{equation}
Likewise, if $AH=\emptyset$ (and $H\neq \emptyset$), it holds that $P(A|H)=0$; then
\begin{equation}\label{EQ:AgHzero}
A|H=0+0\no{H}=0, \;\; (\mbox{when }
AH=\emptyset). 
\end{equation}
For the indicator of the negation  of $A|H$ it holds that  $\no{A}|H=1-A|H$.
Given a random quantity $X$ and an event $H \neq \emptyset$, 
with a conditional prevision assessment $\prev(X|H) = \mu$, in our approach, likewise formula (\ref{EQ:AgH}), the conditional random quantity $X|H$ is defined as 
\begin{equation}\label{EQ:XgH}
	X|H=XH+\mu\no{H}.
\end{equation}
Notice that  the prevision of the conditional random quantity  $X|H$ coincides with the conditional prevision $\mu$, indeed
\begin{equation}\label{EQ:PREVXgH}
	\prev(XH + \mu\no{H})=\prev(XH)+\mu P(\no{H})=\mu P(H)+\mu P(\no{H})=\mu.
\end{equation}
For a discussion on this extended notion of a conditional random quantity and on the notion of coherence of a conditional prevision assessment see, e.g., \cite{GiSa14,GiSa20,SGOP20}. 

\begin{remark}\label{REM:CONST} Given a random quantity $X$ and an event $H \neq \emptyset$,
	if $XH=cH$ for a suitable constant $c$,  then $\prev(X|H)=c$ and hence  $X|H=cH+c\no{H}=c$. In particular, for any event $A\neq \emptyset$, $A|A=1$ and $\no{A}|A=0$.
\end{remark}

\subsection{Compound conditionals}
We recall below the notion of conjunction of two conditional events (\cite{GiSa14}).
\begin{definition}\label{CONJUNCTION}
	Given a coherent probability assessment $P(A|H)=x, P(B|K)=y$, the conjunction of  $A|H$ and $B|K$ is defined as 
	\[
	(A|H) \wedge (B|K) = \left\{
	\begin{array}{ll}
		1, & \mbox{if $AHBK$ is true,} \\
		0, & \mbox{if $\no{A}H \vee \no{B}K$ is true,} \\
		x, & \mbox{if $\no{H}BK$ is true,} \\
		y, & \mbox{if $AH\no{K}$ is true,} \\
		z, & \mbox{if $\no{H}\no{K}$ is true,} \\
	\end{array}
	\right.
	\]
	that is,
	\begin{equation}\label{EQ:CONJRQ}
		(A|H) \wedge (B|K) = AHBK + x \no{H}BK + y AH\no{K} + z \no{H}\no{K} \,,
	\end{equation}
	where $z$ is the prevision of $(A|H) \wedge (B|K)$.
\end{definition}
We require that  $(x,y,z)$ be coherent.   Notice that, differently from conditional events which are three-valued objects, the conjunction $(A|H) \wedge (B|K)$
is no longer a three-valued object, but  a five-valued object with values in $[0,1]$.
As for the conditional event, the values $1$, $0$, and  $z$, in a conditional bet where you pay $z$, correspond to  the cases \emph{win}, \emph{lose}, and \emph{money back}, when both conjuncts are true, at least one is false, or both conjuncts are void, respectively. The additional two values $x$ and $y$ stem from the fact that we need also to consider the two cases of \emph{partial win}, where one conjunct is true and the other one is void. 
In these last two cases, one can also say  that  the conjunction  is \emph{partially true} (\cite{Cantwell22}).

As it will be shown in Remark (\ref{REM:CONGCRQ}), the conjunction $(A|H) \wedge (B|K)$ coincides with the conditional random quantity $(AHBK + x \no{H}BK + y AH\no{K})|(H\vee K)$. 
We recall below the  notion of conjunction  of $n$ conditional events.
\begin{definition}\label{DEF:CONGn}	
	Let  $n$ conditional events $E_1|H_1,\ldots,E_n|H_n$ be given.
	For each  non-empty strict subset $S$  of $\{1,\ldots,n\}$,  let $x_{S}$ be a prevision assessment on $\bigwedge_{i\in S} (E_i|H_i)$.
	Then, the conjunction  $(E_1|H_1) \wedge \cdots \wedge (E_n|H_n)$ is the conditional random quantity $\C_{1\cdots n}$ defined as
	\begin{equation}\label{EQ:CF}
		\begin{array}{lll}
			\C_{1\cdots n}
			=\left\{
			\begin{array}{llll}
				1, &\mbox{ if } \bigwedge_{i=1}^n E_iH_i\, \mbox{ is true,} \\
				0, &\mbox{ if } \bigvee_{i=1}^n \no{E}_iH_i\, \mbox{ is true}, \\
				x_{S}, &\mbox{ if } (\bigwedge_{i\in S} \no{H}_i)\wedge(\bigwedge_{i\notin S} E_i{H}_i)\, \mbox{ is true}, \; \emptyset \neq S\subset \{1,2\ldots,n\},\\
				x_{1\cdots n}, &\mbox{ if } \bigwedge_{i=1}^n \no{H}_i \mbox{ is true},
			\end{array}
			\right.
		\end{array}
	\end{equation}
	where $x_{1\cdots n}=x_{\{1,\ldots, n\}}=\prev(\C_{1\cdots n})$.
	
\end{definition}
For  $n=1$ we obtain $\C_1=E_1|H_1$.  In  Definition \ref{DEF:CONGn}  each possible value $x_S$ of $\C_{1\cdots n}$,  $\emptyset\neq  S\subset \{1,\ldots,n\}$, is evaluated  when defining (in a previous step) the conjunction $\C_{S}=\bigwedge_{i\in S} (E_i|H_i)$. 
Then, after the conditional prevision $x_{1\cdots n}$ is evaluated, $\C_{1\cdots n}$ is completely specified. Of course, 
we require coherence for  the prevision assessment $(x_{S}, \emptyset\neq  S\subseteq \{1,\ldots,n\})$, so that $\C_{1\cdots n}\in[0,1]$.
In the framework of the betting scheme, $x_{1\cdots n}$ is the amount that you agree to pay with the proviso that you will receive:\\
- the amount $1$, if all conditional events are true;\\
- the amount  $0$, if at least one of the conditional events is false; \\
- the amount $x_S$ equal to the prevision of the conjunction of that conditional events which are void,  otherwise. In particular you receive back $x_{1\cdots n}$ when all  conditional events are void.\\
As we can see from (\ref{EQ:CF}), the conjunction $\C_{1\cdots n}$ assumes values in the interval $[0,1]$ and is (in general) a $(2^n+1)$-valued object because the number of nonempty subsets $S$, and hence the number of possible values $x_S$,  is $2^n-1$. 
\begin{remark}
	It can be verified that 
	\begin{equation}\label{EQ:PREVCN}
		\begin{array}{ll}
			x_{1\cdots n}=\prev[(\bigwedge_{i=1}^n E_iH_i+\sum_{\emptyset \neq S\subset \{1,2\ldots,n\}}x_{S}(\bigwedge_{i\in S} \no{H}_i)\wedge(\bigwedge_{i\notin S} E_i{H}_i))|(\bigvee_{i=1}^n H_i)]
		\end{array}
	\end{equation}
	and
	\[
	\C_{1\cdots n}=[\bigwedge_{i=1}^n E_iH_i+\sum_{\emptyset \neq S\subset \{1,2\ldots,n\}}x_{S}(\bigwedge_{i\in S} \no{H}_i)\wedge(\bigwedge_{i\notin S} E_i{H}_i)]|(\bigvee_{i=1}^n H_i).
	\]
	A similar comment can be done for  $x_S$ and $\C_S$, for each non empty subset $S\subset \{1,2\ldots,n\}$. In particular
	\begin{equation}
	(A|H)\wedge (B|K)=(AHBK + x \no{H}BK + y AH\no{K})|(H\vee K)=AHBK|(H\vee K)+  x \no{H}BK|(H\vee K) + y AH\no{K}|(H\vee K),
	\end{equation}
where $x=P(A|H)$, $y=P(B|K)$.
Then,
	\begin{equation}\label{EQ:PCONG}
	\prev[(A|H)\wedge (B|K)]=P[AHBK|(H\vee K)]+  P(A|H)P[\no{H}BK|(H\vee K)] + P(B|K) P[AH\no{K}|(H\vee K)].
	\end{equation}	
Notice that, when $P(H\vee K)>0$, formula (\ref{EQ:PCONG}) becomes the well known formula of McGee  (\cite{McGe89}) and Kaufmann (\cite{Kauf09})
\[
	\prev[(A|H)\wedge (B|K)]=	\frac{P(AHBK)+  P(A|H) P(\no{H}BK) + P(B|K) P(AH\no{K})}{P(H\vee K)}.
\]
A critical examination of claimed counterexamples to this notion of the conjunction of conditionals is given in \cite{Cantwell22}.
\end{remark}
We recall a result which shows that the prevision of the  conjunction on $n$  conditional events satisfies the Fréchet-Hoeffding bounds (\cite[Theorem13]{GiSa19}).
\begin{theorem}\label{THM:TEOREMAAI13}
	Let  $n$ conditional events $E_1|H_1,\ldots,E_{n}|H_{n}$ be given, with $x_i=P(E_i|H_i)$, $i=1,\ldots, n$  and $x_{1\cdots n}=\prev(\C_{1 \cdots n })$. Then
	\[
	\max\{x_1+\cdots+x_{n}-n+1,0\}
	\,\,\leq \,\, x_{1\cdots n} \,\,\leq\,\, \min\{x_1,\ldots,x_n\}.
	\]
\end{theorem}
In \cite[Theorem 10]{GiSa21} we have shown, under logical independence, the sharpness of the Fréchet-Hoeffding bounds.
\begin{remark}\label{REM:CONGCONG}
	Given a finite family  $\F$ of  conditional events,  their conjunction is also  denoted by $\C(\F)$. We recall that in \cite{GiSa19}, given  two finite families of conditional events $\F_1$ and $\F_2$,  the object $\C(\F_1) \wedge  \C(\F_2)$ is 
	defined as $\C(\F_1\cup \F_2)$. Then, conjunction satisfies the commutativity and associativity properties (\cite[Propositions 1 and 2]{GiSa19}). Moreover,  the operation of conjunction satisfies the monotonicity property  (\cite[Theorem7]{GiSa19}), that is
	$
	\C_{1\cdots n+1}\leq \C_{1\cdots n}.$
	Then,
	\begin{equation}\label{EQ:MONOTGEN}
		\C(\mathcal{F}_1\cup \mathcal{F}_2)\leq \C(\mathcal{F}_1),\;\;\C(\mathcal{F}_1\cup \mathcal{F}_2)\leq \C(\mathcal{F}_2).     
	\end{equation}
\end{remark}
\noindent \emph{Iterated conditioning.}
We now recall the notion of iterated conditional given in \cite{GiSa13a} (see also \cite{GiSa14}). 
Such notion has the structure 
\def\mcirc{\mathbin{\scalerel*{\bigcirc}{t}}}
$ \Box | \mcirc= \Box\wedge \mcirc +\prev(\Box|\mcirc)\no{\mcirc}$, 
where $\prev$ denotes the prevision,
which reduces to formula (\ref{EQ:AgH})  when $\Box=A$ and $\mcirc=H$.
\begin{definition}[Iterated conditioning]
	\label{DEF:ITER-COND} Given any pair of conditional events $A|H$ and $B|K$, with $AH\neq \emptyset$, the iterated
	conditional $(B|K)|(A|H)$ is defined as 
	\begin{equation}\label{EQ:ITER-COND}
		(B|K)|(A|H) = (B|K) \wedge (A|H) + \mu \no{A}|H,
	\end{equation}
	where
	$\mu =\mathbb{P}[(B|K)|(A|H)]$.
\end{definition}
\begin{remark}\label{REM:A|H=0}
	Notice that we assumed that $AH\neq \emptyset$ to give a nontrivial meaning to the notion of iterated conditional. Indeed,   if  $AH$ were equal to $\emptyset$,   then  $A|H=(B|K) \wedge (A|H)=0$ and  $\no{A}|H =1$, from which it  would follow  
	$(B|K)| (A|H)=(B|K)|0=(B|K) \wedge (A|H) + \mu \no{A}|H=\mu$; that is, $(B|K)| (A|H)$ 
	would coincide with the (indeterminate) value $\mu$. Similarly to the case of a conditional event $A|H$, which is of no interest when $H=\emptyset$ (in numerical terms $H=0$), the iterated conditional $(B|K)| (A|H)$ is not considered  in our approach when $AH=\emptyset$ (and $H\neq \emptyset$),  in which case, from (\ref{EQ:AgHzero}),  $A|H$ is constant and coincides with 0. Of course,  we do not consider the iterated conditional also when $H=\emptyset$, because in this case $A|H$ is not defined.
\end{remark}
Definition \ref{DEF:ITER-COND} has been generalized in \cite{GiSa19} (see also \cite{GiSa21E}) to the case where the antecedent is the conjunction of more than two conditional events.
\begin{definition}\label{DEF:GENITER}
	Let   $n+1$ conditional events $E_1|H_1, \ldots, E_{n+1}|H_{n+1}$ be given, with $\C_{1\cdots n}=(E_1|H_1) \wedge \cdots \wedge (E_n|H_n)\neq 0$. We denote by $(E_{n+1}|H_{n+1})|\C_{1\cdots n}$ the random quantity
	\[
	\begin{array}{ll}
		(E_1|H_1) \wedge \cdots \wedge (E_{n+1}|H_{n+1})  + \mu \, (1-(E_1|H_1) \wedge \cdots \wedge (E_n|H_n))= \\
		=\C_{1\cdots n+1}+\mu \, (1-\C_{1\cdots n}),
	\end{array}
	\]
	where $\mu = \prev[(E_{n+1}|H_{n+1})|\C_{1\cdots n}]$.
\end{definition}
We observe that, based on the betting analysis,  the quantity $\mu$ is the amount to be paid in order to receive the amount $\C_{1\cdots n+1}+\mu \, (1-\C_{1\cdots n})$. 
We also observe that, defining 	$\prev(\C_{1\cdots n})=x_{1\cdots n}$ and 
$\prev(\C_{1\cdots n+1})=x_{1\cdots n+1}$, by the linearity of prevision it holds that $\mu=x_{1\cdots n+1}+\mu \, (1-x_{1\cdots n})$; then, $x_{1\cdots n+1}=\mu \, x_{1\cdots n}$, that is (\emph{compound prevision theorem})
\begin{equation}\label{EQ:COMPPREVTHM}
	\prev(\C_{1\cdots n+1})=\prev[(E_{n+1}|H_{n+1})|\C_{1\cdots n}]\prev(\C_{1\cdots n}).
\end{equation} 
\subsection{Conditional random quantities and the notions of p-consistency and p-entailment}
We recall the notions of p-consistency and p-entailment for  conditional random quantities  which take values in a finite subset of $[0,1]$ (\cite{SPOG18}).

\begin{definition}\label{DEF:PC}
	Let $\mathcal{F}_n = \{X_i|H_i \, , \; i=1,\ldots,n\}$ be  a family of $n$  conditional random quantities which take values in a finite subset of $[0,1]$. Then, $\mathcal{F}_n$ is  {\em p-consistent} if and only if,
	the (prevision) assessment $(\mu_1,\mu_2,\ldots,\mu_n)=(1,1,\ldots,1)$ on $\mathcal{F}_n$ is coherent.
\end{definition}
\begin{definition}\label{DEF:PE}
	A p-consistent family  $\mathcal{F}_n = \{X_i|H_i \, , \; i=1,\ldots,n\}$ {\em p-entails} a conditional random quantity $X|H$ which takes values in a finite subset of $[0,1]$, denoted by $\mathcal{F}_n \; \Rightarrow_p \; X|H$, if and only if  for any  coherent (prevision) assessment $(\mu_1,\ldots,\mu_n,z)$ on $\mathcal{F}_n \cup \{X|H\}$: if $\mu_1=\cdots=\mu_n=1$, then  $z=1$.
\end{definition}
We say that the inference from a p-consistent family $\F_n$ to $X|H$ is \emph{p-valid} if and only if  $\mathcal{F}_n \; \Rightarrow_p X|H$.

\begin{remark}
	Notice that, if we consider conditional events instead of conditional random quantities, we 
	obtain the  notions of p-consistency, p-entailment,
	and p-validity given in the setting of coherence (see, e.g., \cite{biazzo02,gilio02,gilio10,gilio11ecsqaru,GiSa13IJAR}).
	We recall that  the notion of p-entailment in the setting of coherence (\cite[Definition 6]{gilio02}, \cite[Theorem 4.9]{biazzo02}) is 
	based on Adams's theory (\cite{adams75}).
	However, our analysis is different from Adams's in our treatment  of zero-probability antecedents. We recall that in his theory Adams by convention  defines $P(E|H)=1$ when $P(H)=0$. 
	From this convention, however, problematic consequences follow. For example, we would obtain that  the following condition
	\begin{equation}\label{EQ:CONTRAP}
		P(C|A)=1 \; \Longleftrightarrow \; P(\no{A}|\no{C})=1 
	\end{equation}
	is  satisfied, as shown below.\\
	$(\Rightarrow)$ If $P(C|A)=1$, it holds that $P(A)=P(C|A)P(A)=P(AC)$ and hence $P(A\no{C})=0$. Then, $P(\no{C})=P(\no{A}\no{C})$ and, in the case $P(\no{C})>0$, it  follows $P(\no{A}|\no{C})=\frac{P(\no{A}\,\no{C})}{P(\no{C})}=1$. In the case $P(\no{C})=0$, by the convention of Adams, it still follows $P(\no{A}|\no{C})=1$. \\
	$(\Leftarrow)$ If $P(\no{A}|\no{C})=1$, by a symmetrical reasoning, it  follows $P(C|A)=1$.
	
	 Adams could not accept Definition \ref{DEF:PE}, because his convention would then make Contraposition valid, and he knew that it should be invalid in a probabilistic approach. Note that, in the coherence-based analysis, the assessment  $(x,y)$ on $\{C|A, \no{A}|\no{C}\}$ is coherent, for every $(x,y) \in [0,1]^2$; thus, the condition (\ref{EQ:CONTRAP}) does not hold and hence Contraposition is not p-valid. 
	Of course, equation (\ref{EQ:CONTRAP})  is satisfied under the restriction $P(A)>0$ and $P(\no{C})>0$, that is under the assumption that the antecedents have positive probability. However,  the notion of p-entailment given in the setting of coherence does not include these restrictions and allows zero-probability  antecedents.
\end{remark}
We recall the quasi conjunction $QC(\F)$ of a family of conditional events $\mathcal{F}_n = \{E_1|H_1, \ldots, E_n|H_n\}$ is defined as the following conditional event
\begin{equation}
	QC(\F)=\bigwedge_{i=1}^n(\no{H}_i\vee E_iH_i)|\bigvee_{i=1}^n(H_i).
\end{equation}	
 A characterization of p-entailment by means of  the quasi conjunction is given below
 (\cite[Theorem 6]{GiSa13IJAR}, see also \cite[Definition 1]{DuPr94}).
\begin{theorem}\label{THM:ENTAIL-CS}{\rm
		Let be given a p-consistent family $\mathcal{F}_n = \{E_1|H_1, \ldots, E_n|H_n\}$ and a conditional event $E|H$. The following assertions are equivalent: \\
		1. $\mathcal{F}_n$ p-entails $E|H$; \\
		2. The assessment $\mathcal{P}=(1,\ldots,1,z)$ on $\mathcal{F}=\mathcal{F}_n \cup \{E|H\}$, where $P(E_i|H_i)=1,$ $i=1,\ldots,n, P(E|H)=z$, is coherent if and only if $z=1$; \\
		3. The assessment $\mathcal{P}=(1,\ldots,1,0)$ on $\mathcal{F}=\mathcal{F}_n \cup \{E|H\}$, where $P(E_i|H_i)=1,$ $i=1,\ldots,n, P(E|H)=0$, is not coherent; \\
		4. Either there exists a nonempty subset $\mathcal{S}$ of  $\mathcal{F}_n$  such that $\mathcal{QC}(\mathcal{S})$ implies $E|H$, or $H \subseteq E$. \\
		5.  There exists a nonempty subset  $\mathcal{S}$ of $\mathcal{F}_n$ such that $\mathcal{QC}(\mathcal{S})$ p-entails $E|H$.
}\end{theorem}
The next two results characterize p-consistency and p-entailment by exploiting the notion of conjunction
(\cite[Theorems 17 and 18 ]{GiSa19}).
\begin{theorem}\label{THM:PCC}
	A family of n conditional events $\F$  is p-consistent if and only if the prevision assessment $\prev[\C(\F)] = 1$ is coherent.
\end{theorem}
\begin{theorem}\label{THM:PENT}
	Let be given a p-consistent family of $n$ conditional events $\F$ and a further conditional event $E|H$. Then, the following assertions are equivalent:\\
	(i) $\F$ p-entails $E|H$;\\
	(ii) the conjunction $\C(\F \cup \{E|H\})$ coincides with
	the  conjunction $\C(\F)$;\\
	(iii) the inequality  $\C(\F)\leq \; E|H$ is satisfied.
\end{theorem}
Notice that in the inequality  $(iii)$ of Theorem (\ref{THM:PENT}) the symbol $E|H$ is the indicator of the corresponding conditional event.
We recall a  result where it is shown that the p-entailment of a conditional event $E|H$ from a p-consistent family $\F$ is equivalent to condition  $(E|H)|\mathcal{C}(\F)=1 $
(\cite[Theorem 7]{GiSa21E}).
\begin{theorem}\label{THM:PEITER}
	A p-consistent family $\F$ p-entails $E|H$ if and only if the iterated conditional $(E|H)|\mathcal{C}(\F)$ is equal to 1.
\end{theorem}
In particular,	given two (p-consistent) conditional events $A|H$ and $B|K$, it holds that  (\cite[Theorem 4]{GiPS20}) 
	\begin{equation}\label{EQ:p-entail-iter}
		A|H \text{ p-entails }  B|K\;\; \Longleftrightarrow\;\; (B|K)|(A|H)=1.
	\end{equation}

%
%
%
\section{Some notes on trivalent logics and our approach}
\label{SEC:TRIVLOG}
In this section we firstly  deepen the notions of  conditional bets and bets on conditionals. Secondly, we recall that  many papers  adopt the three-valued  approach of de Finetti for the study of conditionals. In particular, we compare and discuss  two recent papers which adopt a trivalent logic for conditionals, and we show why selected properties hold in our Level 2 framework, which do not hold in these other approaches  (\cite{EgRS20,LaBa21}).

\subsection{Conditional bets and bets on conditionals} \label{SUBSEC:CONDBETCOND}
In this section we show that, in order to make a conditional probability assessment $P(B|A)=x$, the notions of {\em bet on a conditional} and {\em conditional bet} are equivalent. Then, we examine a further equivalent scheme for such probability assessment.
In our approach, a conditional \emph{if $A$ then $B$} is represented by  the conditional event  $B|A$, which in a pioneering paper  by de Finetti (\cite{definetti36}) was defined as a  three-valued object with possible  values: {\em true, false, void}; thus we adopt the Equation, or conditional probability hypothesis (CPH), that is    $P(\emph{if } A\emph{ then }B) = P(B|A)$. 
We point out that the approach of de Finetti has been developed in many respects. In numerical terms $B|A$ has possible values 1, 0, and $P(B|A)$, which is useful, for instance, in order to extend the  geometrical approach for coherence checking to the case of conditional events. Moreover, the  coherence-based  approach allows to study in full generality the notion of p-entailment. We next consider the three equivalent betting schemes ($S_1$), ($S_2$), and ($S_3$)  in order to evaluate $P(B|A)$.
In the  paper of de Finetti (\cite[p. 186]{definetti36}) the assessment $P(B|A)=x$ is operatively based on the following first betting scheme, which corresponds to a {\em conditional bet}: 
\paragraph{Scheme $(S_1)$} You evaluate $P(B|A)$ when $A$ is uncertain. After $A$ is verified, the bet has effect  and you accept to pay $x$ in order to bet on $B$, by receiving 1 if $B$ is true, or 0 if $B$ is false. In the case where $A$ is not verified the bet has no effect, because there is no bet. Within the scheme $(S_1)$ we can say that \emph{if $A$, then I bet that $B$}, that is a bet on $B$ conditionally on $A$ being true. However, if $\no{A}$, then there is \emph{no bet} within the scheme~$(S_1)$.  \\
Then, by definition, the coherence of the assessment $P(B|A)=x$ is checked by (only) considering the cases where the bet is effective, that is  when $A$ is verified. \\ ~\\
In equivalent terms, scheme $(S_1)$ can be expressed as a bet on a  conditional by the following scheme (\cite[p. 145]{definetti37}): 
\paragraph{Scheme $(S_2)$} If you assess $P(B|A)=x$, then (before knowing the truth value of $A$) you accept to pay $x$, by receiving 1 if $AB$ is true, or 0 if $A\no{B}$ is true, or $x$ if $\no{A}$ is true (the bet is called off). For  the checking of  coherence only the cases in which the bet is not called off are considered (that is, the case when $A$ is \emph{false}, in which the bet is called off, is discarded).\\
Based on  the scheme $(S_2)$, we can introduce the indicator of $B|A$, denoted by the same symbol, defined as (see, e.g. \cite[Section 2.2]{GiSa21})
\[
B|A = AB + x \no{A} = \left\{\begin{array}{ll}
	1, &\mbox{if $AB$ is true,}\\
	0, &\mbox{if $A\no{B}$ is true,}\\
	x,  &\mbox{if $\no{A}$ is true.}	
\end{array}
\right.
\]
Then, when you assess $P(B|A)=x$, you accept to pay the amount $x$, by receiving the random quantity $B|A$. Of course, for the prevision of the indicator it holds that
\[
\pr(B|A)= \pr(AB + x \no{A})= P(AB) + x P(\no{A}) = x[P(A)+P(\no{A})] = x = P(B|A).
\]
Within the scheme $(S_2)$ we can speak of a bet on the conditional event $B|A$, or on the conditional {\em if} $A$ {\em then} $B$, and this bet is  equivalent to the conditional bet on $B$, supposing that $A$ is true (and nothing more). In other words the notions of  conditional bets and  bets on conditionals coincide, in agreement with the equivalence of $(S_1)$ and $(S_2)$.\\~\\
Now we will examine a further scheme $(S_3)$, in which we consider the random quantity $Y=AB+y\no{A}$, where by definition $y$ is the prevision of $Y$. In the betting framework $y$ is the amount to be paid in order to receive $Y$. You must assess $y$, by remembering that for the checking of coherence {\em you discard all the cases where you receive back the amount you paid, whatever the amount be}, that is you discard the case where $A$ is not verified.\\
Question: in what way does the  quantity $y$ (subjectively) depend on the events $A$ and $B$? \\ {\em Answer: It can be shown that $y=P(B|A)$.} \\
\paragraph{Scheme $(S_3)$} You have to assess the value $y$, which represents your prevision of the random quantity $AB + y\no{A}$; then, by the betting scheme, you agree to pay $y$ with the proviso to receive the random quantity $AB + y \no{A}$.\\
{\em Coherence condition for the scheme $(S_3)$:} in order to check the coherence of $y$, you must discard all the cases where you receive back $y$, whatever $y$ be. \\ \ \\
{\em For an individual who wants to assess $P(B|A)$ are the schemes $(S_2)$ and $(S_3)$ equivalent? In other words, is it the case that $y=x=P(B|A)$?} \\ 
{\em We show below that the answer is ``Yes", i.e. $(S_2)$ and $(S_3)$ are equivalent.} \\
We observe that, within the scheme $(S_2)$, you assess $P(B|A)=x$; then you pay $x$ and you receive $AB + x\no{A}$. Within the scheme $(S_3)$, you assess the prevision $y$ of $AB + y\no{A}$; then you pay $y$ with the proviso that you receive $AB + y\no{A}$. As a consequence, you agree to pay $x-y$ by receiving $(AB + x\no{A}) - (AB + y\no{A}) = (x-y)\no{A}$, which is equal to 0, or $x-y$, according to whether $A$ is true, or false, respectively. By the coherence condition  in the scheme $(S_3)$, the case where $A$ is false must be discarded because in this case you receive back the paid amount $x-y$ (whatever $x-y$ be). When $A$ is true you receive 0; then, in order that the assessment $x-y$ be coherent, it must be that $x-y=0$, that is $x=y$. \\ In conclusion, for the same individual, it is equivalent to evaluate $P(B|A)$ by the scheme $(S_1)$, or $(S_2)$, or $(S_3)$.
\begin{remark}\label{REM:CONGCRQ}
By the previous reasoning we can verify that the conjunction 
\[
(A|H) \wedge (B|K)=AHBK + x \no{H}BK + y AH\no{K} + z \no{H}\no{K},
\]
where $x=P(A|H)$, $y=P(B|K)$, and $z=\prev[(A|H) \wedge (B|K)]$, 
 coincides with the conditional random quantity $Z|(H \vee K)=
(AHBK + x \no{H}BK + y AH\no{K})|(H\vee K)$.  
 In the framework of the betting scheme $z$ is the amount to be paid (resp., to be received) in order to receive (resp., to pay) the random amount $(A|H) \wedge (B|K)$. Moreover,   in the case when $\no{H}\no{K}$ is true you receive  back (resp., give back)  $z$, whatever it be; thus, in order to check coherence, the case $\no{H}\no{K}$ must be discarded. 
By setting  $\mu=\prev(Z|(H \vee K))$, from (\ref{EQ:XgH}),  it holds that 
\begin{equation} \label{EQ:ZRQ}
	\begin{array}{l}
		Z|(H \vee K) =
		(AHBK + x \no{H}BK + y AH\no{K})(H\vee K) + \mu \no{H}\no{K} =\\
		=
		AHBK + x \no{H}BK + y AH\no{K} + \mu \no{H}\no{K} \,.
	\end{array}
\end{equation}
Moreover, for the random quantity
\[
D = (A|H) \wedge (B|K) - Z|(H \vee K) = (z - \mu) \no{H}\no{K} \,,
\]
it holds that $\prev(D)=z-\mu$.  Then,  in a bet on $D$
one should pay, for instance,  $z - \mu$ by receiving the random amount $D$, with the bet called off when $\no{H}\,\no{K}$ is true (indeed, in this case   one would receive back the paid amount $z-\mu$, whatever it be). We observe that $D$ is 0, or $z-\mu$, according to whether $H \vee K$ is true, or $\no{H}\no{K}$ is true, respectively. Therefore, when the bet is not called off it holds that $D=0$ and, by coherence, $\prev(D)=z-\mu=0$, that is $\pr[(A|H) \wedge (B|K)] =z=\mu= \pr[Z|(H \vee K)]$. Therefore,   from   (\ref{EQ:CONJRQ}) and  (\ref{EQ:ZRQ}) we obtain that 
\begin{equation}\label{EQ:CONJCRQ}
	(A|H) \wedge (B|K) = (AHBK + x \no{H}BK + y AH\no{K})|(H \vee K) \,.
\end{equation}

	By the previous comments  the conjunction  can be also  defined as the conditional random quantity in (\ref{EQ:CONJCRQ}) (as done in \cite[Definition 2]{GiSa21}), that is, by recalling  (\ref{EQ:PREVXgH}),
	\[
	(A|H) \wedge (B|K)=AHBK + x \no{H}BK + y AH\no{K} + z \no{H}\no{K}, \]
	where $x=P(A|H)$, $y=P(B|K)$,  
	$z=\prev[(AHBK + x \no{H}BK + y AH\no{K})|(H \vee K)]$.  
\end{remark}	


\subsection{Compound and iterated conditionals in the approach of de Finetti}
In this section we illustrate the notions of  compound and iterated conditionals in the trivalent logic  of de Finetti. 
By the  trivalent truth tables  in de Finetti (\cite{definetti36}) the conjunction of two conditional events $A|H$ and $B|K$, here denoted by $(A|H)\wedge_{df} (B|K)$, is the logical product between tri-events, which is \emph{true} when $AHBK$ is true, \emph{false} when $\no{A}H\vee \no{B}K$, and  \emph{void} otherwise. In other words, $(A|H)\wedge_{df}(B|K)$ is true when all conditional events are true, is false when at least a conditional event is false, and is void otherwise.
Thus
\begin{equation}\label{EQ:DEFC}
	(A|H)\wedge_{df}(B|K)=AHBK|(AHBK \vee \no{A}H \vee  \no{B}K),
\end{equation}
and hence  $P[(A|H)\wedge_{df}(B|K)]$, that is $P[AHBK|(AHBK \vee \no{A}H \vee  \no{B}K)]$, is the probability that the conjunction is true, given that it is true or false.
In particular,  when $P(AHBK \vee \no{A}H \vee  \no{B}K)>0$, it holds that 
\[
P[(A|H)\wedge_{df}(B|K)]=\frac{P(AHBK)}{P(AHBK \vee \no{A}H \vee  \no{B}K)},
\]
which is the probability that the conjunction is true divided by  the probability  that it is true or false.

The notion of logical inclusion among events has been generalized to conditional events by Goodman and Nguyen in \cite{GoNg88}. 
Given two conditional events  $A|H$ and $B|K$,  $A|H$ logically implies $B|K$, denoted by $A|H \subseteq B|K$, if and only if $AH$ logically implies $BK$  and $\no{B}K$ logically implies $\no{A}H$, that is (\cite[formula (3.18)]{GoNg88})
%
\begin{equation}\label{EQ:GN}
	A|H \subseteq B| K\;\; \Longleftrightarrow \;\;
	AH\subseteq BK \text{ and }  \no{B}K\subseteq \no{A}H.
\end{equation}

\begin{remark}
We recall that, when $A|H \subseteq B|K$, it holds that $P(A|H)\leq P(B|K)$, for every conditional probability $P$ (see, e.g. \cite[Theorem 6]{gilio13}). This means that the set of coherent assessments $(x,y)$ on $\{A|H,B|K\}$, where $x=P(A|H),y=P(B|K)$, is characterized by the property that $x \leq y$; that is to say, there does not exist any coherent assessment $(x,y)$ such that $x>y$. In terms of comparative, or qualitative, probabilities (see, e.g., \cite{BiGS03Wupes,Cole90,Cole94,definetti31,definetti37,KrPS59}), if $A|H \subseteq B|K$, then ``$A$ given $H$ is no more probable than $B$ given $K$''. 
A computational analysis of comparative conditional probabilities in the framework of coherence has been given in \cite{mundici21}.
In \cite{Delgrande2019} a logic of qualitative probability has been developed, where the binary operator $\preceq$ describing the relation "is no more probable than" is regarded as an operator on finite sequences of formulas.
\end{remark}
We observe that the conjunction of Goodman and Nguyen (\cite[formula (3.4)]{GoNg88}) coincides with $\wedge_{df}$. Then, (see \cite[formula (3.17)]{GoNg88}, see also \cite[Remark 1]{SUM2018S}) 
\begin{equation}\label{EQ:GNDEF}
	(A|H)\wedge_{df}(B|K)=A|H \;\;	 \Longleftrightarrow \;\; A|H \subseteq B| K.
\end{equation}	
This trivalent notion of conjunction of de Finetti coincides with the Kleene-Lukasiewicz-Heyting, or KLH,  conjunction (see, e.g., \cite{BPOT18}). 
Such trivalent logic has been used for studying the connections among  belief revision, defeasible conditionals and nonmonotonic inference in \cite{Kern-Isb2023}.
We also observe that the Fréchet-Hoeffding  bounds are not preserved by $\wedge_{df}$. Indeed, under logical independence of the basic events, the lower and upper bounds for $z=P((A|H)\wedge_{df}(B|K))$
 are  $z'=0$ and $z''=\min\{P(A|H),P(B|K)\}$, respectively (\cite[Theorem 3]{SUM2018S}),   while the  Fréchet-Hoeffding lower bound
is $\max\{P(A|H)+P(B|K)-1,0\}$.
We  observe that, if the Fréchet-Hoeffding  bounds were valid,  the \emph{And} rule would be also satisfied, that is  $\{A|H,B|H\}$ would p-entail 
 $(A|H)\wedge_{df}(B|K)$. On the contrary, with de Finetti's trivalent conjunction,  when there are some logical dependencies, it may happen that the only coherent extension on the  conjunction of conditional events is zero, even if we assign probability 1 to each  conditional event.  
 That the And rule is invalid is counterintuitive as its validity is also required by the basic nonmonotonic reasoning System P (\cite{gilio02,KrLM90}).   We now illustrate the 
 invalidity of the And rule, when using the conjunction $\wedge_{df}$, by two examples.
\begin{example}\label{EX:AND1}
	We show a counterintuitive aspect of  de Finetti's trivalent conjunction.
	A double-headed coin is going to be either tossed or spun. Consider the events
	$H=$``the coin is tossed'', $\no{H}=$``the coin is spun'', and $A=$``the coin comes up heads''. Of course $P(A|H)=P(A|\no{H})=1$, but 
	as $H\no{H}=\emptyset$ it holds  that $(A|H)\wedge_{df}(A|\no{H})=AH\no{H}|(AH\no{H}\vee \no{A}H\vee \no{A}\,\no{H})=\emptyset|\no{A}$, and hence
 $P[(A|H)\wedge_{df}(A|\no{H})]=P(\emptyset|\no{A})=0$.
	Thus, with de Finetti's conjunction, as the coin cannot be both tossed and spun at the same time, the compound	``if the coin is tossed it will come up heads and if the coin is spun it will come up heads'' has, counterintuitively, a probability of 0, even if both conjuncts,   ``if the coin is tossed it will come up heads'' and ``if the coin is spun it will come up heads'', have probability 1.	Notice that  this problem is avoided when  using our Definition \ref{CONJUNCTION}. Indeed,   by Theorem \ref{THM:TEOREMAAI13} Fréchet-Hoeffding  bounds are satisfied, and hence  $\prev[(A|H)\wedge (A|\no{H})]=1$ when $P(A|H)=P(A|\no{H})=1$.
\end{example}
Example \ref{EX:AND1} is interesting for connexive logic, and its probabilistic semantics (\cite{PS21connexive}),  as it concerns a special case of what is negated in Aristotle's Second Thesis (\cite{PfSa23SL}).
\begin{example}
	In \cite[Example 2]{GiSa13IJAR}, given any logically independent events $A,B,C,D$, it has been shown that the family $\{C|B,B|A,A|(A\vee B),B|(A\vee B),D|\no{A}\}$ is p-consistent and p-entails $C|A$.
	In this example we consider the sub-family  $\F=\{C|B,B|A,A|(A\vee B),D|\no{A}\}$, which of course is p-consistent too. Then, the assessment 
	\begin{equation}\label{EQ:EX1}
		P(C|B)=P(B|A)=P(A|(A\vee B))=P(D|\no{A})=1
	\end{equation}
	is coherent.
	We observe that
	\[
	(C|B)\wedge_{df}(B|A)\wedge_{df}(A|(A\vee B))=ABC|(ABC\vee B\no{C}\vee A\no{B}\vee \no{A}B)=ABC|(A\vee B)
	\] 
	and 
	\[
	\begin{array}{ll}
		(C|B)\wedge_{df}(B|A)\wedge_{df}(A|(A\vee B))\wedge_{df}(D|\no{A})=(ABC|(A\vee B))\wedge_{df}(D|\no{A})=\\
		=ABC\no{A}D|(ABC\no{A}D \vee AB\no{C}\vee A\no{B}\vee \no{A}B \vee \no{A}\no{D})
		=\emptyset| (AB\no{C}\vee A\no{B}\vee \no{A}B \vee \no{A}\no{D}).
	\end{array}
	\]
	Then, the unique coherent extension $z$ of the assessment (\ref{EQ:EX1}) on $(C|B)\wedge_{df}(B|A)\wedge_{df}(A|(A\vee B))\wedge_{df}(D|\no{A})$ is $z=0$. As we can see, with the trivalent notion of conjunction of de Finetti, the well known And rule does not hold. On the contrary, by using our Definition \ref{CONJUNCTION},  the And rule is valid (\cite[Remark 4]{SPOG18}).
\end{example}
As another counterintuitive  aspect, given any event $A$, with $A\neq \emptyset$ and $A\neq \Omega$, 
coherence requires that both   conditional events $A|A$,  $(\no{A}|\no{A})$ have probability 1 
and hence, if the And rule  were valid,  their conjunction
 $(A|A)\wedge_{df}(\no{A}|\no{A})$ 
 would  have probability 1. Actually,  this conjunction does not exist because $(A|A)\wedge_{df}(\no{A}|\no{A})=\emptyset|\emptyset$. Notice that, by using Definition \ref{CONJUNCTION}, in agreement with the intuition, in our approach it holds that $(A|A)\wedge(\no{A}|\no{A})=1$.  This example will  be also examined at the end of  Section  \ref{SEC:EGRE}.

\begin{remark}
	We also observe that, differently from 
	our Level 2 approach based on the later development in de Finetti's thought, de Finetti's early trivalent notions of conjunction and disjunction do not satisfy  the probabilistic sum rule. Indeed, by exploiting De Morgan Laws (which hold in de Finetti approach), 
	\[
	(A|H)\vee_{df}(B|K)=1-(\no{A}|H)\wedge_{df}(\no{B}|K).
	\]
Assuming  $P(A|H)=P(B|K)=1$, the extension $P((A|H)\wedge_{df} (B|K))=0$ is coherent (\cite[Theorem 3]{SUM2018S}).
	Moreover,  $P(\no{A}|H)=P(\no{B}|K)=0$ and  the unique coherent extension on $(\no{A}|H)\wedge_{df}(\no{B}|K)$ is $P((\no{A}|H)\wedge_{df}(\no{B}|K))=0$  (\cite[Theorem 3]{SUM2018S}). As a consequence, the unique coherent extension on $(A|H)\vee_{df}(B|K)$ is $P((A|H)\vee_{df}(B|K))=1$.
	Then, the assessment $P(A|H)=P(B|K)=1$, $P((A|H)\wedge_{df} (B|K))=0$  and $P((A|H)\vee_{df}(B|K))=1$, is coherent, with
	\[
	P((A|H)\vee_{df}(B|K))=1\neq 2=P(A|H)+P(B|K)-P((A|H)\wedge_{df}(B|K)).
	\]
\end{remark}

Concerning the iterated conditional $(B|K)|_{df}(A|H)$, 
de Finetti's trivalent approach implies that the Import-Export principle is valid both for antecedent-nested and consequent-nested conditionals, that is it holds that $B|_{df}(A|H)=B|AH$ and $(B|K)|_{df}A=B|AK$, respectively. Indeed, de Finetti defines the iterated conditional  as the conditional event
\begin{equation}
	\label{EQ:DFITER}
	(B|K)|_{df}(A|H)=B|AHK;
\end{equation}
thus  $B|_{df}(A|H)=B|AH$,  when $K=\Omega$,  and 
$(B|K)|_{df}A=B|AK$, when   $H=\Omega$.

We remark that, differently from our approach,  the  iterated conditional in (\ref{EQ:DFITER}) does not satisfy the relation 
\[
P((A|H)\wedge_{df}(B|K))=P[(B|K)|_{df}(A|H)]P(A|H).
\]
In other words, within the approach of de Finetti,  for  the iterated conditional defined in (\ref{EQ:DFITER})
the compound probability theorem does not hold. Indeed, in general 
\[
\begin{array}{ll}
	P((A|H)\wedge_{df}(B|K))=P[AB|(AHBK \vee \no{A}H \vee  \no{B}K)]\\\neq P(B|AHK)P(A|H)=P[(B|K)|_{df}(A|H)]P(A|H). 
\end{array}
\]
For instance, by assuming that $A,H,B,K$ are stochastically independent, with  
$P(AHBK \vee \no{A}H \vee  \no{B}K)>0$, $P(AHK)>0$, it holds that
\[
\begin{array}{ll}
	P((A|H)\wedge_{df}(B|K))=
	\frac{P(ABHK)}{P(AHBK \vee \no{A}H \vee  \no{B}K)}=P(A)P(B)\frac{P(H)P(K)}{P(AHBK \vee \no{A}H \vee  \no{B}K)},
\end{array}
\]
while
\[
\begin{array}{ll}
	P[(B|K)|_{df}(A|H)]P(A|H)=
	P(B|AHK)P(A|H)=P(A)P(B),
\end{array}
\]
because $P(B|AHK)=P(B)$ and $P(A|H)=P(A)$.

\subsection{Compound and iterated conditionals in our approach}
We remark that in our Level 2 approach compound and iterated conditionals satisfy all the basic  properties relative to unconditional events, as illustrated in \cite{GiSa21A}. For instance: \\
\ \\
(a) the Fr\'echet-Hoeffding lower and upper prevision bounds for the conjunction of conditional events still hold;\\
(b) the inequalities $\max\{A+B-1,0\}\leq AB\leq \min\{A,B\}$ become
\[
\max\{A|H+B|K-1,0\}\leq (A|H)\wedge (B|K)\leq \min\{A|H,B|K\} \,,
\]
and the inequalities $ \max\{A,B\} \leq  A\vee B\leq \min\{A+B,1\}$ become 
\[
\max\{A|H,B|K\}\leq (A|H)\vee (B|K)\leq \min\{A|H+B|K,1\} \,;
\]
\\
(c) by defining $\overline{(A|H)\wedge(B|K)} =1-(A|H)\wedge(B|K)$ and $\overline{(A|H)\vee (B|K)} =1-(A|H)\vee(B|K)$, De Morgan’s Laws are satisfied (\cite{GiSa19}), that is
\[
\overline{(A|H)\wedge(B|K)}=(\no{A|H})\vee(\no{B|K}) \text{ and } \overline{(A|H)\vee(B|K)}=(\no{A|H})\wedge(\no{B|K})\,;
\]
(d) the inclusion-exclusion formula for the disjunction of conditional events is valid (\cite{GiSa20}); for instance, the equality
\[
E_1\vee E_2=E_1+E_2-E_1E_2
\]
becomes
\[
(E_1|H_1)\vee(E_2|H_2) =(E_1|H_1)+(E_2|H_2)-(E_1|H_1)\wedge(E_2|H_2) \,;
\]
(e) we can introduce the set of (conditional) constituents, with properties analogous to the case of unconditional events
(\cite{GiSa20}); \\
(f) by exploiting conjunction we obtain a characterization of the probabilistic entailment of Adams for conditionals; moreover, by exploiting iterated conditionals,
the p-entailment of $E|H$ from a p-consistent family $\F$ is characterized by the property that the iterated conditional $(E|H)|\C(\F)$, where $\C(\F)$ is the conjunction of the conditional events in $\F$, is constant and coincides with~1 (\cite{GiSa21E}, see also \cite{GiPS20}); \\
(g) we give a generalization of probabilistic modus ponens to conditional events (\cite{SaPG17}) and we study one-premise and two-premise centering inferences (\cite{GOPS16,SPOG18}); \\
(h) we explain some intuitive probabilistic assessments
discussed in 
\cite{douven11b}
by making some implicit background information explicit (\cite{SGOP20}); \\
(i) finally, the compound probability theorem, that is $P(AB)=P(B|A)P(A)$,  is generalized  by the compound prevision theorem 
\begin{equation}\label{EQ:COMPREV}
	\prev[(B|K)\wedge (A|H)]=\prev[(B|K)|(A|H)]P(A|H),
\end{equation}
which is a particular case of (\ref{EQ:COMPPREVTHM}). 

\subsection{The generalized Equation and the negation for iterated conditionals} 
The hypothesis that the probability of a conditional in natural language is the conditional probability has so many deep consequences for understanding reasoning in natural language that it has been called the \emph{Equation} (\cite{edgington95}): 
\begin{equation}\label{EQ:EQUATION2}
P(\text{if } A, \text{ then } B) = P(B|A).
\end{equation}
The Equation becomes the \emph{conditional probability hypothesis} (CPH)  in the psychology of reasoning, where it has been highly confirmed as descriptive of people's probability judgments about a wide range of conditionals (see, e.g., \cite{OHEHS07}, see also \cite{kleiter18}).   

We observe that  $P(B|A)=P(AB|(AB\vee A\no{B}))$, then from (\ref{EQ:EQUATION2})
the probability of the conditional is the probability that it is true, given that it is true or false.
Moreover,  when $P(A)>0$ it holds that
\begin{equation}\label{EQ:EQUATION}
	P(\mbox{If $A$, then $B$}) = P(B|A) = \frac{P(AB)}{P(A)} = \frac{P(AB)}{P(AB)+P(A\no{B})} \,,
\end{equation}
that is, the probability of the conditional is the probability that it is true, divided by the probability that it is true or false. 

Given two conditionals  \emph{if
	$H$, then $A$} 
and \emph{if
	$K$, then $B$}, we represent them by the  conditional events $A|H$ and $B|K$, respectively. 
Then, we write the  iterated conditional 
\emph{if ($A$ given 
	$H$), then ($B$ given 
	$K$)}  as \emph{if $A|H$, then $B|K$}, and we  represent it by the iterated conditional $(B|K)|(A|H)$.  Based on  this representation, we identify the ``probability''  of  \emph{if $A|H$, then $B|K$}
with  the prevision of $(B|K)|(A|H)$, that is 
\begin{equation}
	P(\text{\emph{if $A|H$, then $B|K$}})=\prev[(B|K)|(A|H)],
\end{equation}
which can be called \emph{generalized  Equation}, or \emph{conditional prevision hypothesis}.
Now, by the reasoning below, we generalize formula (\ref{EQ:EQUATION}) to the  iterated conditional $(B|K)|(A|H)$.  We set $\mu = \pr[(B|K)|(A|H)]$ and $\nu = \pr[(\no{B}|K)|(A|H)]$, then
\[
(B|K)|(A|H) = (B|K)\wedge(A|H) + \mu \no{A}|H \,,\;\;\;\; (\no{B}|K)|(A|H) = (\no{B}|K)\wedge(A|H) + \nu \no{A}|H \,,
\]
and hence
\[
(B|K)|(A|H) + (\no{B}|K)|(A|H) = (B|K)\wedge(A|H) + (\no{B}|K)\wedge(A|H) + (\mu + \nu) \no{A}|H \,.
\]
Moreover, we set
\[
P(A|H)=x \,,\;\;P(B|K)=y \,,\;\; \pr[(B|K)\wedge(A|H)]=z \,,\;\; \pr[(\no{B}|K)\wedge(A|H)]=\eta \,.
\]
Then, by 
\cite[Proposition 1]{SaPG17} (see also  the decomposition formula  in \cite{GiSa20}), it holds that
\[
(B|K)\wedge(A|H) + (\no{B}|K)\wedge(A|H) = A|H \,.
\]
Indeed
\[
(B|K)\wedge(A|H) + (\no{B}|K)\wedge(A|H) = \left\{\begin{array}{ll}
	1, & \mbox{ if } AHBK \vee AH\no{B}K \vee AH\no{K} \mbox{ is true, } \\
	0, & \mbox{ if } \no{A}HBK \vee \no{A}H\no{B}K \vee \no{A}H\no{K} \mbox{ is true.} \\
	x, & \mbox{ if } \no{H}BK \vee \no{H}\no{B}K \mbox{ is true, } \\
	z+\eta, & \mbox{ if } \no{H}\no{K} \mbox{ is true, } \\
\end{array}
\right.
\]
that is
\[
(B|K)\wedge(A|H) + (\no{B}|K)\wedge(A|H) = \left\{\begin{array}{ll}
	1, & \mbox{ if } AH \mbox{ is true, } \\
	0, & \mbox{ if } \no{A}H \mbox{ is true.} \\
	x, & \mbox{ if } \no{H}K \mbox{ is true, } \\
	z+\eta, & \mbox{ if } \no{H}\no{K} \mbox{ is true, } \\
\end{array}
\right.
\]
moreover
\[
A|H = \left\{\begin{array}{ll}
	1, & \mbox{ if } AH \mbox{ is true, } \\
	0, & \mbox{ if } \no{A}H \mbox{ is true.} \\
	x, & \mbox{ if } \no{H} \mbox{ is true. } \\
\end{array}
\right.
\]
Then,  based on \cite[Theorem 4]{GiSa14}, we obtain $z+\eta = x$, that is
\begin{equation}\label{EQ:DECOMPPREV}
	\pr[(B|K)\wedge(A|H)] + \pr[(\no{B}|K)\wedge(A|H)]  = P(A|H),
\end{equation}
and hence $(B|K)\wedge(A|H) + (\no{B}|K)\wedge(A|H) = A|H$. Therefore, a formula like (\ref{EQ:EQUATION}) also holds for the iterated conditional $(B|K)|(A|H)$, that is when $P(A|H)>0$ from (\ref{EQ:COMPREV}) and   (\ref{EQ:DECOMPPREV}) it holds that
\begin{equation}
	\pr[(B|K)|(A|H)] = \frac{\pr[(B|K)\wedge(A|H)]}{P(A|H)} = \frac{\pr[(B|K)\wedge(A|H)]}{\pr[(B|K)\wedge(A|H)] + \pr[(\no{B}|K)\wedge(A|H)]} \,.
\end{equation}
In addition, as $\pr[(B|K)|(A|H) + (\no{B}|K)|(A|H)]=\mu + \nu$, it holds that
\[
(B|K)|(A|H) + (\no{B}|K)|(A|H) = (B|K)\wedge(A|H) + (\no{B}|K)\wedge(A|H) + (\mu + \nu) \no{A}|H =
\]
\[
= A|H + (\mu + \nu) \no{A}|H = \left\{\begin{array}{ll}
	1, & \mbox{ if } AH \mbox{ is true, } \\
	\mu + \nu, & \mbox{ if } \no{A}H \mbox{ is true,} \\
	x + (\mu + \nu)(1-x), & \mbox{ if } \no{H} \mbox{ is true, } \\
\end{array}
\right. \;\; = \;\; \left\{\begin{array}{ll}
	1, & \mbox{ if } AH \mbox{ is true, } \\
	\mu + \nu, & \mbox{ if } \no{A}H \vee \no{H} \mbox{ is true,} \\
\end{array}
\right.
\]
because
\[
x + (\mu + \nu)(1-x) = \pr[A|H + (\mu + \nu) \no{A}|H] = \pr[(B|K)|(A|H) + (\no{B}|K)|(A|H)] = \mu + \nu \,.
\] 
Then, by coherence, it holds that $\mu + \nu = 1$ and hence 
\[
(B|K)|(A|H) + (\no{B}|K)|(A|H) = 1 \,,\;\;
\mbox{or equivalently:} \;\; (\no{B}|K)|(A|H) = 1 - (B|K)|(A|H) \,.
\]
Finally,
by defining  the negation of 
$(B|K)|(A|H)$ as
$\no{(B|K)|(A|H)}=1-(B|K)|(A|H)$,  it follows that
\begin{equation}\label{EQ:GENNEG}
	\no{(B|K)|(A|H)}=(\no{B}|K)|(A|H),
\end{equation}
which reduces  to the well known relation $\no{B|A}=\no{B}|A$, when $H=K=\Omega$.
\subsection{Comments on a paper by    Lassiter and Baratgin}
In this section we make some comparisons  with a recent paper by  Lassiter and Baratgin   (\cite{LaBa21}), where a trivalent semantics of  nested conditionals is adopted.
Some comments on \cite{LaBa21} are:\\
- the authors argue that the counterexamples to and criticisms of the trivalent logic of de Finetti are resolved by taking into account the notion of genericity and the basic aspect that in de Finetti's theory events are ``singular''. \\
-  the validity of the Import-Export principle is accepted;\\
- no particular analysis and/or proposal is given for defining compound and/or iterated conditionals, under the constraint of preserving the basic probabilistic properties. \\
In our approach  the events are singular, not generic, as clearly stated by de Finetti. Moreover, the Import-Export principle for  consequent-nested conditionals, which we also call {\em consequent-nested conditional reduction}, is not valid, that is $(D|C)|A\neq D|AC$. The invalidity of this principle in our approach is illustrated, for instance,  by the following  example. 
\begin{example}
	\label{EX:HgAgH}
	Given two events $H$ and $A$,
	let us consider
	the sentence 
	\begin{equation}\label{EQ:EXITER1}
		\mbox{``if $H$, then (if $A$ then $H$)''},
	\end{equation}
	which we represents by the iterated conditional
	$(H|A)|H)$. If the Import-Export principle were valid, the sentence (\ref{EQ:EXITER1}) would be equivalent to the sentence 
	\begin{equation}\label{EQ:EXITERIE2}
		\mbox{``if ($H$ and $A$), then $H$'' }.
	\end{equation}
	Then, it would be  $\prev[(H|A)|H)]=P(H|AH)=1$.  However, within our approach, by setting 
	$P(A|H)=x$,
	$P(H|A)=y$ and $\prev[(H|A)|H)]=\mu$,  it holds that
	\[
	(H|A)|H=(H|A)\wedge H+\mu\, \no{H}=
	\left\{
	\begin{array}{ll}
		1, \mbox{ if } AH \mbox{ is true},\\
		y, \mbox{ if } \no{A}H \mbox{ is true},\\
		\mu, \mbox{ if } \no{H} \mbox{ is true}.
	\end{array}
	\right.
	\]
	By coherence $\mu \in[y,1]$, hence $\prev[(H|A)|H]\geq P(H|A)$, with necessarily  $\prev[(H|A)|H]=1$   when $P(H|A)=1$.  More precisely it holds that \[
	\mu=P(A|H)+yP(\no{A}|H)=x+y(1-x).
	\]
\end{example}
Another example on the invalidity of the Import-Export principle for consequent-nested conditionals is given  in Section \ref{SEC:ITTOOR}, where it is shown that $(B|\no{A})|(A\vee B)\neq B|(\no{A}\wedge (A\vee B))=B|\no{A}B$. In the next example 
we observe that an
Import-Export principle  for antecedent-nested conditionals, which we call
{\em antecedent-nested conditional reduction},
is  invalid too (\cite[Remark 7]{SGOP20}).
\begin{example}
	Let us consider
 an experiment of drawing a ball randomly from an urn. Let $A$ and $H$ be the following events:\\
	$A$=``The drawn ball is white'';\\
	$H$=``The urn contains 99 white balls and 1 black ball''.\\
	Of course $P(A|H)=0.99$.
	Then let us consider the sentence ``if $A$ when $H$,  then $H$'', that is the sentence
	\begin{equation}\label{EQ:EXITER}
		\mbox{``if (if $H$ then $A$), then $H$''},
	\end{equation}
	which we represent by the iterated conditional
	$H|(A|H)$. If the antecedent-nested conditional reduction were valid, the sentence (\ref{EQ:EXITER}) would be equivalent to the sentence 
	\begin{equation}\label{EQ:EXITERIE}
		\mbox{``if ($H$ and $A$), then $H$'' }.
	\end{equation}
	Then, it would be that  $\prev[H|(A|H)]=P(H|AH)=1$, which is absurd in the cases were $P(H)$ is low. On the contrary within our approach, by setting $P(A|H)=x$ and $\prev[H|(A|H)]=\mu$,  it holds that
	\[
	H|(A|H)=H\wedge (A|H)+\mu\, \no{A}|H=AH+\mu \no{A}|H=
	\left\{
	\begin{array}{ll}
		1, \mbox{ if } AH \mbox{ is true},\\
		\mu, \mbox{ if } \no{A}H \mbox{ is true},\\
		\mu(1-x), \mbox{ if } \no{H} \mbox{ is true}.
	\end{array}
	\right.
	\]
	By linearity of prevision it holds that
	\[
	\mu=P(AH)+\mu P(\no{A}|H)=xP(H)+\mu (1-x),
	\]
	that is $\mu x=xP(H)$, and when  $x>0$ it follows that $\mu=P(H)$. In the extreme case where $x=0$, it holds that  $H|(A|H)\in\{1,\mu\}$ and by coherence $\mu=1$. Indeed,   as  $A|H=AH+x\no{H}=AH$, it follows that  $H|(A|H)=H|AH=1$ and hence $\mu$=1. 
\end{example}
Another example on the invalidity of the antecedent-nested conditional reduction is obtained if we consider   the iterated sentence ``if (if $H$ then $H$), then $H$'', which we represent by the iterated conditional $H|(H|H)$.
\begin{example}
	Given any event $H\neq \emptyset$, if  the antecedent-nested conditional reduction were valid, we would obtain  that  $
	H|(H|H)=H|HH=1,
	$
	and $\prev[H|(H|H)]=1$.
	But, the equality $H|(H|H)=1$ cannot be valid in our approach. Indeed, if it were $H|(H|H)=1$, then from  
	(\ref{EQ:p-entail-iter})
	it  would follow that  $H|H$ would p-entail $H$,  in which case the unique coherent extension   $P(H)=x$ of the (unique coherent) assessment $P(H|H)=1$ should be $x=1$. On the contrary,  the assessment $(1,x)$ on $\{H|H,H\}$, with $H\neq \Omega$, is coherent for every $x\in[0,1]$.
	
	In our approach, in agreement with the intuition, 
	we can show  that $H|(H|H)=H$. Indeed, 
	by defining $\prev[H|(H|H)]=\mu$ and by observing that $H\wedge (H|H)=H$ and $\no{H}|H=0$, it holds that 
	\begin{equation}\label{EQ:SUE}
		\begin{array}{ll}
			H|(H|H)=H\wedge (H|H)+\mu(\no{H}|H)=H\;\;\; \mbox{and}\;\;\; \prev[H|(H|H)]=P(H).
		\end{array}
	\end{equation}
	For instance, by considering the  event $H=$ {\em Sue passes the exam},  under  the antecedent-nested conditional reduction
	the probability that {\em Sue passes the exam, conditionally on   Sue passes the exam,  given that she passes the exam},  should be 1, which is a very strange result. In our approach,  as a natural result, 
	the prevision of the iterated sentence  {\em if Sue passes the exam  when she passes the exam, then she passes the exam}  simply coincides with the probability that {\em Sue passes the exam}.  
\end{example}
\subsection{Trivalent semantics and validity: notes on a paper by Egr{\'e}, Rossi, and Sprenger}
\label{SEC:EGRE}
In this section we make a  comparison with  some results given in  \cite{EgRS20}. With respect to \cite{EgRS20}, we adopt a restricted approach where basic events, like $A$ and $C$, are two-valued objects, which are true or false. Then, we construct on them  conditional events which are true, or false, or void. We do not consider the case where basic events can be uncertain,  
because this uncertainty results from individual judgments (\cite{deFi80}).
\\
In  \cite{EgRS20} the truth values of a conditional 
$\emph{if } A \emph{ then } C$, true, or false, or void, are represented numerically by 1, or 0, or $\frac{1}{2}$, respectively. In particular, $\frac{1}{2}$ is a kind of objective representation of the logical value {\em void}.  The same representation is used for compound conditionals which are still interpreted as three-valued objects. Egré et al. (\cite{EgRS20}) remark, among many other aspects, that the choice of how ``grouping
indeterminate values with false antecedents ... is a classical point of contention between trivalent logics of conditionals''. Moreover, they observe that de Finetti does not handle conjunction and disjunction à la Bochvar/Weak Kleene (\cite{Bochvar37}), but that he is in line
with the strong Kleene scheme (\cite{BeKr01,Kleene38,Kleene52}). \\
In addition, in \cite{EgRS20} it is noted that ``de Finettian conditionals keep the epistemic notion of assertability and the semantic notion of truth separate, while allowing for a fruitful interaction: degrees of assertability can be defined directly in terms of the truth conditions''. Indeed, given a
simple conditional  $X$, its degree of assertability is defined as:
Ast$(X) = P(X \;\mbox{is true}|X \; \mbox{has a classical truth value})$.
The notion of assertion has a central importance in our approach too. Indeed, given a conditional $\emph{if } A \emph{ then } C$, in our analysis when for instance $P(C|A) = \frac{1}{2}$ we will not be inclined to assert the conditional as true nor as false. As $P(C|A)$ goes up to $0.6, 0.7, \ldots$, we'll become more and more inclined to assert $\emph{if } A \emph{ then } C$. But what does $\emph{if } A \emph{ then } C$ having a value of $\frac{1}{2}$ mean for assertion in \cite{EgRS20}? When do we become more inclined to assert $\emph{if } A \emph{ then } C$, according to them? Probably, their answer would be that $\emph{if } A \emph{ then } C$ having a value of $\frac{1}{2}$ only means that the conditional is `void', while its assertability is represented by $P(C|A)$. We can then observe that our approach is more natural because by $P(C|A)$ we represent both the logical value `void' and the assertability of the conditional. Moreover, similarly to the case of unconditional events where you pay $P(A)$ by receiving $A \in \{1,0\}$, if in a conditional bet you pay the amount $P(C|A)$ you will receive the random amount $C|A \in \{1,0,P(C|A)\}$. Also the close relationship between truth and assertability can be better explained in our approach. Let us consider, for instance, the following sentence, examined in \cite{EgRS20},
\begin{equation}\label{EQ:PPF}
	\text{if } \overbrace{\emph{Paul is in Paris}}^{A}, \text{ then} \overbrace{\emph{Paul is in France}}^{C},
\end{equation}
which typically is judged as true, even if Paul is for instance in Berlin, in which case the value of (\ref{EQ:PPF}) is \emph{void}. In   \cite{EgRS20} it is remarked that ``when we call sentences such as (\ref{EQ:PPF}) \emph{true}, what we really mean is that they are {\em maximally
	assertable}'', that is $P(C|A)=1$. Indeed, as $A$ logically implies $C$, it holds that the value 0 becomes impossible and $P(C|A)=1$; then, $C|A \in \{1,P(C|A)\}=\{1\}$, that is $C|A=1$, which is the natural way to explain the intuitive judgement that the sentence (\ref{EQ:PPF}) is true.
As another example (see  \cite[Conditional (3)]{EgRS20}), suppose Mary believes the
following conditional:
\begin{equation}\label{EQ:CCH}
	\text{If } 
	\overbrace{\emph{the Church is East of the City Hall}}^{A}, \text{ then } \overbrace{\emph{the City Hall is West of the Church}}^{C}.
\end{equation}
As claimed in \cite{EgRS20}, ``the proposition that Mary believes appears analytically true. Nonetheless,
on the de Finettian analysis its truth value depends on the position of the City Hall
with respect to the Church: the conditional may be evaluated either as true or as indeterminate.'' 
The apparent analyticity of (\ref{EQ:CCH}) is  explained because it is 
maximally assertable, regardless of its actual truth value. Indeed, the sentences $A$  and $C$ in (\ref{EQ:CCH}) are equivalent, and   the conditional $\emph{if } A \emph{ then } C$  is maximally assertable because $P(C|A)=P(A|A)=1$. 
\begin{remark}
	We add a further comment about the notion of assertability. For \cite{EgRS20}, given any event $A\neq \emptyset$, the conditional
	$\emph{if } A \emph{ then } A$ is not a logical truth, for it fails to be \emph{true} for them when $A$ is \emph{false} (and in the case where $A$ is a conditional event, also when $A$ has the \emph{middle}, or  $\frac12$, indeterminate truth value). Nevertheless, $\emph{if } A \emph{ then } A$ is always \emph{maximally assertable} for them because $P(A|A) = 1$. They make essentially this point in the example examined in  \cite[p. 194, sentence     (1)]{EgRS20}. 
 But, in their interpretation of  de Finetti's three-valued logic, we can derive maximum assertability also for the following nested conditionals:
	\[
	(a) \;\;\; \emph{\mbox{if (if $A$ then $A$) then $A$} }\,;\;\;\;\;\;\;
	(b) \;\;\; \emph{\mbox{if (if $\no{A}$ then $\no{A}$) then $\no{A}$}} \,.
	\]
	Notice that $(a)$ and $(b)$ do not make the de Finetti three-valued logic inconsistent. That is because 
	$\emph{if } A \emph{ then } A$ and 
	$\emph{if } \no{A} \emph{ then } \no{A}$  are not always \emph{true} in that system, and so we cannot  infer, by MP, that both $A$ and $\no{A}$ are always \emph{true} in the system. But by what \cite{EgRS20} say about \emph{assertability}, $(a)$ and $(b)$, and \emph{if} $A$ \emph{ then} $A$ and 
	\emph{if} $\no{A}$ \emph{then} $\no{A}$, should always be \emph{assertable} for them. According to \cite{EgRS20}'s account of assertability  (\cite[p. 193]{EgRS20}), the probability of a conditional, $\emph{if } A \emph{ then } C$, is the probability that it is true given that it is non-void, i.e. it has a \emph{classical} truth value. This is given by the cases in which $AC$ is true out of the cases in which $A$ is true. Now, considering for instance $(a)$, 
 (\emph{if} $A$ \emph{then} $A$) \emph{and} $A$
 is true whenever 
 \emph{if} $A$ \emph{then} $A$ is true, and so $(a)$ must have maximum assertability.
	By what \cite{EgRS20} hold, $(a)$ is assertable when $A$ is not. To fix the assertability of $(a)$, not-$A$ cases are irrelevant and are ignored, giving $(a)$ maximum assertability. As we see it, their position implies this result.
	Why then are not both $A$ and $\no{A}$ always assertable for them? It should always be the case that $C$ is assertable when 
 \emph{if} $A$ \emph{then} $C$ and $A$ are both assertable. 
 Otherwise, we have to follow \cite{McGe89} in claiming that MP is invalid, and \cite{EgRS20} appear disinclined to do this. \\
In formal terms, in de Finetti trivalent logic, 
 by recalling  (\ref{EQ:DFITER}), it holds that  $(B|K)|_{df}(A|H)=B|AHK$, and hence
	\[
	A|_{df}(A|A) \; = \; A|A, \,\;\;\no{A}|_{df}(\no{A}|\no{A}) \; = \; \no{A}|\no{A}.
	\]	
 Therefore,  both $(a)$ and $(b)$ are  maximally assertable. \\
	This problem can be avoided by taking our approach, where the iterated conditional sentence 
 \emph{if} (\emph{if} $H$ \emph{then} $A$), \emph{then} (\emph{if} $K$ \emph{then} $B$)
 is represented by the iterated conditional 
	\[
	(B|K)|(A|H) = (A|H)\wedge(B|K) + \mu \no{A}|H , 
	\]
	where $\mu$ is the prevision of $(B|K)|(A|H)$, which in the betting framework represents the amount to be paid in order to receive the random quantity $(B|K)|(A|H)$. We observe that, when $B=H=A$ and $K=\Omega$, it follows that $(A|H)\wedge(B|K)=(A|A)\wedge(A|\Omega) = A$ and $\no{A}|H=\no{A}|A=0$. Then the iterated conditional sentence becomes the sentence $(a)$, which we represent by the iterated conditional
	\[
	A|(A|A) = (A|A)\wedge A + \mu \no{A}|A = A  \,,\;\; \mbox{with} \;\; \mu = P(A) \,,
	\]
	and similarly the sentence $(b)$ is represented by the iterated conditional
	\[
	\no{A}|(\no{A}|\no{A}) = (\no{A}|\no{A})\wedge \no{A} + \eta A|\no{A} = \no{A}  \,,\;\; \mbox{with} \;\; \eta = P(\no{A}) \,.
	\]
	As we can see, both $(a)$ and $(b)$  are not maximally assertable in
our approach.
\end{remark}
Our approach differs, in particular, from that  of \cite{EgRS20}  in the notion of validity, which for us is given by the notion of p-validity of Adams,  except for his convention about conditionals whose antecedents have 0 probability. By considering, for instance, just one-premise inferences, we can define p-validity by saying that such an
inference is p-valid if and only if the conclusion has probability 1 whenever the premise has
probability 1. This definition does not imply, for us, that it is p-valid to infer $AC$ from $C|A$,
as $P(AC)$ will be 0 when $P(A) = 0$, while $P(C|A)$ could be  1 for us.
In \cite{EgRS20} the notion of validity where a one-premise inference is valid if and only if its conclusion has value 1 whenever the premise has value 1 is called $SS-$validity ({\em strict-to-strict} validity in \cite{CERV12}). This definition would imply for
them that it was ``valid'' to infer $AC$ from their conditional ``if $A$ then $C$''. When $A$ has value 0, for them ``if $A$ then $C$'' must have value $\frac{1}{2}$.
In our approach when $A$ is false it holds that  $C|A$ has value $P(C|A)$. Moreover, 
to infer $AC$ from $C|A$ is not p-valid because  $P(C|A)=1$  does not imply $P(AC)=1$; indeed the assessment $(1,y)$ on $\{C|A,AC\}$ is coherent for every $y\in[0,1]$.
Another definition of ``validity'' considered in \cite{EgRS20} is the {\em tolerant} one ($TT-$validity), which implies that an inference is invalid when its premise can have value 1/2 while its conclusion has value 0.  
We observe that to infer ``if $A$ then $C$'' from ``not-$A$ or $C$'' is
TT-valid, but not p-valid because $P(\no{A}\vee C)=1$ does not imply $P(C|A)=1$. Indeed 
the assessment $(1,y)$ on $\{\no{A}\vee C,C|A\}$ is coherent for every $y\in[0,1]$.
\\
In our approach compound conditionals are no longer three-valued conditionals, but they are introduced at Level 2 as suitable conditional random quantities with values in the unit interval $[0,1]$. We continue to  represent a conditional ``if $A$ then $C$'' by the conditional event $C|A$, which is a three-valued object; but, compound conditionals  have an increasing number of numerical  values. In particular, under logical independence of basic events, the conjunction of $n$ conditional events has $2^n+1$ possible values. For instance, given any logically independent events $A,H,B,K$, the conjunction of the  two conditionals ``if $H$ then $A$'' and ``if $K$ then $B$'' is a conditional random quantity, denoted by $(A|H)\wedge(B|K)$, with a set of $2^2+1=5$ possible values $\{1, 0, x, y, z\} \subset [0,1]$, where $x=P(A|H), y=P(B|K),$ and $z$ is the prevision of $(A|H)\wedge(B|K)$. The values $x, y, z$ are coherently assessed in a subjective way and hence, when specifying them, there are infinitely many ways of defining the conjunction. In particular, for $A|H$ there are infinitely many ways of (subjectively) assessing $x$ and hence, as a numerical object (the so-called indicator), we can (subjectively) define $A|H$ in infinitely many ways. \\
We can also briefly examine the Linearity principle, which we can represent by the disjunction $(A|B) \vee (B|A)$. This principle is validated by the  two main logics (CC/TT and DF/TT), studied in \cite{EgRS20}, and hence appears as a limitation of these logics. For instance, neither of ``if John is red-haired, then John is a doctor'' and ``if John is a doctor, then he is red-haired'' is accepted in ordinary reasoning (\cite{MacColl08}). \\ In our approach, the Linearity principle appears as not valid; indeed, we observe that in general $(A|B) \vee (B|A)$ is not maximally assertable because its prevision is not  necessarily equal to 1. Hence the random object $(A|B) \vee (B|A)$ does not coincide with the constant 1. Indeed, defining $P(A|B)=x, P(B|A)=y$ and $P[AB|(A \vee B)]=z$, in our approach it holds that (\cite{gilio13}): \\
\[
(A|B) \wedge (B|A) = AB|(A \vee B) \,,\;\;\; z= \left\{\begin{array}{ll}
	\frac{xy}{x+y-xy}, & \mbox{if}\; x+y>0, \\
	0,  & \mbox{if} \; x=y=0.
\end{array}
\right.
\]
Moreover, it can be verified that $(A|B) \vee (B|A) = (AB + x A\no{B} + y \no{A}B)|(A \vee B)$ and, defining  $\prev[(AB + x A\no{B} + y \no{A}B)|(A \vee B)]=\mu$, it holds that
\[
\mu=P[AB|(A \vee B)] + x P[A\no{B}|(A \vee B)] + y P[\no{A}B|(A \vee B)] \,.
\]
Then, by observing that
\[
P[AB|(A \vee B)]= P(A|B)P[B|(A \vee B)]=P(B|A)P[A|(A \vee B)]\,,
\]
and
\[
P[A\no{B}|(A \vee B)] = 1-P[B|(A \vee B)] \,,\;\;\; P[\no{A}B|(A \vee B)] = 1-P[A|(A \vee B)] \,,
\]
it follows that
\[
\mu = z + x (1-P[B|(A \vee B)]) + y (1-P[A|(A \vee B)]) = z+x-z+y-z =
\]
\[= x+y-z = \left\{\begin{array}{ll}
	\frac{(x+y)^2-xy(x+y+1)}{x+y-xy}, & \mbox{if}\; x+y>0, \\
	0,  & \mbox{if} \; x=y=0.
\end{array}
\right.
\]
Notice that $(A|B) \vee (B|A) = (A|B) + (B|A) - (A|B) \wedge  (B|A)$; moreover, the equality $\mu = x + y - z$ represents the {\em prevision sum rule} (\cite{GiSa14}). As we can see, $(A|B) \vee (B|A)$ and $(A|B) \wedge (B|A)$ are ``maximally assertable'' only if $x=y=1$, in which case it holds that $\mu=z=1$. \\
In \cite{EgRS20}  are also examined some conjunctive sentences which can never be true on DF/TT or CC/TT logics, because one of the conjuncts will always be indeterminate. Given any event $A$, with  $A\neq \emptyset$,
and $A\neq \Omega$, 
 an instance of ``obvious truth'' is obtained when considering the conjoined conditional $(A|A) \wedge (\no{A}|\no{A})$, which is always classified as indeterminate.  However,  in \cite{EgRS20}  it is observed that ``a
sentence such as: \\
\emph{(4) If the sun shines tomorrow, John goes to the beach; and if it rains, he goes to
	the museum. }\\
seems to be true (with hindsight) if the sun shines tomorrow and John goes indeed to
the beach. $\ldots$  How can intuitively plausible compound sentences have positive degree of assertability if they can never be true?'' We can easily verify that, in our approach, the conjunction $(A|A) \wedge (\no{A}|\no{A})$ is  maximally assertable. Indeed, $P(A|A)=P(\no{A}|\no{A})=1$ and $A|A=\no{A}|\no{A}=1$; then, by recalling that
\[
(A|H) \wedge (B|K) = (AHBK + P(A|H) \no{H}BK + P(B|K) AH\no{K})|(H \vee K) \,,
\]
when $H=A$ and $B=K=\no{A}$ it follows that
\[
(A|A) \wedge (\no{A}|\no{A}) = (P(A|A) \no{A} + P(\no{A}|\no{A}) A)|(A \vee \no{A}) = \no{A} + A =1 \,,
\]
and hence $\pr[(A|A) \wedge (\no{A}|\no{A})]=1$, that is $(A|A) \wedge (\no{A}|\no{A})$ is maximally assertable. \\
\subsection{On $SS\cap TT$ validity}
In  \cite{EgRS20}  the  inference of $if \; A \; then \; B$ from $if \; H \; then \; E$ is said to be $SS$-valid (denoted $\emph{if $H$  then  $E\,\,$ $\models_{SS}$\,\,
	if $A$ then $B$}$) if $AB$ is true when $EH$ is true. The same  inference is said to be  $TT$-valid (denoted $\emph{if $H$  then  $E\,\,$ $\models_{TT}$\,\,
	if $A$ then $B$}$) if
$\no{A}\vee B$ is true when $\no{H}\vee E$ is
true (i.e., $H\no{E}$ is true
when $A\no{B}$ is true). 
Based on \cite[Definition 3.3]{EgRS20}, 
an inference from a  set of conditionals $\Gamma$  to a conclusion $B|A$ 
is	$SS$-valid if and only if 
when each  conditional  $E|H$ in $\Gamma$ is $S$-true then $B|A$ is $S$-true.
Moreover,  an inference from  $\Gamma$  to  $B|A$ 
is	$TT$-valid if and only if 
when each conditional $E|H$ in $\Gamma$ is $T$-true then $B|A$ is $T$-true.

The symbol  $\models_{SS\cap TT}$ denotes that  the inference is both $SS$-valid and $TT$-valid.
By the Goodman and Nguyen inclusion relation $\subseteq $, given in formula (\ref{EQ:GN}),
it holds that 
\begin{equation}\label{EQ:GNSSTT}
	\emph{if $H$  then  $E\,\,$ $\models_{SS\cap TT}$\,\,
		if $A$ then $B$}  \;\;\; \Longleftrightarrow  \;\;\; E|H \subseteq B|A.
\end{equation}	
 An inference from a  set $\Gamma$ of conditionals to a conclusion $B|A$ 
is	$SS\cap TT$-valid if and only if 
when each  conditional  $E|H$ in $\Gamma$ is $S$-true   then $B|A$ is $S$-true  and when each conditional $E|H$ in $\Gamma$ is $T$-true  then $B|A$ is $T$-true.

\begin{remark}
	Given a family $\Gamma=\{E_1|H_1,\ldots, E_n|H_n\}$, in the approach of de Finetti the conjunction $
	C_{df}(\Gamma)=
	(E_1|H_1)\wedge_{df} \cdots \wedge_{df} (E_n|H_n)$
	is given by
	\begin{equation}
		C_{df}(\Gamma)=E_1H_1\cdots E_nH_n|(E_1H_1\cdots E_nH_n \vee   \no{E}_1H_1 \vee \cdots \vee \no{E}_nH_n).
	\end{equation}
	We observe that $C_{df}(\Gamma)$ 
	is true if and only if all the conditional events in $\Gamma$ are true;  $C_{df}(\Gamma)$ is false if and only if at least a  conditional event in $\Gamma$ is false;  $C_{df}(\Gamma)$ is void otherwise.
	Then, given a further conditional event $B|A$, it holds that 
	\begin{equation}
		\Gamma \models_{SS\cap TT}  B|A\;\; \Longleftrightarrow\;\; C_{df}(\Gamma)\models_{SS\cap TT}  B|A,
	\end{equation}
	which,
	by recalling (\ref{EQ:GNSSTT}) and (\ref{EQ:GNDEF}), amounts to 
	\begin{equation}
		\Gamma \models_{SS\cap TT}  B|A\;\; \Longleftrightarrow\;\;  C_{df}(\Gamma)\subseteq B|A \;\;   \Longleftrightarrow
    C_{df}(\Gamma \cup \{B|A\})=C_{df}(\Gamma).
	\end{equation}
\end{remark}
\begin{example}
In this example we show that 
$SS\cap TT$-validity implies transitivity. We set  $\Gamma=\{C|B,B|A\}$ and we consider the conclusion $C|A$.
We observe that 
\[
C_{df}(\Gamma)=ABC|(ABC\vee A\no{B}\vee B\no{C})\subseteq C|A,
\]
because $ABC\subseteq AC$ and $A\no{C}=AB\no{C}\vee A\no{B}\no{C}\subseteq
A\no{B}\vee B\no{C}=
 \no{ABC}\wedge (ABC\vee A\no{B}\vee B\no{C})$. Therefore, $\{C|B,B|A\}\models_{SS\cap TT} C|A$, that is transitivity is $SS\cap TT$-valid. However, 
 $SS\cap TT$ validity from a set of premises does not imply that the set p-entails the conclusion, because transitivity is not p-valid (\cite{gilio16}). For instance, let us consider  the events 
$A=$``Peter will win millions in a lottery'',  
$B=$``Peter quits his job'', $C=$``Peter will have money troubles''.  Peter might be confident that, if he quits his job, he will have money troubles (i.e. $P(C|B)$ high), and that, if he wins millions in a lottery, he will quit his job (i.e. $P(B|A)$ high). But Peter will have no confidence at all that, if he win millions in the lottery, he will have money troubles (i.e. $P(C|A)$ low).  We also recall that Transitivity is a non-theorem of any nonmonontonic logic \cite{KrLM90}.
\end{example}

Now we show 
that the property of   p-entailment (see Definition \ref{DEF:PE}) does not imply $SS \cap TT$-validity.
We observe that (\cite[Equation (15)]{GiSa21A})
\begin{equation}
	E|H\leq B|A \Longleftrightarrow
	E|H \subseteq B|A \,, \text{ or }  \; EH=\emptyset\,, \text{ or }   \; A \subseteq B,
\end{equation}
and  by assuming $E|H$ p-consistent (i.e., $EH\neq \emptyset$), from (\cite[Theorem 7]{gilio13}) and (\ref{EQ:GNSSTT})
it holds that
\begin{equation}\label{EQ:GNPENTAIL}
	E|H \Rightarrow_p B|A \;\; \Longleftrightarrow \;\; E|H \subseteq B|A \,,\; or \; A \subseteq B  \;\;\; \Longleftrightarrow \;\; E|H \models_{SS\cap TT} B|A \,,\; or \; A \subseteq B  \;\;\; \Longleftrightarrow \;\;\; E|H \leq B|A \,.
\end{equation}
We observe that, in the particular case where $A \subseteq B$ it holds that $P(B|A)=1$, thus $B|A=1$; then $E|H \leq B|A$ and hence $E|H \Rightarrow_p B|A$. However, in this case one has that $\emph{if $H$  then  $E\,\,$ $\not \models_{SS\cap TT}$\,\,
	if $A$ then $B$}$, that is the inference is not $SS \cap TT$-valid. Indeed, when $\no{A}EH$ is true, it holds that $if \; H \; then \; E$ is true, but $if \; A \; then \; B$ is void, and hence the inference is not $SS$-valid. Thus,
 p-entailment does not imply $SS \cap TT$-validity. \\
To deepen the differences among p-entailment and $SS\cap TT$-validity, let us 
consider  a premise set  $\Gamma=\{E_1|H_1,E_2|H_2\}$ and a conclusion $E_3|H_3$, with $H_3 \nsubseteq E_3$. Under logical independence,  the 27 constituents generated by $\{E_1|H_1,E_2|H_2,E_3|H_3\}$ are illustrated in Table \ref{TAB:TABLE}.
\begin{table*}[]
	\centering
	\begin{tabular}{L|L||L|L||L|L}
		
		C_1    & E_1H_1       E_2H_2       E_3H_3      &  	C_{10} & \no{E}_1H_1  E_2H_2       E_3H_3      & 	C_{19} & \no{H}_1     E_2H_2       E_3H_3      \\
		C_2    & E_1H_1       E_2H_2       \no{E}_3H_3 &	C_{11} & \no{E}_1H_1  E_2H_2       \no{E}_3H_3 &	C_{20} & \no{H}_1     E_2H_2       \no{E}_3H_3 \\
		C_3    & E_1H_1       E_2H_2       \no{H}_3    &	C_{12} & \no{E}_1H_1  E_2H_2       \no{H}_3    &	C_{21} & \no{H}_1     E_2H_2       \no{H}_3    \\
		C_4    & E_1H_1       \no{E}_2H_2  E_3H_3      &	C_{13} & \no{E}_1H_1  \no{E}_2H_2  E_3H_3      &	C_{22} & \no{H}_1     \no{E}_2H_2  E_3H_3      \\
		C_5    & E_1H_1       \no{E}_2H_2  \no{E}_3H_3 &	C_{14} & \no{E}_1H_1  \no{E}_2H_2  \no{E}_3H_3 &	C_{23} & \no{H}_1     \no{E}_2H_2  \no{E}_3H_3 \\
		C_6    & E_1H_1       \no{E}_2H_2  \no{H}_3    &	C_{15} & \no{E}_1H_1  \no{E}_2H_2  \no{H}_3    &	C_{24} & \no{H}_1     \no{E}_2H_2  \no{H}_3    \\
		C_7    & E_1H_1       \no{H}_2     E_3H_3      &	C_{16} & \no{E}_1H_1  \no{H}_2     E_3H_3      &	C_{25} & \no{H}_1     \no{H}_2     E_3H_3      \\
		C_8    & E_1H_1       \no{H}_2     \no{E}_3H_3 &	C_{17} & \no{E}_1H_1  \no{H}_2     \no{E}_3H_3 &	C_{26} & \no{H}_1     \no{H}_2     \no{E}_3H_3 \\
		C_9    & E_1H_1       \no{H}_2     \no{H}_3    &	C_{18} & \no{E}_1H_1  \no{H}_2     \no{H}_3    &	C_0    & \no{H}_1     \no{H}_2     \no{H}_3    \\
		
	\end{tabular}
	\caption{Constituents $C_h$'s  associated with  the family 
		$\{E_1|H_1,E_2|H_2,E_3|H_3\}$}
	\label{TAB:TABLE}
\end{table*}
We observe that, in order the inference from $\Gamma$ to $E_3|H_3$ be $ SS\cap TT$-valid, it should be  $C_2= C_3=C_8=C_{20}=C_{26}=\emptyset$. Moreover, in order the quasi conjunction $QC(\Gamma)$ of the premises in $\Gamma$ satisfies the condition $QC(\Gamma) \subseteq E_3|H_3$, it should be  $C_2= C_3=C_8=C_{9}=C_{20}=C_{21}=C_{26}=\emptyset$. Then
\[
QC(\Gamma) \subseteq E_3|H_3 \;\;\; \Longrightarrow \;\;\;
\Gamma\; {\models_{SS\cap TT} }\;E_3|H_3,
\]
but
\[
\Gamma\; {\models_{SS\cap TT} }\;E_3|H_3\;\;\; \not \Longrightarrow \;\;\; QC(\Gamma) \subseteq E_3|H_3 \,.
\]
In addition, the inference 
from $\Gamma$ to $E_3|H_3$ may be $SS\cap TT$-valid, but not p-valid, i.e. it may be that 
$\Gamma \models_{SS\cap TT}  E_3|H_3$, but
$\Gamma \not \Rightarrow_p E_3|H_3$.
In order to illustrate this aspect  let us assume that $C_2= C_3=C_8=C_{20}=C_{26}=\emptyset$ and, so that the inference is 
$SS\cap TT$-valid, with $C_9\neq \emptyset$ and $C_{21}\neq \emptyset$. Then, we can verify that  $\Gamma \not \Rightarrow_p E_3|H_3$. Indeed, if $C_9$ is true, then  $E_1|H_1$ is true and $E_3|H_3$ is void, and hence $E_1|H_1 \nsubseteq E_3|H_3$. Likewise, if $C_{21}$ is true, then  $E_2|H_2$ is true and $E_3|H_3$ is void, and hence $E_2|H_2 \nsubseteq E_3|H_3$. Moreover, if $C_{9}\vee C_{21}$ is true, then $QC(\Gamma)$ is true and $E_3|H_3$ is void, and hence $QC(\Gamma)\nsubseteq E_3|H_3$. Therefore, by Theorem \ref{THM:ENTAIL-CS}, $\Gamma \not \Rightarrow_p E_3|H_3$, and the assessment $(1,1,0)$ on $\{E_1|H_1,E_2|H_2,E_3|H_3\}$ is coherent.

%
%
The notion of $SS\cap TT$-validity would become equivalent to the condition $QC(\Gamma) \subseteq E_3|H_3$ if we  slightly modified it  in the following way: \\
\emph{An inference from a set $\Gamma$ to a conclusion $B|A$ is $(SS\cap TT)^*$-valid, denoted by $\Gamma \models_{(SS\cap TT)^*} B|A$, if and only if the following conditions are satisfied: $(i)$ when at least one conditional in $\Gamma$ is true and all the other ones are void it holds that $B|A$ is true; $(ii)$ when $B|A$ is false it holds that at least a conditional in $\Gamma$ is false. }\\
Notice that, in the case where $\Gamma=\{E|H\}$,  the condition  $E|H\models_{(SS\cap TT)^*} B|A$ is equivalent to $E|H\models_{SS\cap TT} B|A$, that is $E|H\subseteq B|A$, as shown in (\ref{EQ:GNSSTT}).\\
Under the notion of $(SS\cap TT)^*$-validity, we can show that the concepts of p-validity (in the setting of coherence), delta-epsilon validity and ``yielding'' of Adams, and $(SS\cap TT)^*$-validity are all deeply related to each other. These relationships are illustrated by the following steps: \\
(a) $(SS\cap TT)^*$-validity of an inference from a set of premises $\Gamma$ to a conclusion $B|A$ means that when $QC(\Gamma)$ is true then $B|A$ is true and when $B|A$ is false then $QC(\Gamma)$ is false. \\
(b) $(SS\cap TT)^*$-validity of the inference from $\Gamma$ to $B|A$ is equivalent to the condition $QC(\Gamma) \subseteq B|A$. \\
(c) as defined by Adams (\cite[p.168]{adams98}), 	 $\Gamma$ ``yields'' $B|A$ if only if $QC(\Gamma) \subseteq B|A$ and hence 
$(SS\cap TT)^*$-validity coincides with  the notion of ``yielding''. \\
(d) a p-consistent family $\F$ p-entails $B|A$ if and only if there exists $\Gamma \subseteq F$ such that  $\Gamma$ ``yields" $B|A$, with $\Gamma=\emptyset$ when $B|A$ is a logical truth, i.e. $A \subseteq B$ 
(\cite[p. 187]{adams98}). \\
(e) analogously, in the setting of coherence, by Theorem \ref{THM:ENTAIL-CS}, it holds that 
\[
\F \; \Rightarrow_p \; B|A \;\;\;\;\;\; \Longleftrightarrow \;\;\;\;\;\; QC(S) \subseteq B|A, \; \mbox{for some nonempty} \; S \subseteq \F, \;\; \mbox{or} \;\; A \subseteq B.
\]
or, in other terms, the p-entailment of $B|A$ from $\F$  means that, given the coherent assessment $P(E|H)=1,\, \forall \, E|H \in \F$, its unique coherent extension on $B|A$ is $P(B|A)=1$. \\
(f) equivalently, denoting by $\C(\F)$ the conjunction of the conditional events in $\F$, p-entailment has been characterized in the following way (see Theorem \ref{THM:PENT})
\begin{equation}\label{EQ:THM18}
	\F \; \Rightarrow_p \; B|A \;\;\; \Longleftrightarrow \;\;\; \C(\F) \leq B|A  \;\;\; \Longleftrightarrow \;\;\; \C(\F) \wedge B|A = \C(\F) \,.
\end{equation}
We also remark that, given any finite families $\F$ and $\K$, with $\K \subseteq \F$, it holds that $\C(\K) \leq \C(\F)$. As a consequence, if $\F$ p-entails $B|A$ and $\C(S) \leq B|A$ for some $S \subset \F$, then $S$ p-entails $B|A$ and, for each $\K$ such that $S \subset \K \subset \F$, it holds that $\C(S) \leq \C(\K) \leq \C(\F) \leq B|A$, and hence $\K$ p-entails $B|A$.
\section{A Probabilistic   Deduction Theorem}
\label{SEC:PROBDEDTHM}
In this section we point out that the (full) deduction theorem is not valid in our approach; then, we
give some probabilistic deduction theorems with further related results and examples. \\
In \cite{EgRS20}, after observing that  the  de Finetti conditional  operator is not adequate in some respects (for instance, Modus Ponens and the Identity Law are not satisfied), the   Jeffrey conditional operator ([51]) is recalled and some comparison is made with the de Finetti operator. More precisely, given two conditionals $A$ and $B$, with values in $\{1, \frac12, 0\}$, the (trivalent) Jeffrey conditional $A \rightarrow B$ is defined in Table \ref{TAB:JEFFREY}, where $d_i \in \{1, \frac12\},\, i=1,2,3,4$. 
\begin{table}[!ht]
	\centering
	\begin{tabular}{rl|c|c|c}
	&$ B \;$ &  \;$1$\;   &  \;$\frac12$\;   & \;$0$   \\
	$A$ &&&&\\
		\hline		
		\, $1$           &&\; 1 \; & \;$d_1$\;  & \;0\;        \\
		\, $\frac12$   & & \;$d_2$\; & \;$d_3$\;  & \;0\;        \\
		\, $0$ && \;$\frac12$\; & \;$d_4$\; & \;$\frac12$\;   \\

	\end{tabular}
	\vspace{0.3cm}	
	\caption{Values of the Jeffrey conditional $A \rightarrow B$,  where $d_i \in \{1, \frac12\},\, i=1,2,3,4$.}
	\label{TAB:JEFFREY}
\end{table} \\
Then, in   \cite[Proposition 5.5]{EgRS20},  it is shown that, given a set $\Gamma$ of conditionals and two further conditionals $A$ and $B$, for Jeffrey’s conditionals and $TT$-validity the full classical deduction theorem holds, that is  \begin{equation}\label{EQ:DTTT}
	\Gamma \,,\; A \;\; \models_{J/TT} \;\; B \;\;\; \Longleftrightarrow \;\;\; \Gamma \;\; \models_{J/TT} \;\; A \rightarrow B \,.
\end{equation}
We recall that, in our analysis,  $A$ and $B$ are events (i.e., with values in $\{1, 0\}$) and hence $A \rightarrow B$ is the 
 conditional event $B|A$. Moreover, in our approach we use the notion  of p-validity; then,  the  classical deduction theorem  is not valid, as shown by the counterexample below.
\begin{example}\label{EX:ORIF}	
Given two logically independent events $A$ and $B$,	consider the \emph{or-to-if}  inference
	\begin{equation}
		\label{EQ:DISCOND}
		\mbox{from }
		``A \mbox{\emph{ or} } B \mbox{''} \mbox{ to }
		``\mbox{\emph{if }} \no{A} \mbox{ \emph{then }} B\mbox{''},
	\end{equation}
	which  is valid under the material conditional interpretation, where ``\emph{if} $\bigcirc$ \emph{then } $\Box$'' is interpreted as ``\;$\no{\bigcirc}\vee \Box$''.
	In addition, under the conditional event interpretation, where
	``\emph{if} $\bigcirc$ \emph{then } $\Box$'' is interpreted as ``\;$\Box|\bigcirc$'', the inference (\ref{EQ:DISCOND})
	%
	is $J/TT$-valid. Indeed,   by setting  $\Gamma=\{A\vee B\}$, it holds that $\Gamma \models_{J/TT} B|\no{A}$, because $A\vee B$ is false when  $B|\no{A}$ is false. From (\ref{EQ:DTTT}) it also holds that $\Gamma \cup\{\no{A}\} \models_{J/TT} B$;  indeed, when $A \vee B$ and $\no{A}$ are both $T$-true (i.e., both true, or equivalently $\no{A}B$ true), then $B$ is $T$-true (i.e., $B$ is true). \\
 However, in our approach,  by observing that $\Gamma\cup\{\no{A}\}$ is p-consistent, it holds that
	\begin{equation}\label{EQ:DTP}
		\Gamma \cup \{\no{A} \}\;\; \Rightarrow_{p} \;\; B \,,\;\;\; \mbox{but} \;\; \Gamma\;\; \nRightarrow_{p} \;\; B|\no{A} \,,
	\end{equation}
	and hence (\ref{EQ:DTTT}) is not satisfied  when TT-validity is replaced by p-validity.
	In order to verify (\ref{EQ:DTP})  we observe that  if  $P(A \vee B)=1$ and $P(\no{A})=1$, then $P(AB)=P(A\no{B})=0$ and $P(B)=P(A \vee B)-P(A\no{B})=P(A \vee B)=1$, thus $\Gamma \cup \{\no{A} \}\;\; \Rightarrow_{p} \;\; B$. But  $P(A\vee B)=1$ does not imply  $P(B|\no{A})=1$; for instance, when $P(A)=1$ and $P(B)=0$, it holds that $P(A \vee B)=1$ and the assessment $(1,x)$ on $\{A \vee B,B|\no{A}\}$ is coherent for every $x \in [0,1]$.  Thus
	$\Gamma \nRightarrow_{p}  B|\no{A}$. The same conclusion also follows  from  (\ref{EQ:GNPENTAIL}) by observing that  $A\vee B\nleq B|\no{A}$.  Indeed, when $x<1$, if  $A$ is true, then   $A\vee B=1> x=B|\no{A}$.
Note that some instances of (\ref{EQ:DISCOND}) in natural language are highly intuitive, and that, if (\ref{EQ:DISCOND}) were valid, then the paradoxical inference 
	 from ``$A$'' to ``if $\no{A}$ then $B$'' would also be valid. For ``$A$ or $B$'' would follow from ``$A$'', and then ``if $\no{A}$ then $B$'' would follow by (\ref{EQ:DISCOND})
 (see \cite{gilio12} for a full treatment of when this inference is, and is not, intuitive; see also \cite{Cruz2015}). 
\end{example}
We remark that analogous comments can be made for the full deduction theorem derivation of ``$B$ entails (if $A$ then $B$)'' in the trivalent analysis, making the conditional ``close to'' the material conditional and its paradoxes. Indeed, by the full deduction theorem it holds that
\[
B \,,\; A \;\; \models_{J/TT} \;\;  \;\; B \;\;\; \Longleftrightarrow \;\;\; B \;\; \models_{J/TT} \;\; B|A \,.
\]
We do not have this problem either with p-entailment. Indeed $\{B,\, A\} \; \Rightarrow_{p} \; B$, but $B \; \nRightarrow_{p} \; B|A$. 
As observed above, the full deduction theorem states that 
given a set $\Gamma$ of conditionals and two further conditionals $A$ and $B$,   $\Gamma \cup\{A\}$ entails $B$ if and only if  $\Gamma$ entails the conditional \emph{if} $A$ \emph{then} $B$; we  call  this result \emph{Classical Deduction Theorem}.  It  holds when a conditional  \emph{if} $H$ \emph{then} $E$ is looked at as the event $\no{H}\vee E$ (the material conditional) and,  
as we observed before, it also holds  for Jeffrey's conditionals and TT-validity.

Our approach, where  we look at  conditionals as  conditional events,  is based on the notion of  p-entailment. Using the notion of p-entailment the  Classical Deduction Theorem is not valid; indeed, given two events 
$A\neq \emptyset$, $B\neq \emptyset$ and a   p-consistent family  $\Gamma$  of conditional events, 
from   $\Gamma\cup \{A\} \Rightarrow_p B$  it does not follow that  $\Gamma \Rightarrow_p B|A$. Interestingly,  under the assumption that $\Gamma\cup \{A\}$ is p-consistent,   
the converse holds, as shown  by the result below.
\begin{theorem}\label{THM:CONVERSE}
	Let $A\neq \emptyset$, $B\neq \emptyset$ be any events and   $\Gamma$ be any finite  family of conditional events such that  $\Gamma \Rightarrow_p B|A$.  If $\Gamma\cup\{A\}$ is p-consistent, then  ${\Gamma \cup \{ A\}\Rightarrow_p B}$.
\end{theorem}
\begin{proof}
	As $\Gamma \Rightarrow_p B|A$, it holds that $\C(\Gamma)\leq B|A$.
	We observe that $\C(\Gamma)\wedge A$ coincides with $\C(\Gamma)$, or $0$ according to whether $A$ is true, or false, respectively.  Moreover, $(B|A)\wedge A$ coincides with $B|A$, or $0$ according to whether $A$ is true, or false, respectively.
	Then, as $\C(\Gamma)\leq B|A$ and $(B|A)\wedge A=AB$, it follows that 
	\[
	\C(\Gamma)\wedge A\leq  (B|A)\wedge A=AB\leq B,
	\]
	which, under the assumption that $\Gamma\cup\{A\}$ is p-consistent,  amounts to the p-entailment of $B$ from $\Gamma \cup \{ A\}$, that is ${\Gamma \cup \{ A\} \Rightarrow_p B}$.
\end{proof} 
\begin{remark}
	We observe that even if $\Gamma$ is  p-consistent and 
	$\Gamma \Rightarrow_p B|A$, it could be that $\Gamma \cup \{ A\}$ is not p-consistent. For instance, let  $\Gamma=\{\no{A}\}$, where $A\neq \emptyset$ and $A\neq \Omega$.   $\Gamma$ is p-consistent because $\no{A}\neq \emptyset$, and 
	$\Gamma \Rightarrow_p A|A$ because    $A|A=1$; but  $\Gamma\cup\{A\}=\{\no{A},A\}$ is not p-consistent because $P(A)=P(\no{A})=1$ is not coherent.
\end{remark}
In  what follows, after a preliminary result, we show that,  by adding the  condition $\Gamma \Rightarrow_p A$  to the assumption  $\Gamma \cup\{A\} \Rightarrow_p B$, we obtain that $\Gamma  \Rightarrow_p B|A$ (and $\Gamma  \Rightarrow_p A|B$), which we have been calling  \emph{Probabilistic   Deduction Theorem}.
By recalling Theorem \ref{THM:PENT},   the property 	$\F \Rightarrow_p E|H$, with   $\F$  p-consistent, can be equivalently written as  $\C(\F)\leq E|H$, or $\C(\F)\wedge(E|H)= \C(\F)$, where by recalling Remark \ref{REM:CONGCONG}, $\C(\F)\wedge(E|H)=\C(\F\cup \{E|H\})$. 

\begin{theorem}\label{THM:PRELPDT}
	Let $A\neq \emptyset$, $B\neq \emptyset$ be any events and   $\Gamma$ be any finite  family of conditional events, with $\Gamma \cup\{A\}$ p-consistent. Then,  the following assertions are equivalent:
	\\
	$(i)$	 
	$\Gamma \Rightarrow_p A$ and  $\Gamma \cup \{ A\} \Rightarrow_p B$.  \\
	$(ii)$		$\Gamma \Rightarrow_p AB$;\\
	$(iii)$
	$\Gamma \Rightarrow_p A$ and  $\Gamma  \Rightarrow_p B$; 
\end{theorem}
\begin{proof}
	We will prove the theorem by verifying that 
	\[
	(i)\Longrightarrow (ii) \Longrightarrow (iii) \Longrightarrow (i).
	\]
	$(i)\Longrightarrow (ii).$	
	By  Theorem \ref{THM:PENT}, 
	as $\Gamma \Rightarrow_p A$ and 
	$\Gamma \cup\{A\}\Rightarrow_p B$,  it holds that 
	\[
	\C(\Gamma)=\C(\Gamma\cup\{A\})=\C(\Gamma \cup \{A\}\cup \{B\})=\C(\Gamma) \wedge AB.
	\]
	Then,   still by Theorem \ref{THM:PENT},  $\Gamma\Rightarrow_p AB$.\\
	$(ii)\Longrightarrow (iii)$. 
	The condition  $\Gamma \Rightarrow_p AB$ is equivalent to the condition
	$\C(\Gamma) \leq AB$. Thus, $\C(\Gamma) \leq A$ and $\C(\Gamma) \leq B$;  hence 
	$\Gamma \Rightarrow_p A$ and 	   $\Gamma \Rightarrow_p B$.\\
	$(iii)\Longrightarrow (i)$. 
	We only need to prove that  $\Gamma \cup \{ A\} \Rightarrow_p B$.
	As  $\Gamma \Rightarrow_p A$ and 	   $\Gamma \Rightarrow_p B$, by Theorem \ref{THM:PENT}, it follows that 
	$\C(\Gamma)\wedge A=\C(\Gamma\cup\{A\})= \C(\Gamma)\leq B$. Then,  $\Gamma \cup \{ A\} \Rightarrow_p B$.  
\end{proof}
\begin{remark}\label{REM:EXCHANGE}
	We observe that Theorem \ref{THM:PRELPDT} still holds if we exchange the roles of $A$ and $B$. Therefore, 
	\[
	\Gamma \Rightarrow_p A\;\;\mbox{ and }\;\; \Gamma \cup \{ A\} \Rightarrow_p B \;\;\;\; \Longleftrightarrow \;\;\;\; \Gamma \Rightarrow_p B\;\;\mbox{ and }\;\; \Gamma \cup \{ B\} \Rightarrow_p A.
	\]
	Notice that, if $\Gamma \cup \{A\}\Rightarrow_p B$, then the assessment $P(E|H)=1,\forall E|H\in\Gamma, P(A)=P(B)=1$ is coherent, and hence $\Gamma \cup \{B\}$ is p-consistent.
\end{remark}
By Theorem \ref{THM:PRELPDT} we obtain the following result.

\begin{theorem}[Probabilistic Deduction Theorem]\label{THM:PDT}
	Let $A\neq \emptyset$, $B\neq \emptyset$ be any events and   $\Gamma$ be any finite  family of conditional events, with $\Gamma \cup\{A\}$ p-consistent.	 If $\Gamma \Rightarrow_p A$ and $\Gamma \cup \{ A\} \Rightarrow_p B$, then: $(a)$ $\Gamma \Rightarrow_p B|A$;  $(b)$ $\Gamma \Rightarrow_p A|B$. 
\end{theorem}
\begin{proof}
	By Theorem \ref{THM:PRELPDT} ,  as $\Gamma \Rightarrow_p A$ and $\Gamma \cup \{ A\} \Rightarrow_p B$, it holds that 	$\Gamma \Rightarrow_p AB$, or equivalently $\C(\Gamma)\leq AB$. By defining $P(A|B)=x$ and $P(B|A)=y$, it holds that $AB\leq AB+x\no{A}=A|B$ and $AB\leq AB+y\no{B}=B|A$. Therefore 
	$\C(\Gamma)\leq A|B$ and $\C(\Gamma)\leq B|A$. Then, 
	both conditions $(a)$ and $(b)$ are satisfied.
\end{proof}
\begin{remark}\label{REM:GAMMA'}
	Given a non p-valid inference rule ``from $\Gamma$ infer $B|A$'', which satisfies the Classical Deduction Theorem, we can get a p-valid inference rule
	by adding the antecedent $A$ to the premise set  $\Gamma$. More precisely, given a set $\Gamma$
	and a conditional event $B|A$, with   $\Gamma\cup\{A\}$ p-consistent, such 
	that $\Gamma \cup\{A\}\Rightarrow_p B$, but  $\Gamma \nRightarrow_p B|A$,
	for the set  $\Gamma'=\Gamma\cup \{A\}$ it holds that $\Gamma' \Rightarrow_p A$ and $\Gamma'\cup \{A\} \Rightarrow_p B$. Thus, by 
	Theorem \ref{THM:PDT}, it holds that 
	$\Gamma' \Rightarrow_p B|A$ and $\Gamma' \Rightarrow_p A|B$.
\end{remark}	

The next result, which concerns the biconditional event $(A|B)\wedge (B|A)$, suggests an equivalent formulation of the Probabilistic Deduction Theorem.
\begin{theorem}\label{THM:BICONDRULE}
	Let $A\neq \emptyset$, $B\neq \emptyset$ be any events and   $\Gamma$ be any finite p-consistent  family of conditional events. Then, 
	$\Gamma \Rightarrow_p B|A$ and
	$\Gamma \Rightarrow_p A|B$ if and only if 
	$\Gamma \Rightarrow_p (A|B)\wedge (B|A)$. 
\end{theorem}
\begin{proof}
	$(\Rightarrow)$	
	We recall that  $(A|B)\wedge (B|A)=AB|(A\vee B)$ (\cite[Theorem 7]{SPOG18}). 
	Given a coherent assessment $(x,y)$ on $\{A|B,B|A\}$ the unique coherent extension $z$ on $AB|(A\vee B)$ is $z=0$, when $(x,y)=(0,0)$, or $z=\frac{xy}{x+y-xy}$ when $(x,y)\neq (0,0)$ (\cite[Section 7]{gilio13}).  By assuming $P(E|H)=1$, for each  $E|H\in \Gamma$ it follows that $x=y=1$. Then $z=1$ and hence $\Gamma \Rightarrow_p (A|B)\wedge (B|A)$.\\
	$(\Leftarrow)$
	We observe that $z=1$ if and only if $x=y=1$.
	By assuming $P(E|H)=1$, for each  $E|H\in \Gamma$ it follows that $z=1$. Then $x=y=1$ and hence
	$\Gamma \Rightarrow_p B|A$
	and 
	$\Gamma \Rightarrow_p A|B$.
\end{proof}
In order to illustrate the previous results we examine some examples.

\begin{example}[Transitivity rule]\label{EX:TR}	
	Let $A,B,C$	 be logically independent events and 
	$\Gamma=\{C|B,B|A\}$. We consider the inference from $\Gamma$ to the conditional event  $C|A$.
	It is well known that transitivity is not p-valid, that is  $\Gamma \nRightarrow_{p} C|A$ (\cite{gilio16}, see also \cite[Section 10.1]{GiSa19}). 
	However,   it can be verified  that  $\Gamma \cup\{A\}\Rightarrow_{p} C$ (see also \cite[p.132, Exercise 1]{adams98}). Indeed, as $(B|A)\wedge A=BA$, it follows that 
	\begin{equation}\label{EQ:TRR}
\C(\Gamma \cup\{A\})=		(C|B)\wedge(B|A)\wedge A=(C|B)\wedge B A=ABC.
	\end{equation}
As 
	$\C(\Gamma \cup\{A\})=ABC\leq C  $, by Theorem \ref{THM:PENT} , $\Gamma \cup\{A\}\Rightarrow_{p} C$. Like Example \ref{EX:ORIF}, this example shows that the Classical Deduction Theorem does not hold in our approach, because   $\Gamma \cup\{A\}\Rightarrow_{p} C$, but $\Gamma \nRightarrow_{p} C|A$.
\end{example}	
\begin{example}[Modus ponens and  Transitivity]\label{EQ:TRR2}
	As in Example \ref{EX:TR},	
	let us consider the set $\Gamma=\{C|B,B|A\}$ and the conditional event  $C|A$. Based on Remark \ref{REM:GAMMA'}, we set $\Gamma'=\Gamma\cup\{A\}=\{C|B,B|A,A\}$.
	Of course $
	\Gamma'\Rightarrow_p A 
	$ because $A\in \Gamma'$; moreover,  it can be shown that 
	$
	\Gamma'\cup \{A\}\Rightarrow_p C 
	$. Indeed, $\Gamma' \cup\{A\}=\Gamma'$ and, as shown in (\ref{EQ:TRR}),  it holds that $\C(\Gamma')=(C|B)\wedge(B|A)\wedge A=ABC\leq C$.
	Then, by Theorems \ref{THM:PDT} and \ref{THM:BICONDRULE}, it holds that: $\Gamma' \Rightarrow_{p} C|A$,  $\Gamma' \Rightarrow_{p} A|C$, and $\Gamma' \Rightarrow_{p} (A|C)\wedge (C|A)$. 
\end{example}	
\begin{example}[On combining evidence: an example from Boole]\label{EX:BOOLE}	
	Let $A,B,C$	 be logically independent events.  We recall that  the inference from $\Gamma=\{C|B,C|A\}$ to  $C|AB$ is not p-valid (\cite[Section 10.1]{GiSa19}). 
	Moreover, $\Gamma\cup \{AB\}\Rightarrow_p C$, because 
	\begin{equation}\label{EQ:GAMMABOOLE}
		\C(\Gamma \cup\{AB\})=(C|B)\wedge (C|A)\wedge AB=(C|B)\wedge B \wedge (C|A)\wedge A=ABC\leq C.
	\end{equation}
	However,  the Classical Deduction Theorem does not hold in our approach because 
	$\Gamma\nRightarrow_{p} C|AB$ (\cite[Section 10.1]{GiSa19}. 
	Based on Remark \ref{REM:GAMMA'}, by defining  $\Gamma'=\Gamma \cup\{AB\}=\{C|B,C|A,AB\}$,  
	by Theorems \ref{THM:PDT} and \ref{THM:BICONDRULE}  it holds that: $\Gamma'\Rightarrow_{p} C|AB$,  $\Gamma'\Rightarrow_{p} AB|C$, and $\Gamma'\Rightarrow_{p} (AB|C)\wedge (C|AB)=ABC|(AB\vee C)$.
	Indeed, the assumptions 
	$	\Gamma'\Rightarrow_p \{AB\} 
	$ and	$
	\Gamma'\cup \{AB\}\Rightarrow_p C 
	$ of Theorem \ref{THM:PDT}   
	are satisfied because $AB\in \Gamma'$ and 
	$\C(\Gamma'\cup \{AB\})=\C(\Gamma\cup \{AB\})\leq AB$,  as shown in (\ref{EQ:GAMMABOOLE}).
\end{example}	

\begin{example}[Contraposition]
	\label{EX:CONTR}
	Recall that the inference from $\Gamma=\{C|A\}$ to $\no{A}|\no{C}$ is not p-valid. Indeed,  the assessment $(1,z)$ on $\{C|A,\no{A}|\no{C}\}$ is coherent for every $z\in[0,1]$.
	Moreover, $\Gamma \cup\{\no{C}\}\Rightarrow_{p} \no{A}$ because
	\begin{equation}\label{EQ:GAMMACONTRAP}
		\C(\Gamma \cup\{\no{C}\})=(C|A)\wedge \no{C}=0(A+ \no{A}C)+P(C|A)\no{A}\,\no{C}=P(C|A)\no{A}\,\no{C}\leq \no{A}.
	\end{equation}
	Based on Remark \ref{REM:GAMMA'}, by defining  $\Gamma'=\Gamma \cup\{\no{C}\}=\{C|A,\no{C}\}$,  
	by Theorems \ref{THM:PDT} and \ref{THM:BICONDRULE} it holds that: $\{C|A,\no{C}\}\Rightarrow_{p} \no{A}|\no{C}$, $\{C|A,\no{C}\}\Rightarrow_{p} \no{C}|\no{A}$, $\{C|A,\no{C}\}\Rightarrow_{p} \no{A}\,\no{C}|(\no{A}\vee \no{C})$.
	Indeed, the assumptions 
	$	\Gamma'\Rightarrow_p \{\no{C}\} 
	$ and	$
	\Gamma'\cup \{\no{C}\}\Rightarrow_p \no{A} 
	$ of Theorem \ref{THM:PDT}   
	are satisfied because $\no{C}\in \Gamma'$ and 
	$\C(\Gamma'\cup \{\no{C}\})=\C(\Gamma\cup \{\no{C}\})\leq  \no{A}$, as shown in (\ref{EQ:GAMMACONTRAP}).
\end{example}
\begin{remark}
	Notice that in Example \ref{EX:CONTR}, as $\Gamma \cup\{\no{C}\}= \{C|A,\no{C}\}$,  the inference $\Gamma \cup\{\no{C}\}\Rightarrow_{p} \no{A}$  coincides with  Modus Tollens.   Then, in the context of iterated conditionals, from  Theorem \ref{THM:PEITER} it holds that  $\no{A}|(\no{C}\wedge (C|A))=1$
		(see also \cite[Section 4.1]{GiPS20}). However, as Contraposition is not p-valid, from (\ref{EQ:p-entail-iter}) it holds that  $(\no{A}|\no{C})| (C|A)\neq 1$, that is the ``export''  of $C|A$ from $\no{A}|(\no{C}\wedge (C|A))$ to $(\no{A}|\no{C})| (C|A)$ does not hold. Moreover, as $\{C|A,\no{C}\}\Rightarrow_{p} \no{A}|\no{C}$, in terms of iterated conditionals, by Theorem \ref{THM:PEITER} , 
	it holds  that  $(\no{A}|\no{C})|(\no{C}\wedge (C|A))=1$.	
\end{remark}	

\begin{example}[Weak monotonicity]\label{EX:MONOT}
	 Monotonicity, that is the inference from $\Gamma=\{C|A\}$ to $C|AB$ is not p-valid. Indeed,  the assessment $(1,z)$ on $\{C|A,C|AB\}$ is coherent for every $z\in[0,1]$.
	Moreover, $\Gamma \cup\{AB\}\Rightarrow_{p} C$ because
	\begin{equation}\label{EQ:GAMMAMONOT}
		\C(\Gamma \cup\{AB\})=(C|A)\wedge A\wedge B=ABC\leq C.
	\end{equation}
	Based on Remark \ref{REM:GAMMA'}, by defining  $\Gamma'=\Gamma \cup\{AB\}=\{C|A,AB\}$,  
	by Theorems \ref{THM:PDT} and \ref{THM:BICONDRULE} it holds that:
	\begin{equation}\label{EQ:WEAKMONOTO}
		\{C|A,AB\}\Rightarrow_{p} C|AB \;\;\; \text{(Weak monotonicity)}
	\end{equation}	
	and
	\[
	\{C|A,AB\}\Rightarrow_{p} AB|C,\;\; \{C|A,AB\}\Rightarrow_{p} ABC|(AB\vee C).
	\]
	Indeed, the assumptions 
	$	\Gamma'\Rightarrow_p \{AB\} 
	$ and	$
	\Gamma'\cup \{AB\}\Rightarrow_p C 
	$ of Theorem \ref{THM:PDT}   
	are satisfied because $AB\in \Gamma'$ and 
	$\C(\Gamma'\cup \{AB\})=\C(\Gamma\cup \{AB\})\leq  C$, as shown in (\ref{EQ:GAMMAMONOT}).
\end{example}	
The next theorem gives a more general result related with 	Transitivity.
\begin{theorem}
	Let $A_1,\ldots, A_n$ be logically independent events. Then, 
	\[
	\{A_1, A_2|A_1,\ldots, A_n|A_{n-1}\} \Rightarrow_p A_i|A_j,\;\;\; \forall\; i,j=1,\ldots,n.
	\]
\end{theorem} 
\begin{proof}
	We set $ \Gamma=\{A_1, A_2|A_1,\ldots, A_n|A_{n-1}\}$. Then, by observing that  $A_iA_j\leq A_i|A_j$, for all  $i,j$, and that 
	\[
	A_1\cdots A_{k-1}\wedge (A_k|A_{k-1})=A_1\cdots A_k, \;\; k=2,\ldots,n,
	\]
	we obtain
	\[
	\begin{array}{ll}
		\C(\Gamma)=A_1\wedge (A_2|A_1)\wedge \cdots \wedge (A_n|A_{n-1})=A_1A_2\wedge  (A_3|A_2)\wedge \cdots \wedge (A_n|A_{n-1})=\\
		=\cdots =A_1\cdots A_{n-1}\wedge   (A_n|A_{n-1})=A_1\cdots A_n\leq A_iA_j\leq A_i|A_j,  \;\; i, j = 1, 2,...,n.
	\end{array}
	\]
	Thus, $\Gamma\Rightarrow_p A_i|A_j,$ for every $ i,j=1,\ldots,n$.
	\end{proof}
%
We now give two results, 
Theorem \ref{THM:PRELPDTGEN} and Theorem \ref{THM:PDTGEN},
which generalize  Theorem \ref{THM:PRELPDT} and Theorem \ref{THM:PDT}, respectively. In particular, the events $A$ and $B$ are replaced by the conditional events $A|H$ and $B|K$, respectively. Then, the conjunction $AB$ becomes $(A|H)\wedge (B|K)$; moreover, $B|A$ and $A|B$ become the iterated conditionals $(B|K)|(A|H)$ and $(A|H)|(B|K)$, respectively.
\begin{theorem}\label{THM:PRELPDTGEN}
	Let $A|H$, $B|K$ be two conditional events, with $AH\neq \emptyset$  and  $BK\neq \emptyset$, and  $\Gamma$  any finite  family of  conditional events, with $\Gamma \cup\{A|H\}$ p-consistent. 
	Then,  the following assertions are equivalent:
	\\
	$(i)$	 
	$\Gamma \Rightarrow_p A|H$ and  $\Gamma \cup \{ A|H\} \Rightarrow_p B|K$.  \\
	$(ii)$		$\Gamma \Rightarrow_p (A|H)\wedge (B|K)$;\\
	$(iii)$
	$\Gamma \Rightarrow_p A| H$ and  $\Gamma  \Rightarrow_p B|K$; 
\end{theorem}
\begin{proof}
	We will prove the theorem by verifying that 
	\[
	(i)\Longrightarrow (ii) \Longrightarrow (iii) \Longrightarrow (i).
	\]
	$(i)\Longrightarrow (ii).$	
	By 	Theorem \ref{THM:PENT}, 
	as $\Gamma \Rightarrow_p A|H$ and 
	$\Gamma \cup\{A|H\}\Rightarrow_p B|K$,  it holds that 
	\begin{equation}\label{EQ:GAMMACONG}
	\C(\Gamma)=\C(\Gamma)\wedge (A|H)=\C(\Gamma) \wedge (A|H)\wedge (B|K)\leq (A|H)\wedge (B|K).
	\end{equation}
	As $\Gamma$ is p-consistent, the assessment $P(E|H)=1$, for every $E|H\in\Gamma$, is coherent.
	Moreover,  $P(E|H)=1$,  $\forall E|H\in\Gamma$, by Theorem \ref{THM:TEOREMAAI13}, implies that  $\prev[C(\Gamma)]=1$ and  from (\ref{EQ:GAMMACONG}) $\prev[(A|H)\wedge (B|K)]=1$.
	Thus, by Definition \ref{DEF:PE},  $\Gamma\Rightarrow_p (A|H)\wedge (B|K)$.\\
	$(ii)\Longrightarrow (iii)$. 
	As $\Gamma \Rightarrow_p (A|H)\wedge (B|K)$, when $P(E|H)=1$, for every $E|H\in\Gamma$, it follows that  $\prev[(A|H)\wedge (B|K)]=1$ and hence  $P(A|H)=P(B|K)=1$, because  $(A|H)\wedge (B|K)\leq \min\{(A|H),(B|K)\}$. Then 
	$\Gamma \Rightarrow_p A|H$ and 	   $\Gamma \Rightarrow_p B|K$.\\
	$(iii)\Longrightarrow (i)$. 
	We only need to prove that  $\Gamma \cup \{ A|H\} \Rightarrow_p B|K$.
	As  $\Gamma \Rightarrow_p A|H$ and 	   $\Gamma \Rightarrow_p B|K$, by Theorem~\ref{THM:PENT}, it follows that 
	$\C(\Gamma)\wedge (A|H)=\C(\Gamma\cup\{A|H\})= \C(\Gamma)\leq B|K$. Then,  $\Gamma \cup \{ A|H\} \Rightarrow_p B|K$.  
\end{proof}
\begin{remark}\label{REM:EXCHANGEGEN}
	Theorem \ref{THM:PRELPDTGEN} still holds if we exchange the roles of $A|H$ and $B|K$. Therefore, 
	\[
	\Gamma \Rightarrow_p A|H\;\;\mbox{ and }\;\; \Gamma \cup \{ A|H\} \Rightarrow_p B|K \;\; \Longleftrightarrow \;\;\Gamma \Rightarrow_p B|K\;\;\mbox{ and }\;\; \Gamma \cup \{ B|K\} \Rightarrow_p A|H.
	\]
\end{remark}
\begin{theorem}[Generalized Probabilistic Deduction Theorem]\label{THM:PDTGEN}
	Let $A|H$, $B|K$ be two conditional events, with $AH\neq \emptyset$  and  $BK\neq \emptyset$, and  $\Gamma$  any p-consistent finite  family of  conditional events, with $\Gamma \cup \{ A|H\}$ p-consistent.  If $\Gamma \Rightarrow_p A|H$ and $\Gamma \cup \{ A|H\} \Rightarrow_p B|K$, then: $(a)$ $\Gamma \Rightarrow_p (B|K)|(A|H)$;  $(b)$ $\Gamma \Rightarrow_p (A|H)|(B|K)$. 
\end{theorem}
\begin{proof}
	By Theorem \ref{THM:PRELPDTGEN},  as $\Gamma \Rightarrow_p A|H$ and $\Gamma \cup \{ A|H\} \Rightarrow_p B|K $, it holds that 	$\Gamma \Rightarrow_p (A|H)\wedge(B|K)$.
	Moreover, 	by defining 
	$\mu =\prev[(B|K)|(A|H)]$,
	it holds that $(A|H)\wedge(B|K)\leq
	(A|H)\wedge(B|K)+\mu(1-(A|H))=(B|K)|(A|H)$.
	Therefore,  
	when $P(E|H)=1$, for every $E|H\in\Gamma$, it follows that 
	$\prev[(A|H)\wedge (B|K)]=P(A|H)=1$ and hence 
	\[
	\prev[(B|K)|(A|H)]=\mu =\prev[(A|H)\wedge (B|K)]+\mu(1-P(A|H))=1.
	\]
	Then, 
	condition $(a)$ is satisfied. Likewise, the condition $(b)$ is satisfied too.
	
\end{proof} 
\begin{remark}
	We can verify that, 
	if  
	\begin{equation}\label{EQ:COND1BIS}
		\Gamma \; \Rightarrow_p \; A \; \mbox{ and } \;	\Gamma \cup \{A\} \; \Rightarrow_p \; B,
	\end{equation}
	then, for every event $H^*$, it holds that 
	\begin{equation}\label{EQ:COND2BIS}
		\Gamma \; \Rightarrow_p \; A|(A \vee H^*) \; \mbox{ and } \; \Gamma \cup \{A|(A \vee H^*)\}  \;	\Rightarrow_p \; B|(A \vee H^*),
	\end{equation}
	and  
	\begin{equation}\label{EQ:COND3BIS}
		\Gamma \; \Rightarrow_p \; A|(A \vee B \vee H^*) \; \mbox{ and } \; \Gamma \cup \{A|(A \vee B \vee H^*)\}  \; \Rightarrow_p  \;	B|(A \vee B \vee H^*)  \,.
	\end{equation}
	Indeed,  by Theorem \ref{THM:PRELPDT} the conditions in (\ref{EQ:COND1BIS}) are equivalent to $\Gamma \Rightarrow_p AB$.  Then, by observing that $AB\leq AB|(A \vee H^*)$ and $AB\leq AB|(A\vee B \vee H^*)$, when the conditions in (\ref{EQ:COND1BIS}) are satisfied it holds that 
	\[
	\Gamma \; \Rightarrow_{p} \; AB|(A\vee H^*)  \; \mbox{ and } \; \Gamma \; \Rightarrow_{p} \; AB|(A\vee B\vee H^*) \,,
	\]
	which, by Theorem \ref{THM:PRELPDTGEN}, are equivalent to the  conditions in (\ref{EQ:COND2BIS}) and
	(\ref{EQ:COND3BIS}), respectively.
\end{remark}
We  give a weaker version of Theorem \ref{THM:PDT}, where the set of  assumptions $\Gamma \Rightarrow_p A$  and $\Gamma \cup \{A\} \Rightarrow_p B$ are replaced by the weaker assumptions $\Gamma \Rightarrow_p A|(A \vee H^*)$  and $\Gamma \cup \{A|(A \vee H^*)\} \Rightarrow_p B|(A \vee H^*)$.
\begin{theorem}\label{THM:PDTWEAK}
	Let $A\neq \emptyset$, $B\neq \emptyset$, and $H^*$ be any events and   $\Gamma$ be any finite  family of $n$ conditional events, with $\Gamma \cup\{A|(A\vee H^*)\}$ p-consistent.	 If $\Gamma \Rightarrow_p A|(A \vee H^*)$  and $\Gamma \cup \{A|(A \vee H^*)\} \Rightarrow_p B|(A \vee H^*)$, then  $\Gamma \Rightarrow_p B|A$.
\end{theorem}
\begin{proof}
	By  applying Theorem \ref{THM:PRELPDTGEN}, with $H=K=A\vee H^*$, we obtain that  $\Gamma \Rightarrow_{p} AB|(A \vee H^*)$. Moreover,  as  $AB|(A \vee H^*)\leq B|A$, it follows that $ \Gamma \Rightarrow_{p} B|A$.
	
\end{proof}
When $B\nsubseteq H^*$, it holds that $AB|(A \vee H^*) \; \nleq \; A|B$; thus, from $\Gamma \Rightarrow_p AB|(A \vee H^*)$, it does not follow in general that $\Gamma  \Rightarrow_p A|B$, that is Theorem \ref{THM:PDTWEAK} is not symmetric with respect to $A$ and $B$. 

\begin{example}[A weak version of transitivity rule \cite{freund1991}]
	We illustrate  an example of p-entailment, where   the assumptions of Theorem \ref{THM:PDTWEAK} are satisfied, while the assumptions of Theorem \ref{THM:PDT} are  not.	
	Let $A,B,C$	 be three logically independent events and 
	$\Gamma=\{C|B,B|A,A|(A\vee B)\}$. We consider the inference from $\Gamma$ to the  conditional event  $C|A$.	
	With respect to Example \ref{EX:TR} we added to $\Gamma$ the event $A|(A\vee B)$.
	The assumptions of Theorem \ref{THM:PDTWEAK} are satisfied with $H^*=B$; thus	$\Gamma \Rightarrow_{p} C|A$ (see also \cite[Section 10.2]{GiSa19}, or \cite[Section 4]{gilio16}). However, the assumptions of  Theorem \ref{THM:PDT} are not satisfied because $\Gamma \cup\{A\}\Rightarrow_p C$, but 
	$\Gamma \nRightarrow_p A$. 
	Indeed,  based on (\ref{EQ:CONGQC}), it holds that $(B|A)\wedge (A|(A\vee B))=AB|(A\vee B)$. 
	Then,  $\C(\Gamma)=(C|B)\wedge (B|A)\wedge(A|(A\vee B))=(C|B)\wedge (AB|(A\vee B))$ and still from (\ref{EQ:CONGQC}) it follows that $\C(\Gamma)=ABC|(A\vee B)$. Moreover,
	$\C(\Gamma)\nleq A$ because in the case where  $P(ABC|(A\vee B))>0$ when  $\no{A}\,\no{B}$ is true it holds that   $\C(\Gamma)=P(ABC|(A\vee B))\nleq 0=A$.
	We also observe that $\Gamma \nRightarrow_p A|C$, because $C(\Gamma)=ABC|(A\vee B)\nleq A|C$. 
\end{example}
\begin{example}[CM rule]
	In this example, by recalling that the inference  from $\{C|A\}$ to $C|AB$ is not p-valid, we show that by applying Theorem \ref{THM:PDTWEAK} with  $A,B$, and  $H^*$ replaced by $AB,C$, and $A$,  respectively, and with   $\Gamma =\{C|A,B|A\}$,  it holds that 
	$A|(A\vee H^*)$ and $B|(A\vee H^*)$ become   $B|A$ and $C|A$, respectively.
	Moreover,  as
	$\Gamma\Rightarrow_p B|A$ and $\Gamma\cup\{C|A\}\Rightarrow_p C|A$, it  follows that $\Gamma\Rightarrow_p C|AB$,
	which is the CM rule. 	
	We observe that the premise $AB$, added to $\{C|A\}$, for obtaining weak monotonicity (formula (\ref{EQ:WEAKMONOTO})), and the premise $B|A$ added to $\{C|A\}$, for obtaining CM rule, satisfy the relation $AB\leq B|A$. Then,  CM rule is ``weaker'' (or more ``cautious'') than weak monotonicity, indeed $P(AB)=1$ implies $P(B|A)=1$, while the converse does not hold. Finally, in this example $\Gamma \nRightarrow_p AB|C$ because 
	\[\
	\C(\Gamma)=(C|A)\wedge (B|A)=BC|A\nleq AB|C.
	\]

\end{example}

A different weaker version of  Theorem \ref{THM:PDT} is given below, where the set of  assumptions $\Gamma \Rightarrow_p A$  and $\Gamma \cup \{A\} \Rightarrow_p B$ are replaced by the weaker assumptions 
\[
\Gamma \Rightarrow_p A|(A \vee B\vee H^*) \; \mbox{ and } \; \Gamma \cup \{A|(A \vee B\vee H^*)\}  \Rightarrow_p B|(A \vee B\vee H^*) \,.
\]
\begin{theorem}
	\label{THM:PDTWEAKSYM}	Let $A\neq \emptyset$, $B\neq \emptyset$, and $H^*$ be any events and   $\Gamma$ be any finite  family of conditional events, with $\Gamma \cup\{A|(A\vee B\vee H^*)\}$ p-consistent.	 If $\Gamma \Rightarrow_p A|(A \vee B\vee H^*)$  and $\Gamma \cup \{A|(A \vee B\vee H^*)\} \Rightarrow_p B|(A \vee B\vee H^*)$, then: $(a)\;$  $\Gamma \Rightarrow_p B|A$; $(b)\;$  $\Gamma \Rightarrow_p A|B$.
\end{theorem}
\begin{proof}
	By  applying Theorem \ref{THM:PRELPDTGEN}, with $H=K=A\vee B\vee H^*$, we obtain that  $\Gamma \Rightarrow_{p} AB|(A \vee B\vee H^*)$. Moreover,  as  $AB|(A \vee B\vee H^*)\leq B|A$ and $AB|(A \vee B\vee H^*)\leq A|B$, it follows that the conditions $(a)$ and $(b)$ are satisfied, i.e., $\Gamma \Rightarrow_{p} B|A$ and $\Gamma \Rightarrow_{p} A|B$.
	 
\end{proof}
\begin{remark}
	Given two conditional events $A|H$ and $B|K$, it holds that (\cite[Section 9]{GiSa20})
	\begin{equation}\label{EQ:CONGQC}
		(A|H)\wedge (B|K)=AHBK|(H\vee K), \;\; \mbox{when }  AH\no{K}=\no{H}BK=\emptyset.
	\end{equation}
	In other words, the conjunction $(A|H)\wedge (B|K)$ reduces to the conditional event $AHBK|(H\vee K)$, when are impossible the constituents such that a conditional event is true and the other one is void. In this case,  $(A|H)\wedge (B|K)$  coincides with $(A|H)\wedge_{df} (B|K)$.
	
\end{remark}	

\begin{example}[A different weak version of transitivity rule]
	By considering Example \ref{EX:TR}, we add to the premises  the event $A|(A\vee C)$. We  can show that, by defining  $\Gamma=\{C|B,B|A,A|(A\vee C)\}$, it holds that $\Gamma\Rightarrow_p C|A$ and $\Gamma\Rightarrow_p A|C$. 
	Indeed,  based on (\ref{EQ:CONGQC}), it holds that $(B|A)\wedge (A|(A\vee C))=AB|(A\vee C)$. 
	Then,  $\C(\Gamma)=(C|B)\wedge (B|A)\wedge(A|(A\vee C))=(C|B)\wedge (AB|(A\vee C))$ and still from (\ref{EQ:CONGQC}) it holds that $\C(\Gamma)=ABC|(A\vee B\vee C)$. As  $ABC|(A\vee B\vee C)\leq C|A$ and $ABC|(A\vee B\vee C)\leq A|C$, it follows that  $\Gamma\Rightarrow_p C|A$ and $\Gamma\Rightarrow_p A|C$.
	We observe that the same conclusions can be obtained  by applying  Theorem \ref{THM:PDTWEAKSYM}, with $H^*=\emptyset$, by verifying that  $\Gamma \Rightarrow_p A|(A\vee C)$ and $\Gamma \cup  \{A|(A\vee C)\} \Rightarrow_p C|(A\vee C)$. Indeed,  $\Gamma \Rightarrow_p A|(A\vee C)$ because $A|(A\vee C)\in \Gamma$. Moreover, $\Gamma \cup  \{A|(A\vee C)\}=\Gamma$, with $\C(\Gamma)=ABC|(A\vee B\vee C)\leq C|(A\vee C)$.  
\end{example}

In the next example we show that some counterintuitive results can depend on some strange assessments of the premises and not on the wrongness of the inference rule. In other words it can be the case that the inference rule is consistent, but the counterintuitive results of the conclusion only depend on the ``extravagant'' assessment  of the premises.
\begin{example}\label{EX:ORTOIF}
	Let $A,B$	be logically independent events, with $\no{A}\neq \emptyset$, and 
	$\Gamma=\{A\vee B\}$. As shown in Example \ref{EX:ORIF}, the inference from $\Gamma$ to the  conditional event  $B|\no{A}$ is not p-valid .	
	Based on Remark \ref{REM:GAMMA'}, we set $\Gamma'=\Gamma\cup\{\no{A}\}=\{A\vee B,\no{A}\}$.
	We observe that $\Gamma'$ is p-consistent and of course $
	\Gamma'\Rightarrow_p \no{A}$; moreover,  
	\[
	\C(\Gamma'\cup \no{A})=\C(\Gamma')=\C(\{A\vee B,\no{A}\})=(A\vee B)\wedge \no{A}=\no{A}B\leq B,
	\]
	that is 
	$\Gamma'\cup \no{A} \Rightarrow_p B$, which is the probabilistic version  
	of what in logic is often called ``disjunctive syllogism''.
	Then, by Theorem \ref{THM:PDT}, it holds that $\Gamma' \Rightarrow_{p} B|\no{A}$, that is
	\begin{equation}\label{EQ:ORTOIF'}
		\{A\vee B,\no{A}\}\Rightarrow_{p} B|\no{A}.
	\end{equation}  
	We now consider the case where 
	$A$=``Smoking is unhealthy'';
	$B$=``strawberries  are blue''.
	We observe that the assessment $P(A\vee B)=P(\no{A})=1$ is coherent, even if in this case $P(\no{A})=1$ is counterintuitive. Moreover,
	if we assess $P(A\vee B)=1$, $P(\no{A})=1$,  from (\ref{EQ:ORTOIF'}) it follows that $P(B)=P(B|\no{A})=1$, which is  counterintuitive too. 
	Of course the inference  in (\ref{EQ:ORTOIF'}) is p-valid. Actually, ``what is  strange in  this  inference'' is the probabilistic assessment of the premise $\no{A}$, i.e. $P(\no{A})=1$,
 from which the strange conclusion follows, i.e. ${P(B)=P(B|\no{A})=1}$. In a classical approach, we can reject an inference because it is not classically valid, or because we believe that it has a false premise. In a probabilistic approach, we can reject an inference because it is not p-valid, or because we disagree with the probabilistic assessment of one of the premises.
\end{example}

\section{Iterated conditionals and General Import-Export principle}
\label{SEC:IMPEXP}
In this section we examine the  notion of p-entailment and a \emph{General Import-Export principle} in relation to iterated conditionals, and also in the light of the probabilistic deduction theorems. 
In the first subsection we examine a case  where the inference rule is  not p-valid and the  Import-Export principle is not satisfied.
In the second subsection we define the notion of the General Import-Export principle and we give a result which shows the relation between the General Import-Export principle and p-entailment. In the third subsection we examine selected  p-valid inference rules by showing that the General Import-Export principle is satisfied.  
\subsection{The iterated conditional $ (B|\no{A})|(A\vee B)$}
\label{SEC:ITTOOR}
In this section, based on Example \ref{EX:ORIF}, we examine  \emph{the or-to-if inference}  in relation to iterated conditionals. We show that the invalidity of the  Classical Deduction Theorem can  
be related to  the invalidity of  the Import-Export principle.
More precisely, given two events $A$ and $B$, with $\no{A}B\neq \emptyset$,  the export from $B|((A \vee B)\no{A})$ to the iterated conditional $(B|\no{A})|(A \vee B)$ does not hold. Indeed 
\begin{equation}\label{EQ:IEA}
B|((A \vee B)\no{A}) = B|\no{A}B = 1 \neq (B|\no{A})|(A \vee B) \;.
\end{equation}
We point out that, based on
\cite[Theorem 8]{GiPS20},
if the inference  from $A\vee B$ to $B|\no{A}$ were p-valid, then 
$(B|\no{A})|(A\vee B)$ would be constant and equal to 1.
On the contrary, as shown in \cite[Section 4.2]{GiSa21A}, it holds that
$
(B|\no{A})|(A\vee B)\neq 1,
$
which  is in agreement with the p-invalidity of the or-to-if inference. Indeed, as $B|\no{A}\leq A\vee B$, it follows that $(B|\no{A})\wedge (A\vee B) = B|\no{A}$. Moreover,
by defining 
\[
P(A\vee B)=x,\; P(B|\no{A})=y, \;\prev[(B|\no{A})|(A\vee B)]=\mu,
\]
 it holds that
\[
(B|\no{A})|(A\vee B) =
(B|\no{A})\wedge (A\vee B)+ \mu \no{A}\no{B}=
B|\no{A} + \mu \no{A}\no{B} \,=
\left\{\begin{array}{ll}
	y, & \mbox{ if } A \mbox{ is true, } \\
	1, & \mbox{ if } \no{A}B \mbox{ is true, } \\
	\mu, & \mbox{ if } \no{A}\no{B} \mbox{ is true,} \\
\end{array}
\right.
\]
with $\mu\in [y,1]$. Thus $(B|\no{A})|(A\vee B)\neq 1$. In addition,
by the linearity of prevision we obtain $\mu=y + \mu (1-x)$, that is $y=\mu x$. When $x>0$ it follows that $\mu=\frac{y}{x}$, that is
\begin{equation}\label{EQ:ORTOIF}
	\pr[(B|\no{A})|(A\vee B)]=\frac{P(B|\no{A})}{P(A\vee B)},\;\; \mbox{ when } P(A\vee B)>0.
\end{equation}

As $y\leq x$, when $y=1$ it follows that $x=1$ and $\mu=1$, in which case
\[
(B|\no{A})|(A\vee B)=(A+\no{A}B)|(A\vee B)=(A\vee B)|(A\vee B)=1.
\]
We observe that, if the Import-Export principle were valid, then, for every $x$ and $y$ such  that  $0\leq y\leq x\leq 1$, the iterated conditional  $(B|\no{A})|(A\vee B)$  would coincide with the conditional
\begin{equation}\label{EQ:ORTOIFIE}
	B|(\no{A}\wedge (A\vee B))=B|\no{A}B=1,
\end{equation}
which is clearly unacceptable as shown, for instance, by the following example (\cite[p. 1]{Adam05}, see also
\cite{pfeifer12x}).
Imagine that Jones ``is about to be dealt a five card poker hand from a shuffled deck of
52 cards''. We set $A=$``Jones’s first card is not an ace'' and $B=$``Jones’s second card is  an ace''. It holds that
$P(A\vee B)=P(A)+P(B|\no{A})P(\no{A})=\frac{48}{52}+\frac{4}{51}\frac{4}{52}=\frac{616}{663}\simeq 0.929$, 
$P(B|\no{A})=\frac{3}{51}\simeq 0.059$. Then, the inference from $A\vee B$ to $B|\no{A}$ is weak, because it is  constructive, indeed $P(A)=\frac{48}{52}\simeq 0.923$  (\cite{gilio12}). Actually, in our approach from (\ref{EQ:ORTOIF}) it holds that
\[
\prev[(B|\no{A})|(A\vee B)]=\frac{\frac{3}{51}}{\frac{48}{52}+\frac{4}{51}\frac{4}{52}}=\frac{39}{616}\simeq 0.063,
\]
which is close to  $P(B|\no{A})$,  in   agreement with the intuition, and very different from the value $P(B|\no{A}B)=1$ obtained from  (\ref{EQ:ORTOIFIE}) under the Import-Export assumption.
\subsection{General Import-Export principle}
Given a family of conditional events $\F$ and a further conditional event $E|H$, by recalling Definition~\ref{DEF:GENITER}, let us consider the iterated conditionals 
\[
(E|H)|\C(\F) = (E|H) \wedge \C(\F) + \mu \,[1-\C(\F)] \,,\;\;  E \,|\,(H \wedge \C(\F)) = EH \wedge \C(\F) + \eta\,[1-H \wedge \C(\F)], 
\]
where $\mu = \pr[(E|H)|\C(\F)]$ and $\eta = \pr[E \,|\,(H \wedge \C(\F))]$.
\begin{definition}\label{DEF:GIEPR}
The iterated conditional $(E|H)|\C(\F)$ satisfies the General Import-Export principle if it holds that
\begin{equation} \label{EQ:GIP}
(E|H)|\C(\F) = E \,|\,(H \wedge \C(\F)). 
\end{equation} 
\end{definition}
The next result relates the General Import-Export principle to the notions of p-entailment and p-consistency.
\begin{theorem}\label{THM:GIEENT}
Let  a p-consistent family of conditional events $\F$ and a further conditional event $E|H$ be given. If $\F$ p-entails $E|H$ and $\F \cup \{H\}$ is p-consistent, then  $(E|H)|\C(\F) = E|(H \wedge \C(\F)) = 1$, and hence the General Import-Export principle is satisfied.
\end{theorem}
\begin{proof} By recalling Theorem \ref{THM:CONVERSE}, as $\F$ p-entails $E|H$ and $\F \cup \{H\}$ is p-consistent, it follows that $\F \cup \{H\}$ p-entails $E$. Then, 
by Theorem~\ref{THM:PEITER}, it follows that  $(E|H)|\C(\F) = E|(H \wedge \C(\F)) = 1$. Thus, the 
 General Import-Export principle is satisfied.
\end{proof} 
\subsection{General Import-Export principle and p-validity}
We examine the (p-valid) Cautious Monotonicity (CM), Cumulative Transitivity (Cut), and OR  inference rules of System P. For each inference rule, we show that the General Import-Export principle is satisfied by verifying the p-consistency of the family $\F \cup \{H\}$, where $\F$ is the set of premises and $H$ is the antecedent of the conclusion.

\paragraph{CM rule: $\{C|A \,,\; B|A \}\; \Rightarrow_p \; C|AB$}
 In this case $\F = \{C|A \,,\; B|A \}$ and $E|H = C|AB$. 
We set $P(C|A)=x, P(B|A)=y, P(C|AB)=z$, and we observe that, as $\F$ p-entails $C|AB$, the assessment $(x,y,z)=(1,1,1)$ on $\{C|A,B|A,C|AB\}$ is coherent. In particular, the assessment $z=P(C|AB)=1$ is coherent and hence $ABC \neq \emptyset$. 
The CM rule has been characterized in terms of the iterated conditional $(C|AB)|((C|A)\wedge (B|A))$ in \cite{GiPS20},  by showing that
\begin{equation}\label{EQ:CUTRULEITER}
	(C|AB)|((C|A)\wedge (B|A))=1.
\end{equation}
We observe that 
\[
(C|A) \wedge (B|A) \wedge AB = BC|A \wedge A \wedge B = ABC \neq \emptyset \,.
\]
Thus, the assessment $\pr[(C|A) \wedge (B|A) \wedge AB] = 1$, that is $P(ABC)=1$, is coherent because $ABC \neq \emptyset$ and, by Theorem \ref{THM:PCC}, the family $\{C|A,B|A, AB\}$ is p-consistent. Then, by Theorem \ref{THM:GIEENT}, the General Import-Export principle is satisfied, with
\[
(C|AB)|[(C|A) \wedge (B|A)] = C|[AB \wedge (C|A) \wedge (B|A)] = 1 \,.
\] 
We also observe that 
$\{C|A, B|A\}  \nRightarrow_p AB$ because
\[
(C|A) \wedge (B|A) \wedge AB = ABC \;\neq\; BC|A= (C|A) \wedge (B|A).
\]
Then, CM rule  provides an example where the p-entailment of $C|AB$ from $\{C|A, B|A\}$ does not follow from Theorem \ref{THM:PDT}  (Probabilistic Deduction Theorem). Indeed,
$\{C|A, B|A\} \cup \{AB\} \Rightarrow_p C$, but $\{C|A, B|A\} \nRightarrow_p AB$; thus not all assumptions of Theorem \ref{THM:PDT} are satisfied.
\paragraph{Cut rule: $\{C|AB \,,\; B|A \}\; \Rightarrow_p \; C|A$}
In this case $\F = \{C|AB,B|A \}$ and $E|H = C|A$. We set $P(C|AB)=x, P(B|A)=y, P(C|A)=z$, and we observe that, as $\F$ p-entails $C|A$, the assessment $(x,y,z)=(1,1,1)$ on $\{C|AB,B|A,C|A\}$ is coherent. In particular, the assessment $x=P(C|AB)=1$ is coherent and hence $ABC \neq \emptyset$. 
The Cut rule has been characterized in terms of the iterated conditional $(C|A)|((C|AB)\wedge (B|A))$ in \cite{GiPS20},  by showing that
\begin{equation}
	(C|A)|((C|AB)\wedge (B|A))=1.
\end{equation}
We observe that
\begin{equation*}
(C|AB)\wedge (B|A) \wedge A = (C|AB) \wedge AB = ABC \neq \emptyset \,.
\end{equation*}
Thus, the assessment $\pr[(C|AB) \wedge (B|A) \wedge A] = 1$, that is $P(ABC)=1$, is coherent because $ABC \neq \emptyset$, and, by Theorem \ref{THM:PCC}, the family $\{C|AB, B|A, C|A\}$ is p-consistent. Then,  by Theorem \ref{THM:GIEENT}, the General Import-Export principle is satisfied, with
\[
(C|A)|[(C|AB) \wedge (B|A)] = C|[A \wedge (C|AB) \wedge (B|A)] = 1 \,.
\]
We also observe that   
$\{C|AB,B|A\} \nRightarrow_p A$ because
\[
(C|AB) \wedge (B|A) \wedge A = ABC \;\neq\; BC|A= (C|AB) \wedge (B|A) \,.
\]
Therefore $\{C|AB,B|A\} \Rightarrow_p C|A$, and   $\{C|AB,B|A\} \cup \{A\}\Rightarrow_p C$, but  $\{C|AB,B|A\} \nRightarrow_p A$. In other words, the Cut rule cannot be obtained by applying Theorem \ref{THM:PDT}.
\paragraph{Or rule: $\{C|A \,,\; C|B \}\; \Rightarrow_p \; C|(A\vee B)$}
In this case $\F = \{C|A,C|B \}$ and $E|H = C|(A \vee B)$. \\
The Or rule has been characterized in (\cite{GiPS20}) in terms of the iterated conditional $(C|(A\vee B))|((C|A)\wedge (C|B))$,  by showing that
\begin{equation}\label{EQ:ORRULEITER}
	(C|(A\vee B))|((C|A)\wedge (C|B))=1.
\end{equation}
As $\F$ p-entails $C|(A\vee B)$, the assessment $P(C|A)=P(C|B)=P[C|(A \vee B)]=1$ is coherent, and this implies that 
\[
AC \neq \emptyset \,,\;\;\; BC \neq \emptyset \,,\;\;\; C \wedge (A \vee B) = ABC \vee A\no{B}C \vee \no{A}BC \neq \emptyset \,.
\]
Then, at least one of the following conditions must be satisfied: 
\[
(i) \;\;\;\; ABC \neq \emptyset \,,\;\;\;\;\; (ii) \;\;\;\; A\no{B}C \neq \emptyset \;\; \mbox{and} \;\;\no{A}BC \neq \emptyset \,. 
\]
By considering the partition $\{ABC,\, \no{A}BC,\, A\no{B}C,\, (A \vee B) \wedge \no{C},\, \no{A}\;\no{B}\}$, in the case $(i)$ we assess 
 \[
\P_1: \;\;\;\;\;\;\;\; P(ABC) \;=\; 1 \,,\;\; P(\no{A}BC) \;=\; P(A\no{B}C) \;=\; P[(A \vee B) \wedge \no{C}] \;=\; P(\no{A}\,\no{B}) \;=\; 0 \,;
 \]
in the case $(ii)$ we assess 
 \[
\P_2: \;\;\;\;\;\; P(ABC) \;=\; 0 \,,\;\; P(\no{A}BC) \;=\; P(A\no{B}C) \;=\; \frac12 \,,\;\;  P[(A \vee B) \wedge \no{C}] \;=\; P(\no{A}\, \no{B}) \;=\; 0 \,.
 \]
For both assessments $\P_1$ and $\P_2$, the (unique and coherent) propagation to the family $\{C|A, C|B, A \vee B\}$ is given by
 $P(C|A) \;=\; P(C|B) \;=\; P(A \vee B) \;=\; 1$. 
Therefore, the family $\{C|A, C|B,A \vee B\}$ is p-consistent and, by Theorem \ref{THM:GIEENT}, the General Import-Export principle is satisfied, with
 \begin{equation}\label{EQ:ORRULEITER4}
 (C|(A\vee B))\,|\, (C|A)\wedge (C|B)) = C|((A\vee B) \wedge (C|A)\wedge (C|B)) = 1.
 \end{equation}
 We also observe that 
$\{C|A, C|B\} \nRightarrow_p A\vee B$ because;
\[
(C|A) \wedge (C|B)\nleq (A\vee B);
\]
indeed when $A\vee B=0$ it holds that $(C|A) \wedge (C|B)=\prev[C|A) \wedge (C|B)]$, where $\prev[C|A) \wedge (C|B)]$ is not necessarily zero.
Therefore   $\{C|A, C|B\} \Rightarrow_p C|(A\vee B)$ and $\{C|A), C|B\} \cup \{A\vee B\}\Rightarrow_p C$, but   $\{C|A, C|B\} \nRightarrow_p A\vee B$.  In other words, Or rule cannot be obtained by applying Theorem \ref{THM:PDT}.

\section{Some further remarks}\label{EQ:FURTHER}
In this section we summarize some aspects related with  the p-entailment of $E|H$ from $\F$, and of $E$ from $\F\cup \{H\}$. 
As it is shown by the results in the previous sections,  given a p-consistent family of $n$ conditional events $\F = \{E_1|H_1, \ldots, E_n|H_n\}$ and a further conditional event $E|H$, there are two cases:\\ $(a)$ $\F\cup \{H\}$ is not p-consistent; $(b)$ $\F\cup \{H\}$ is p-consistent.\\
In case $(a)$ there are two sub-cases: $(a.1)$ $\F\nRightarrow_p E|H$; $(a.2)$ $\F\Rightarrow_p E|H$. \\
$(a.1).$ An example is given by the case where  $\F=\{A\}$, $E|H=\no{A}$, with $A\neq \emptyset$. Indeed,  $\{A\}$ is p-consistent, $\{A,\no{A}\}$ is not p-consistent and $\{A\}\nRightarrow_p \no{A}$. \\
$(a.2).$ An example is given by the case where  $\F=\{AB|(A\vee B)\}$, $E|H=(\no{A}\vee \no{B})|A\no{B}$, with $A,B$ logically independent.  Of course, $\{AB|(A\vee B)\}$ is p-consistent.
The family 
$\{AB|(A\vee B),A\no{B}\} $ is not p-consistent; indeed, the assessment $(1,1)$ on $\{AB|(A \vee B), A\no{B}\}$ is not coherent because, as it can be verified, the points $Q_h$'s are $(1,0), (0,1), (0,0)$ and $(1,1)$ do not belong to their convex hull. Moreover, as $(\no{A} \vee \no{B})|A\no{B} = A\no{B}|A\no{B} = 1$, it holds that $\{AB|(A \vee B)\}  \Rightarrow_p (\no{A} \vee \no{B})|A\no{B}$.
\\
In case $(b)$, first of all we observe that if $\F \Rightarrow_p E|H$, as $\F\cup \{H\}$ is p-consistent, by Theorem~\ref{THM:CONVERSE} it follows  that $\F\cup \{H\} \Rightarrow_p E$. Thus, it is not possible that
$\F \Rightarrow_p E|H$ and $\F \cup \{H\} \nRightarrow_p E$.
Then, we have the following cases: 
\[
\begin{array}{lll}
(b.1) \; \F \nRightarrow_p E|H \;\mbox{and}\; \F \cup \{H\} \nRightarrow_p E;\\ \ \\  
(b.2) \; \F \nRightarrow_p E|H \;\mbox{and}\; \F \cup \{H\} \Rightarrow_p E;\; \\ \ \\
(b.3) \; \F \Rightarrow_p E|H \;\mbox{and}\; \F \cup \{H\} \Rightarrow_p E.
\end{array} 
\]
$(b.1).$
This case arises, for instance, when the basic events $E_1,H_1, \ldots, E_n,H_n,E,H$ are logically independent. In this case  the assessment 
\[
P(E_1|H_1)=x_1,\; \ldots, \; P(E_n|H_n)=x_n,\; 
P(E|H)=y,
P(H)=z,
P(E)=t.
\]
is coherent for every $(x_1,\ldots,x_n,y,z) \in 
[0,1]^{n+2}$ and $yz \leq t \leq 1$. Then, $\F$ and $\F\cup\{H\}$ are p-consistent. But, from $(x_1,\ldots,x_n)=(1,\ldots,1)$ it does not necessarily follow that $y=1$. Moreover, from $(x_1,\ldots,x_n,z)=(1,\ldots,1,1)$ it does not necessarily follow that $t=1$. Thus, $\F \nRightarrow_p E|H$ and  $\F \cup \{H\} \nRightarrow_p E$.\\
$(b.2).$
This case concerns, for instance, the examples \ref{EX:TR}, \ref{EX:BOOLE}, \ref{EX:CONTR}, and \ref{EX:MONOT}. \\
$(b.3).$
This  case arises when $\F \Rightarrow_p E|H$ and $\F \cup \{H\}$ is p-consistent, from which, by Theorem~\ref{THM:CONVERSE}, it follows  that $\F\cup \{H\} \Rightarrow_p E$. For instance, in the CM rule, where $\F=\{C|A,C|B\}$ and $E|H=C|AB$, it holds that $\{C|A,C|B\}\Rightarrow_p C|AB$ and $\{C|A,C|B,AB\}$ is p-consistent; thus 
$\{C|A,C|B,AB\}\Rightarrow_p C$. A similar reasoning applies to the  Cut and the Or rules.

\section{Related work}
\label{SEC:RW}
In this section we briefly make a comparison with  related work, especially in the field of AI. 
The contributions of this paper are summarised as follows:
$(i)$  we make a comparison with some trivalent logics; $(ii)$
we provide some probabilistic versions of the classical deduction theorem; $(iii)$ we introduce a General Import-Export principle, by relating it  to the notion of p-entailment.\\
There are several works which are related to this paper, see e.g.,  
\cite{benferhat97,Delgrande2019,DuPr94,FlGH20,KR22,freund1991,Kern-Isb2023,KrLM90,LeMa92,Luka99,mundici21}. \\
For instance, in \cite{FlGH20} a Boolean algebra of conditionals has been given, where some probabilistic results have been obtained based on a suitable   canonical extension of a probability measure. Notice that in \cite{FlGH20}  conditionals are logical objects  and their interpretation remains at the symbolic level, without a  numerical counterpart.
In our paper  compound and iterated conditionals  are suitable conditional random quantities,   defined in the setting of coherence (with a betting interpretation),
which 
allow to obtain probabilistic results   in agreement with \cite{FlGH20} (see \cite{FGGS-ipmu2022}), with  
further developments  given in \cite{KR22}.
\\
As  further instances, we recall that in  \cite{KrLM90} and \cite{benferhat97}  nonmonotonic systems for reasoning with conditional statements have been studied.
In particular, in
 \cite{KrLM90} several  families of nonmonotonic    consequence relations have been examined. 
 Moroever,  the well known  System P, which is related to the Adams' logic of infinitesimal probabilities,  has been introduced in \cite{KrLM90}.
We observe that in \cite{benferhat97} the basic tools for the study of nonmonotonic systems are  three-valued conditional objects,  which are a qualitative counterpart to conditional probabilities  within the framework of  preferential entailment. 
\\
In our approach, 
the basic tools are  compound and iterated conditionals, by means of which we can provide a probabilistic semantics for  more general conditional inference rules. In this paper 
we introduce  a General Import-Export principle and we relate it to the property of p-validity, by examining  selected inference rules of System P: CM, Cut, Or rules.  \\ 
Finally, as a further difference w.r.t. other approaches in  trivalent logics (see, e.g., \cite{EgRS20,Kern-Isb2023,CiDu13}), our compound and iterated conditionals are (not three-valued objects, but) suitable many-valued objects, which are defined at Level 2 of knowledge. 
By means of them we compare some notions of validity for inference rules, in selected trivalent logics, with the notion of p-validity of Adams in the setting of coherence. 
Moreover, after remarking that the Classical Deduction Theorem is not p-valid, we provide   several probabilistic valid versions of it.

\section{Conclusions}
\label{SEC:CONCL}
There has been increasing interest in recent years in de Finetti’s analysis of conditional events. In his early work, he classified a conditional event, $B|A$, as a three-valued entity, which is true when A and B are true, false when $A$ is true and B is false, and ``void'' when $A$ is false. But we have pointed out that, in his later research, he distinguished a higher level of analysis in which the void value becomes the conditional probability, $P(B|A)$. We have acknowledged that there are some benefits in continuing to make the early simple three-valued distinction. But we have also argued that there are significant advantages to adopting the later, higher-level analysis. At this higher level, we can give an account of certain compound conditionals, and we can provide an intuitive definition of probabilistic entailment, p-entailment. By this definition, a set of premises p-entails a conclusion if and only if the conclusion has a probability of 1 when the premises have a probability of 1. We have specified exactly how this definition of p-entailment is related to less straightforward definitions of entailment using just the three values. We have done this by examining two recent articles that are restricted to de Finetti’s trivalent analysis. We showed that the Classical Deduction Theorem fails for p-entailment. We used the or-to-if inference, from $A$ or $B$ to {\em if  not-$A$ then $B$}, to illustrate the invalidity of the Classical Deduction Theorem, and the invalidity of the Import-Export principle, for p-entailment. But we also proved  probabilistic versions of the  deduction theorem for p-entailment.
We  introduced  a General Import-Export principle for iterated conditionals by relating it to  p-consistency and p-entailment. Finally we illustrated the   validity of the General Import-Export principle in some p-valid inference rules of System P: CM, Cut, Or. We also observed that in these examples  the  probabilistic deduction theorem is not appliable.  We plan to extend these results to further relevant cases, such as Aristotelian syllogisms, further nonmonotonic inference schemes,  knowledge representation in AI, and  the psychology of human reasoning under uncertainty.

\section*{Declaration of competing interest}
The authors declare that they have no known competing financial interests or personal relationships that could have appeared to influence the work reported in this paper.

\section*{Acknowledgements}
Giuseppe Sanfilippo acknowledges support by the FFR project of University of Palermo, Italy, and by INdAM-GNAMPA research group.


\begin{thebibliography}{10}
	
	\bibitem{adams75}
	E.~W. Adams.
	\newblock {\em The logic of conditionals}.
	\newblock Reidel, Dordrecht, 1975.
	
	\bibitem{adams98}
	E.~W. Adams.
	\newblock {\em A primer of probability logic}.
	\newblock CSLI, {S}tanford, 1998.
	
	\bibitem{Adam05}
	E.~W. Adams.
	\newblock What is at stake in the controversy over conditionals.
	\newblock In Gabriele Kern-Isberner, Wilhelm R{\"o}dder, and Friedhelm Kulmann,
	editors, {\em Conditionals, Information, and Inference}, pages 1--11, Berlin,
	Heidelberg, 2005. Springer Berlin Heidelberg.
	
	\bibitem{BaOP13}
	J.~Baratgin, D.~E. Over, and G.~Politzer.
	\newblock Uncertainty and the de finetti tables.
	\newblock {\em Thinking \& Reasoning}, 19(3-4):308--328, 2013.
	
	\bibitem{BPOT18}
	J.~Baratgin, G.~Politzer, D.~E. Over, and T.~Takahashi.
	\newblock The psychology of uncertainty and three-valued truth tables.
	\newblock {\em Frontiers in Psychology}, 9:1479, 2018.
	
	\bibitem{BeKr01}
	D.~Beaver and E.~Krahmer.
	\newblock A partial account of presupposition projection.
	\newblock {\em Journal of Logic, Language and Information}, 10(2):147, 2001.
	
	\bibitem{benferhat97}
	S.~Benferhat, D.~Dubois, and H.~Prade.
	\newblock Nonmonotonic reasoning, conditional objects and possibility theory.
	\newblock {\em Artificial Intelligence}, 92:259--276, 1997.
	
	\bibitem{biazzo02}
	V.~Biazzo, A.~Gilio, T.~Lukasiewicz, and G.~Sanfilippo.
	\newblock Probabilistic logic under coherence, model-theoretic probabilistic
	logic, and default reasoning in {S}ystem {P}.
	\newblock {\em Journal of Applied Non-Classical Logics}, 12(2):189--213, 2002.
	
	\bibitem{biazzo05}
	V.~Biazzo, A.~Gilio, T.~Lukasiewicz, and G.~Sanfilippo.
	\newblock Probabilistic logic under coherence: {C}omplexity and algorithms.
	\newblock {\em Annals of Mathematics and Artificial Intelligence},
	45(1-2):35--81, 2005.
	
	\bibitem{BiGS03Wupes}
	V.~Biazzo, A.~Gilio, and G.~Sanfilippo.
	\newblock Logical conditions for coherent qualitative and numerical probability
	assessments.
	\newblock In {\em 6th Workshop on Uncertainty Processing, WUPES'2003, Hejnice,
		Czech Republic, 24 - 27th September}, pages 1--12, 2003.
	
	\bibitem{Bochvar37}
	D.~A. Bochvar.
	\newblock On a three-valued logical calculus and its application to the
	analysis of the paradoxes of the classical extended functional calculus.
	\newblock {\em History and Philosophy of Logic}, 2(1-2):87--112, 1981.
	\newblock English translation by M. Bergmann. First appeared in
	Mathematicheskii Sbornik 4, 287–308 (1937).
	
	\bibitem{Cala87}
	P.~Calabrese.
	\newblock An algebraic synthesis of the foundations of logic and probability.
	\newblock {\em Information Sciences}, 42(3):187 -- 237, 1987.
	
	\bibitem{Cala17}
	P.~Calabrese.
	\newblock {\em Logic and Conditional Probability: A Synthesis}.
	\newblock College Publications, 2017.
	
	\bibitem{Cantwell22}
	J.~Cantwell.
	\newblock Revisiting {McGee}'s probabilistic analysis of conditionals.
	\newblock {\em Journal of Philosophical Logic}, 51(5):973--1017, 2022.
	
	\bibitem{CaLS07}
	A.~Capotorti, F.~Lad, and G.~Sanfilippo.
	\newblock Reassessing accuracy rates of median decisions.
	\newblock {\em American Statistician}, 61(2):132--138, 2007.
	
	\bibitem{CiDu12}
	D.~Ciucci and D.~Dubois.
	\newblock {Relationships between Connectives in Three-Valued Logics}.
	\newblock In {\em Advances on Computational Intelligence}, volume 297 of {\em
		CCIS}, pages 633--642. Springer, 2012.
	
	\bibitem{CiDu13}
	D.~Ciucci and D.~Dubois.
	\newblock A map of dependencies among three-valued logics.
	\newblock {\em Information Sciences}, 250:162 -- 177, 2013.
	
	\bibitem{CERV12}
	P.~Cobreros, P.~Egr{\'e}, D.~Ripley, and R.~van Rooij.
	\newblock Tolerant, classical, strict.
	\newblock {\em Journal of Philosophical Logic}, 41(2):347--385, 2012.
	
	\bibitem{Cole90}
	G.~Coletti.
	\newblock Coherent qualitative probability.
	\newblock {\em Journal of Mathematical Psychology}, 34(3):297--310, 1990.
	
	\bibitem{Cole94}
	G.~Coletti.
	\newblock Coherent numerical and ordinal probabilistic assessments.
	\newblock {\em IEEE Trans. on Systems, Man, and Cybernetics},
	24(12):1747--1754, 1994.
	
	\bibitem{ColettiPV19}
	G.~Coletti, D.~Petturiti, and B.~Vantaggi.
	\newblock Dutch book rationality conditions for conditional preferences under
	ambiguity.
	\newblock {\em Ann. Oper. Res.}, 279(1-2):115--150, 2019.
	
	\bibitem{coletti02}
	G.~Coletti and R.~Scozzafava.
	\newblock {\em Probabilistic logic in a coherent setting}.
	\newblock Kluwer, Dordrecht, 2002.
	
	\bibitem{CoSV15}
	G.~Coletti, R.~Scozzafava, and B.~Vantaggi.
	\newblock Possibilistic and probabilistic logic under coherence: Default
	reasoning and {S}ystem {P}.
	\newblock {\em Mathematica Slovaca}, 65(4):863--890, 2015.
	
	\bibitem{Cruz20}
	N.~Cruz.
	\newblock Deduction from uncertain premises?
	\newblock In S.~Elqayam, I.~Douven, J.~St B.~T. Evans, and N.~Cruz, editors,
	{\em Logic and Uncertainty in the Human Mind: A Tribute to David E. Over},
	pages 27--41. Routledge, Oxon, 2020.
	
	\bibitem{Cruz2015}
	N.~Cruz, J.~Baratgin, M.~Oaksford, and D.~E. Over.
	\newblock Bayesian reasoning with ifs and ands and ors.
	\newblock {\em Frontiers in Psychology}, 6, 2015.
	
	\bibitem{definetti31}
	B.~{de~Finetti}.
	\newblock Sul significato soggettivo della probabilità.
	\newblock {\em Fundamenta Mathematicae}, 17:298--329, 1931.
	
	\bibitem{definetti37}
	B.~{de~Finetti}.
	\newblock Foresight: {I}ts logical laws, its subjective sources.
	\newblock In {\em Studies in subjective probability}, pages 55--118. Krieger,
	Huntington, 1980.
	\newblock English translation of La Pr\'{e}vision: Ses Lois Logiques, Ses
	Sources Subjectives, Annales de l'Institut Henri Poincar\'{e}, 17:1-68
	(1937).
	
	\bibitem{deFi80}
	B.~de~Finetti.
	\newblock Probabilit\'a.
	\newblock In {\em Enciclopedia}, pages 1146--1187. Einauidi, 1980.
	
	\bibitem{definetti36}
	B.~{de~Finetti}.
	\newblock The logic of probability.
	\newblock {\em Philosophical Studies}, 77:181--190, 1995.
	\newblock English translation of La Logique de la Probabilit\'{e}. In Actes du
	Congr\`{e}s International de Philosophie Scientifique, Paris, 1935, pages IV
	1--IV 9. Hermann et C.ie, Paris, (1936).
	
	\bibitem{Delgrande2019}
	J.~P. Delgrande, B.~Renne, and J.~Sack.
	\newblock The logic of qualitative probability.
	\newblock {\em Artificial Intelligence}, 275:457--486, 2019.
	
	\bibitem{douven11b}
	I.~Douven and R.~Dietz.
	\newblock A puzzle about {S}talnaker's hypothesis.
	\newblock {\em Topoi}, pages 31--37, 2011.
	
	\bibitem{DuPr94}
	D.~Dubois and H.~Prade.
	\newblock Conditional objects as nonmonotonic consequence relationships.
	\newblock {\em IEEE Trans. on Syst. Man and Cybernetics,}, 24(12):1724 --1740,
	dec 1994.
	
	\bibitem{edgington95}
	D.~Edgington.
	\newblock On conditionals.
	\newblock {\em Mind}, 104:235--329, 1995.
	
	\bibitem{EgRS20}
	P.~Egr{\'e}, L.~Rossi, and J.~Sprenger.
	\newblock De {F}inettian {L}ogics of {I}ndicative {C}onditionals {P}art {I}:
	{T}rivalent {S}emantics and {V}alidity.
	\newblock {\em J Philos Logic}, 2020.
	
	\bibitem{FGGS-ipmu2022}
	T.~Flaminio, A.~Gilio, L.~Godo, and G.~Sanfilippo.
	\newblock Canonical extensions of conditional probabilities and compound
	conditionals.
	\newblock In {\em Information Processing and Management of Uncertainty in
		Knowledge-Based Systems - 19th International Conference, {IPMU} 2022, Milan,
		Italy, July 11-15, 2022, Proceedings, Part {II}}, volume 1602 of {\em
		Communications in Computer and Information Science}, pages 584--597.
	Springer, 2022.
	
	\bibitem{KR22}
	T.~Flaminio, A.~Gilio, L.~Godo, and G.~Sanfilippo.
	\newblock {Compound Conditionals as Random Quantities and Boolean Algebras}.
	\newblock In Gabriele Kern{-}Isberner, Gerhard Lakemeyer, and Thomas Meyer,
	editors, {\em Proceedings of the 19th International Conference on Principles
		of Knowledge Representation and Reasoning, {KR} 2022, Haifa, Israel. July 31
		- August 5, 2022}, pages 141--151, 2022.
	
	\bibitem{FlGH20}
	T.~Flaminio, L.~Godo, and H.~Hosni.
	\newblock Boolean algebras of conditionals, probability and logic.
	\newblock {\em Artificial Intelligence}, 286:103347, 2020.
	
	\bibitem{Franks2021}
	C.~Franks.
	\newblock The deduction theorem (before and after {H}erbrand).
	\newblock {\em History and Philosophy of Logic}, 42(2):129--159, 2021.
	
	\bibitem{freund1991}
	M.~Freund, D.~Lehmann, and P.~Morris.
	\newblock Rationality, transitivity, and contraposition.
	\newblock {\em Artificial Intelligence}, 52(2):191--203, 1991.
	
	\bibitem{gilio90}
	A.~Gilio.
	\newblock Criterio di penalizzazione e condizioni di coerenza nella valutazione
	soggettiva della probabilit{\`a}.
	\newblock {\em Bollettino dell'Unione Matematica Italiana}, 4--B:645--660,
	1990.
	
	\bibitem{gilio02}
	A.~Gilio.
	\newblock Probabilistic reasoning under coherence in {S}ystem {P}.
	\newblock {\em Annals of Mathematics and Artificial Intelligence}, 34:5--34,
	2002.
	
	\bibitem{GOPS16}
	A.~Gilio, D.~Over, N.~Pfeifer, and G.~Sanfilippo.
	\newblock Centering and compound conditionals under coherence.
	\newblock In {\em Soft Methods for Data Science}, volume 456 of {\em AISC},
	pages 253--260. Springer, 2017.
	
	\bibitem{gilio12}
	A.~Gilio and D.~E. Over.
	\newblock The psychology of inferring conditionals from disjunctions: {A}
	probabilistic study.
	\newblock {\em Journal of Mathematical Psychology}, 56:118--131, 2012.
	
	\bibitem{gilio16}
	A.~Gilio, N.~Pfeifer, and G.~Sanfilippo.
	\newblock Transitivity in coherence-based probability logic.
	\newblock {\em Journal of Applied Logic}, 14:46--64, 2016.
	
	\bibitem{GiPS20}
	A.~Gilio, N.~Pfeifer, and G.~Sanfilippo.
	\newblock Probabilistic entailment and iterated conditionals.
	\newblock In S.~Elqayam, I.~Douven, J.~St B.~T. Evans, and N.~Cruz, editors,
	{\em Logic and Uncertainty in the Human Mind: A Tribute to David E. Over},
	pages 71--101. Routledge, Oxon, 2020.
	
	\bibitem{gilio10}
	A.~Gilio and G.~Sanfilippo.
	\newblock Quasi conjunction and p-entailment in nonmonotonic reasoning.
	\newblock In C.~Borgelt, G.~Gonz{\'a}lez-Rodr{\'i}guez, W.~Trutschnig, M.~A.
	Lubiano, M.~{\'A}. Gil, P.~Grzegorzewski, and O.~Hryniewicz, editors, {\em
		Combining Soft Computing and Statistical Methods in Data Analysis},
	{A}dvances in {I}ntelligent and {S}oft {C}omputing, pages 321--328.
	Springer-{V}erlag, 2010.
	
	\bibitem{gilio11ecsqaru}
	A.~Gilio and G.~Sanfilippo.
	\newblock Quasi conjunction and inclusion relation in probabilistic default
	reasoning.
	\newblock In W.~Liu, editor, {\em Symbolic and Quantitative Approaches to
		Reasoning with Uncertainty}, Lecture Notes in Computer Science, pages
	497--508. Springer, Berlin, 2011.
	
	\bibitem{GiSa13a}
	A.~Gilio and G.~Sanfilippo.
	\newblock {C}onjunction, disjunction and iterated conditioning of conditional
	events.
	\newblock In {\em Synergies of Soft Computing and Statistics for Intelligent
		Data Analysis}, volume 190 of {\em AISC}, pages 399--407. Springer, Berlin,
	2013.
	
	\bibitem{GiSa13IJAR}
	A.~Gilio and G.~Sanfilippo.
	\newblock Probabilistic entailment in the setting of coherence: {T}he role of
	quasi conjunction and inclusion relation.
	\newblock {\em International Journal of Approximate Reasoning}, 54(4):513--525,
	2013.
	
	\bibitem{gilio13}
	A.~Gilio and G.~Sanfilippo.
	\newblock Quasi conjunction, quasi disjunction, t-norms and t-conorms:
	{P}robabilistic aspects.
	\newblock {\em Information Sciences}, 245:146--167, 2013.
	
	\bibitem{GiSa14}
	A.~Gilio and G.~Sanfilippo.
	\newblock Conditional random quantities and compounds of conditionals.
	\newblock {\em Studia Logica}, 102(4):709--729, 2014.
	
	\bibitem{GiSa19}
	A.~Gilio and G.~Sanfilippo.
	\newblock Generalized logical operations among conditional events.
	\newblock {\em Applied Intelligence}, 49(1):79--102, Jan 2019.
	
	\bibitem{GiSa20}
	A.~Gilio and G.~Sanfilippo.
	\newblock Algebraic aspects and coherence conditions for conjoined and
	disjoined conditionals.
	\newblock {\em International Journal of Approximate Reasoning}, 126:98 -- 123,
	2020.
	
	\bibitem{GiSa21}
	A.~Gilio and G.~Sanfilippo.
	\newblock Compound conditionals, {F}r\'echet-{H}oeffding bounds, and {F}rank
	t-norms.
	\newblock {\em International Journal of Approximate Reasoning}, 136:168--200,
	2021.
	
	\bibitem{GiSa21E}
	A.~Gilio and G.~Sanfilippo.
	\newblock Iterated conditionals and characterization of p-entailment.
	\newblock In J.~Vejnarov\'a and N.~Wilson, editors, {\em Symbolic and
		Quantitative Approaches to Reasoning with Uncertainty, ECSQARU 2021}, volume
	12897 of {\em LNCS}, pages 629--643. Springer International Publishing, 2021.
	
	\bibitem{GiSa21A}
	A.~Gilio and G.~Sanfilippo.
	\newblock On compound and iterated conditionals.
	\newblock {\em Argumenta}, 6(2):241--266, 2021.
	
	\bibitem{GoNg88}
	I.~R. Goodman and H.~T. Nguyen.
	\newblock {Conditional Objects and the Modeling of Uncertainties}.
	\newblock In M.~M. Gupta and T.~Yamakawa, editors, {\em Fuzzy Computing}, pages
	119--138. North-Holland, 1988.
	
	\bibitem{GoNW91}
	I.~R. Goodman, Hung~T. Nguyen, and Elbert~A. Walker.
	\newblock {\em Conditional Inference and Logic for Intelligent Systems: A
		Theory of Measure-Free Conditioning}.
	\newblock North-Holland, 1991.
	
	\bibitem{Kern-Isb2023}
	J.~Heyninck, G.~Kern-Isberner, T.~Rienstra, K.~Skiba, and M.~Thimm.
	\newblock Revision, defeasible conditionals and non-monotonic inference for
	abstract dialectical frameworks.
	\newblock {\em Artificial Intelligence}, 317:103876, 2023.
	
	\bibitem{jeffrey91}
	R.~Jeffrey.
	\newblock Matter-of-fact conditionals i.
	\newblock {\em Proceedings of the Aristotelian Society}, 65:161--209, 1991.
	
	\bibitem{Kauf09}
	S.~Kaufmann.
	\newblock Conditionals right and left: Probabilities for the whole family.
	\newblock {\em Journal of Philosophical Logic}, 38:1--53, 2009.
	
	\bibitem{Kleene38}
	S.~C. Kleene.
	\newblock On notation for ordinal numbers.
	\newblock {\em Journal of Symbolic Logic}, 3(4):150–155, 1938.
	
	\bibitem{Kleene52}
	S.~C. Kleene.
	\newblock {\em {Introduction to Metamathematics}}.
	\newblock North-Holland, Amsterdam, 1952.
	
	\bibitem{kleene02}
	S.~C. Kleene.
	\newblock {\em Mathematical Logic}.
	\newblock Dover Publications, Dover, 2002.
	\newblock Original work published in 1967.
	
	\bibitem{kleiter18}
	G.~D. Kleiter, A.~J.~B. Fugard, and N.~Pfeifer.
	\newblock A process model of the understanding of uncertain conditionals.
	\newblock {\em Thinking \& Reasoning}, 24(3):386--422, 2018.
	
	\bibitem{kneale84}
	W.~Kneale and M.~Kneale.
	\newblock {\em The development of logic}.
	\newblock Clarendon Press, Oxford, 1984.
	
	\bibitem{KrPS59}
	C.~H. Kraft, J.~W. Pratt, and A.~Seidenberg.
	\newblock {Intuitive Probability on Finite Sets}.
	\newblock {\em The Annals of Mathematical Statistics}, 30(2):408 -- 419, 1959.
	
	\bibitem{KrLM90}
	S.~Kraus, D.~Lehmann, and M.~Magidor.
	\newblock {Nonmonotonic reasoning, preferential models and cumulative logics}.
	\newblock {\em Artif. Intell.}, 44:167--207, 1990.
	
	\bibitem{Lad95}
	F.~Lad.
	\newblock Coherent prevision as a linear functional without an underlying
	measure space: the purely arithmetic structure of conditional quantities.
	\newblock In G~Coletti et~al., editors, {\em Mathematical Models for Handling
		Partial Knowledge in Artificial Intelligence}, pages 101--112. Plenum Press,
	New York, 1995.
	
	\bibitem{Lass20}
	D.~Lassiter.
	\newblock What we can learn from how trivalent conditionals avoid triviality.
	\newblock {\em Inquiry}, 63(9-10):1087--1114, 2020.
	
	\bibitem{LaBa21}
	D.~Lassiter and J.~Baratgin.
	\newblock Nested conditionals and genericity in the de {F}inetti semantics.
	\newblock {\em Thought: A Journal of Philosophy}, 10(1):42--52, 2021.
	
	\bibitem{LeMa92}
	D.~Lehmann and M.~Magidor.
	\newblock What does a conditional knowledge base entail?
	\newblock {\em Artificial Intelligence}, 55(1):1--60, 1992.
	
	\bibitem{Luka99}
	Thomas Lukasiewicz.
	\newblock Probabilistic deduction with conditional constraints over basic
	events.
	\newblock {\em J. Artif. Int. Res.}, 10(1):199–241, apr 1999.
	
	\bibitem{MacColl08}
	H.~MacColl.
	\newblock ’if’ and ’imply’.
	\newblock {\em Mind}, 17:453–455, 1908.
	
	\bibitem{McGe89}
	V.~McGee.
	\newblock Conditional probabilities and compounds of conditionals.
	\newblock {\em Philosophical Review}, 98(4):485--541, 1989.
	
	\bibitem{Miln97}
	P.~Milne.
	\newblock {Bruno de Finetti and the Logic of Conditional Events}.
	\newblock {\em British Journal for the Philosophy of Science}, 48(2):195--232,
	1997.
	
	\bibitem{mundici21}
	D.~Mundici.
	\newblock Deciding {K}oopman's qualitative probability.
	\newblock {\em Artificial Intelligence}, 299:103524, 2021.
	
	\bibitem{NgWa94}
	H.~T. {Nguyen} and E.~A. {Walker}.
	\newblock A history and introduction to the algebra of conditional events and
	probability logic.
	\newblock {\em IEEE Transactions on Systems, Man, and Cybernetics},
	24(12):1671--1675, 1994.
	
	\bibitem{OvBa16}
	D.~E. Over and J.~Baratgin.
	\newblock The ``defective'' truth table: its past, present, and future.
	\newblock In N.~Galbraith, D.~E. Over, and E.~Lucas, editors, {\em The Thinking
		Mind: The use of thinking in everyday life}, pages 15--28. Psychology Press,
	2016.
	
	\bibitem{OHEHS07}
	D.~E. Over, C.~Hadjichristidis, J.~St.~B.T. Evans, S.~J. Handley, and S.~A.
	Sloman.
	\newblock The probability of causal conditionals.
	\newblock {\em Cognitive Psychology}, 54(1):62--97, 2007.
	
	\bibitem{PeVa17}
	D.~Petturiti and B.~Vantaggi.
	\newblock Envelopes of conditional probabilities extending a strategy and a
	prior probability.
	\newblock {\em International Journal of Approximate Reasoning}, 81:160 -- 182,
	2017.
	
	\bibitem{pfeifer12x}
	N.~Pfeifer.
	\newblock Experiments on {A}ristotle's {T}hesis: {T}owards an experimental
	philosophy of conditionals.
	\newblock {\em The Monist}, 95(2):223--240, 2012.
	
	\bibitem{pfeifer09b}
	N.~Pfeifer and G.~D. Kleiter.
	\newblock Framing human inference by coherence based probability logic.
	\newblock {\em Journal of Applied Logic}, 7(2):206--217, 2009.
	
	\bibitem{PfSa23SL}
	N.~Pfeifer and G.~Sanfilippo.
	\newblock Connexive logic, probabilistic default reasoning, and compound
	conditionals.
	\newblock Submitted.
	
	\bibitem{PfSa17}
	N.~Pfeifer and G.~Sanfilippo.
	\newblock Probabilistic squares and hexagons of opposition under coherence.
	\newblock {\em International Journal of Approximate Reasoning}, 88:282--294,
	2017.
	
	\bibitem{PS21connexive}
	N.~Pfeifer and G.~Sanfilippo.
	\newblock Interpreting connexive principles incoherence-based probability
	logic.
	\newblock In J.~Vejnarov\'a and J.~Wilson, editors, {\em Symbolic and
		Quantitative Approaches to Reasoning with Uncertainty ({ECSQARU} 2021)}, LNAI
	12897, pages 672--687. Springer, Cham, 2021.
	
	\bibitem{PoOB10}
	G.~Politzer, D.~E. Over, and J.~Baratgin.
	\newblock Betting on conditionals.
	\newblock {\em Thinking \& Reasoning}, 16(3):172--197, 2010.
	
	\bibitem{ramsey94}
	F.~P. Ramsey.
	\newblock General propositions and causality (1929).
	\newblock In D.~H. Mellor, editor, {\em Philosophical Papers by {F. P.
			R}amsey}, pages 145--163. Cambridge University Press, Cambridge, 1929/1990.
	
	\bibitem{SUM2018S}
	G.~Sanfilippo.
	\newblock Lower and upper probability bounds for some conjunctions of two
	conditional events.
	\newblock In {\em {SUM} 2018}, volume 11142 of {\em LNCS}, pages 260--275.
	Springer International Publishing, Cham, 2018.
	
	\bibitem{SGOP20}
	G.~Sanfilippo, A.~Gilio, D.~E. Over, and N.~Pfeifer.
	\newblock Probabilities of conditionals and previsions of iterated
	conditionals.
	\newblock {\em International Journal of Approximate Reasoning}, 121:150 -- 173,
	2020.
	
	\bibitem{SaPG17}
	G.~Sanfilippo, N.~Pfeifer, and A.~Gilio.
	\newblock Generalized probabilistic modus ponens.
	\newblock In A.~Antonucci, L.~Cholvy, and O.~Papini, editors, {\em ECSQARU
		2017}, volume 10369 of {\em LNCS}, pages 480--490. Springer, 2017.
	
	\bibitem{SPOG18}
	G.~Sanfilippo, N.~Pfeifer, D.~E. Over, and A.~Gilio.
	\newblock Probabilistic inferences from conjoined to iterated conditionals.
	\newblock {\em International Journal of Approximate Reasoning}, 93(Supplement
	C):103 -- 118, 2018.
	
	\bibitem{stalnaker68}
	R.~Stalnaker.
	\newblock A theory of conditionals.
	\newblock In N.~Rescher, editor, {\em Studies in logical theory}, pages
	98--112. Blackwell, Oxford, 1968.
	
	\bibitem{surma81}
	S.~J. Surma.
	\newblock Deduction theorem.
	\newblock In W.~Marciszeweski, editor, {\em Dictionary of logic as applied in
		the study of language}, pages 79--81. Martinus Nijhoff Publishers, The Hague,
	1981.
	
	\bibitem{vanFraassen76}
	B.~C. Van~Fraassen.
	\newblock Probabilities of conditionals.
	\newblock In {\em Foundations of Probability Theory, Statistical Inference, and
		Statistical Theories of Science}, volume~1, pages 261--308. Springer, 1976.
	
\end{thebibliography}

\end{document}